\documentclass{article}

\usepackage[preprint]{neurips_2025}

\usepackage{amsmath,amssymb,amsfonts,amsthm}
\usepackage{algorithmic}
\usepackage{algorithm}
\usepackage{graphicx}
\usepackage{float}
\usepackage{xcolor}
\usepackage{booktabs}
\usepackage{natbib}
\usepackage{xspace}
\usepackage{enumitem}
\usepackage{tabularx}
\usepackage{multirow}
\usepackage{hyperref}
\usepackage{url}
\usepackage{tikz}
\usepackage[most]{tcolorbox}
\usetikzlibrary{arrows.meta,positioning,shapes.geometric,fit,backgrounds,calc,decorations.pathreplacing}

\setlist{nosep,leftmargin=*}

\newcommand{\rler}{\textsc{RLER}\xspace}
\newcommand{\erm}{\textsc{ERM}\xspace}
\newcommand{\doop}{\text{do}}

\newtheorem{definition}{Definition}
\newtheorem{theorem}{Theorem}
\newtheorem{proposition}{Proposition}
\newtheorem{corollary}{Corollary}
\newtheorem{lemma}{Lemma}

\title{Epistemic Regret Minimization: Label-Free Causal Critique Beyond Outcome Reward}

\author{Edward Y. Chang \quad Longling Geng \\
Stanford University
}

\begin{document}
\maketitle

\begin{abstract}
Large language models can answer causal questions correctly for the wrong reasons. Current RL methods reward \emph{what} a model concludes but ignore \emph{why}, reinforcing correlational shortcuts---a failure we call \emph{Reward Entrenchment}. We introduce \emph{Epistemic Regret Minimization} (\erm), a framework that critiques the causal \emph{structure} of a model's reasoning trace rather than its answer. Applying established causal principles, \erm flags unexamined confounders, correlation--intervention conflation, and unchecked back-door paths from exposed reasoning traces. The framework admits \emph{label-free} operation---without the true causal graph or correct answer---and we separately distinguish favorable benchmark-derived critique, error-direction cues, and fully label-free judge-generated critique in the experiments. Within a single episode, \erm detects and repairs causal reasoning errors; across episodes, it accumulates interventional evidence into a reward signal applicable where no answer key exists. Experiments on 1,360 scenarios across six frontier LLMs show that reasoning-heavy models (GPT-4 Turbo, GPT-5.2) resist outcome-only correction (25--31\% recovery) yet respond to causal critique (78--91\%), gaining $+53$--$59$ pp. Standard test-time methods (self-consistency, Best-of-$N$, Self-Refine) \emph{underperform} outcome-only reprompting on causal tasks, while ERM reduces residual Rung Collapse from 55--70\% to 4\%. A separation theorem proves outcome-only reward cannot close this gap; a controlled simulation confirms epistemic feedback does, outperforming outcome-only baselines 38-fold.
\end{abstract}

\section{Introduction}
\label{sec:intro}
\vspace{-.1in}

Large language models can answer causal questions correctly for the wrong reasons. A model that substitutes associational shortcuts $P(Y|X)$ for the interventional query $P(Y|\doop(X))$ receives the same reward as one that reasons causally. Every current RL variant, from RLHF~\citep{ouyang2022rlhf,christiano2017rlhf} to RLVR~\citep{rlvr2025} to process reward models~\citep{lightman2023prm}, looks at \emph{what} the agent produces, not \emph{why}.

This blind spot has measurable consequences. \citet{jin2023cladder} showed that GPT-4 achieves only 64\% on causal queries with failures concentrated at the interventional level; \citet{zevcevic2023causal} argued this is structural in autoregressive models; \citet{chi2024unveiling} confirmed it persists across model generations. We formalize this as \emph{Rung Collapse} (Definition~\ref{def:rung-collapse}, \S\ref{sec:problem}): the systematic substitution of lower-rung reasoning for higher-rung queries in Pearl's causal hierarchy~\citep{pearl2009causality,bareinboim2022fusion}. Worse, when an LLM treats correlation as causation and still gets the answer right, outcome-based reward \emph{reinforces} the flawed reasoning, a feedback loop we call \emph{Reward Entrenchment} that persists until distribution shift exposes the shortcut. Most consequential reasoning (scientific inference, clinical decision-making, open-world causal claims) lacks ground-truth verifiers at inference time (\emph{Regime~B}), placing it squarely in the regime where RLVR (which requires a verifier, \emph{Regime~A}) is structurally inapplicable.

\vspace{-.1in}
\paragraph{Our approach.} We introduce \emph{Epistemic Regret Minimization} (\erm), a framework that critiques \emph{how} a model reasons about cause and effect, rather than whether it reaches the right answer. Applying established causal principles grounded in Pearl's do-calculus, \erm reads the model's exposed reasoning trace and flags structural weaknesses---unexamined confounders, conflation of correlation with intervention, and unchecked back-door paths---none of which require knowing the true causal graph (Appendix~\ref{app:illustrated-examples} provides illustrated walk-throughs). We call this \emph{label-free} (gold-free) operation: the critique is generated from the reasoning trace alone, without access to the correct answer. Within a single episode, ERM critique detects and repairs these errors; fully label-free operation is validated in the ablations (\S\ref{sec:ablations}, Appendix~\ref{app:exp2b}) and cross-episode protocol. Across episodes, \erm accumulates interventional evidence into a reward signal applicable where no answer key exists.

\vspace{-.10in}
\paragraph{Contributions.}
\begin{enumerate}
\item \textbf{C1: Problem formalization} (\S\ref{sec:problem}). We formalize Rung Collapse and Reward Entrenchment within Pearl's causal hierarchy and prove that neither autoregressive training nor outcome-based RL provides signal to distinguish $\mathcal{L}_1$ from $\mathcal{L}_2$ reasoning when both yield the same output. An Interventional Grounding Theorem bridges action traces and do-calculus.

\item \textbf{C2: Single-episode correction and ablation} (\S\ref{sec:erm} and \S\ref{sec:ablations}). Across 1,360 scenarios on 6 models (CausalT5K~\citep{causalT5K}, confirmed on CLadder~\citep{jin2023cladder}), reasoning-heavy models resist outcome-only correction (25--31\%) yet respond to causal critique (78--91\%). Ablations confirm causal content is load-bearing and label-free critique works without correct answers; a scenario-blind judge argues against answer leakage. On GPT-4T Wolf Cases, ERM outperforms self-consistency, Best-of-$N$, and Self-Refine, which all \emph{underperform} outcome-only reprompting while leaving 55--70\% residual Rung Collapse versus ERM's 4\%.

\item \textbf{C3: Cross-episode reward signal} (\S\ref{sec:rler} and \S\ref{sec:separation-sim}). A separation theorem proves outcome-only RL cannot distinguish correct from flawed causal models in confounded environments. Epistemic signal is confirmed on real traces across four LLMs. A controlled simulation (CausalBench-Seq) closes the reward-to-policy loop, achieving 38-fold regret separation; cross-episode memory is load-bearing. End-to-end LLM policy optimization remains future work.
\end{enumerate}

\vspace{-.1in}
\paragraph{Related work.} CLadder~\citep{jin2023cladder}, CausalProbe~\citep{chi2024unveiling}, and CauSciBench~\citep{acharya2025causcibench} document systematic causal reasoning failures across frontier LLMs; \citet{zevcevic2023causal} give formal arguments that autoregressive training cannot implement the do-operator. The ``right for the wrong reasons'' phenomenon~\citep{mccoy2019right,geirhos2020shortcut} and simplicity bias~\citep{shah2020pitfalls} establish that neural networks exploit shallow regularities; our Reward Entrenchment (Definition~\ref{def:entrenchment}) is the causal-domain instantiation. Process reward models~\citep{lightman2023prm} verify step-to-step consistency but cannot detect Rung Collapse; causal bandits~\citep{lattimore2016causal} optimize outcome regret given known structure rather than discovering flawed structure from traces. Our \erm instantiates AGM belief revision~\citep{agm1985} with interventional observations, inheriting causal guarantees from Theorem~\ref{thm:grounding}. Recent label-free methods such as MM-UPT~\citep{wei2025mmupt} use majority-vote pseudo-labels for weight updates on mathematical tasks; this self-consistency mechanism is precisely what breaks down under Rung Collapse (Corollary~\ref{prop:rung-collapse}) for causal reasoning, where multiple incorrect chains converge on the same $\mathcal{L}_1$ shortcut (Appendix~\ref{app:related}). Object-centric world models such as Causal-JEPA~\citep{nam2026causaljepa} induce causal inductive biases through object-level latent interventions; \erm addresses a complementary reasoning layer, auditing whether an LLM's exposed causal explanation relies on valid interventional structure rather than outcome-compatible shortcuts. Extended comparisons appear in Appendix~\ref{app:related}.
\label{sec:related}

\vspace{-.1in}
\section{Epistemic Regret Minimization}
\label{sec:method}
\vspace{-.08in}

We first establish the failure modes that motivate the framework (\S\ref{sec:problem}), then present the core single-episode correction method (\S\ref{sec:erm}) and its extension to cross-episode reward signals (\S\ref{sec:rler}). All symbols are defined inline at first use and summarized in Table~\ref{tab:notation} (Appendix~\ref{app:notation}).

\vspace{-.12in}
\subsection{Background: Rung Collapse and Reward Entrenchment}
\label{sec:problem}
\vspace{-.08in}

We work within Pearl's Structural Causal Models (SCM)~\citep{pearl2009causality}: $\mathcal{M} = \langle \mathbf{U}, \mathbf{V}, \mathcal{F}, P(\mathbf{U})\rangle$ with independent mechanisms and semi-Markovian DAGs (component definitions in Table~\ref{tab:notation}). Let $G^*$ denote the true causal DAG. Pearl's Causal Hierarchy~\citep{bareinboim2022fusion} distinguishes Association ($\mathcal{L}_1$: $P(Y|X)$), Intervention ($\mathcal{L}_2$: $P(Y|\doop(X))$), and Counterfactuals ($\mathcal{L}_3$). The Causal Hierarchy Theorem (CHT) dictates that $\mathcal{L}_1$ data alone cannot determine $\mathcal{L}_2$ distributions.

\begin{definition}[Rung Collapse]
\label{def:rung-collapse}
An agent exhibits Rung Collapse when it responds to a Rung-$j$ query using Rung-$i$ reasoning where $i < j$. For query $Q_j$ requiring $\mathcal{L}_j$ inference, $(i \to j)$ collapse occurs when the agent's answer function operates on $\mathcal{L}_i$ quantities but the correct answer requires $\mathcal{L}_j$.
\end{definition}

\vspace{-.07in}
\begin{corollary}[LLM Rung Collapse: immediate from CHT]
\label{prop:rung-collapse}
Neither the next-token loss nor outcome-based RL reward $R_\text{out}(Y, Y^*)$ (where $Y^*$ is the correct answer) distinguishes $\mathcal{L}_1$ from $\mathcal{L}_2$ reasoning when both produce the same output. Consequently, neither pre-training nor outcome-based post-training introduces interventional signal~\citep{zevcevic2023causal}. Proof in Appendix~\ref{app:proofs}.
\end{corollary}

\vspace{-.07in}
\begin{definition}[Reward Entrenchment]
\label{def:entrenchment}
An agent achieves Spurious Success when it produces a correct outcome $Y^*$ despite using an incorrect causal model $\hat{G} \neq G^*$: $P(Y^*|\hat{G}) > 0$ but $\hat{G} \not\models G^*$. Outcome-based learning (RLHF, RLVR) cannot, from reward alone, distinguish Spurious Success from genuine understanding. The positive reward reinforces the flawed model, producing Reward Entrenchment: the agent becomes increasingly committed to an incorrect model that fails under distributional shift.
\end{definition}

A downstream consequence is Epistemic Stubbornness: reasoning-heavy models resist outcome-only correction yet respond to targeted epistemic feedback (\S\ref{sec:erm-experiments}). Four established results anchor the framework (Appendix~\ref{app:foundations}): Murphy's decomposition~\citep{murphy1973vector}, prequential calibration~\citep{dawid1984prequential,foster1998asymptotic}, FCI completeness~\citep{spirtes1993causation,zhang2008completeness}, and the AGM representation theorem~\citep{agm1985}.
\label{sec:foundations}

\vspace{-.12in}
\subsection{Single-Episode Correction}
\label{sec:erm}
\label{sec:grounding}
\vspace{-.08in}

To escape Rung Collapse, \erm requires a source of interventional evidence: data generated under $\doop(X{=}x)$ rather than passive observation. We first establish when an agent's actions produce valid do-calculus interventions, then define the \erm objective and architecture.

\vspace{-.07in}
\begin{definition}[Intervention Independence]
\label{def:actuator}
An action $A$ on variable $X \in \mathbf{V}$ satisfies Intervention Independence if $A(X \leftarrow x)$ is determined by the agent's control signal, independent of $\mathrm{Pa}(X)$ or the exogenous noise $U_X \in \mathbf{U}$.
\end{definition}

\vspace{-.07in}
\begin{theorem}[Interventional Grounding]
\label{thm:grounding}
Let $\mathcal{M}$ be an SCM as in \S\ref{sec:problem}. Let $A$ satisfy Intervention Independence (Definition~\ref{def:actuator}) for variable $X \in \mathbf{V}$. Then $A(X \leftarrow x)$ implements Pearl's do-operator: $P(Y \mid A(X \leftarrow x)) = P(Y \mid \doop(X=x))$ for all $Y \in \mathbf{V} \setminus \{X\}$.
\end{theorem}

\noindent The proof (Appendix~\ref{app:proofs}) follows from modularity: replacing the structural equation $f_X \in \mathcal{F}$ with a constant leaves all other structural equations invariant. This bridges action traces and do-calculus: tool-using LLM agents (API calls, code execution, database writes) produce valid interventional data when Intervention Independence holds. The assumption fails when the ``intervention'' is mediated by the same process being reasoned about (e.g., rephrasing a prompt to the same LLM).

\vspace{-.06in}
\paragraph{Corollary 2 (Interventional unconfoundedness).}
\label{cor:immunity}
Under Intervention Independence, the action assignment is independent of
pre-treatment causes of $X$. Hence
\begin{small}
$P(Y\mid A(X\leftarrow x))=P(Y\mid \doop(X=x))$,
\end{small}
so the marginal effect of the action estimates the interventional distribution
without back-door adjustment. 
Conditioning on common cause $Z$ may change the distribution of $Y$ if $Z \to Y$.

\vspace{-.06in}
\subsubsection{The ERM Objective}
\label{sec:erm-objective}

We first define the three stateful artifacts that \erm maintains across episodes.

\vspace{-.06in}
\begin{definition}[Causal Model $G_t$]
\label{def:causal-model}
The agent maintains an explicit causal DAG $G_t = (V, E_t, w_t)$ where $V$ is the set of observed variables, $E_t \subseteq V \times V$ is the current edge set (respecting acyclicity), and $w_t : E_t \to [0,1]$ assigns confidence weights derived from cumulative interventional evidence. $G_t$ is the agent's inspectable epistemic state: it represents the agent's current best hypothesis about the causal structure, and is updated by \erm's belief revision operator (Algorithm~\ref{alg:erm}).
\end{definition}

\vspace{-.06in}
\begin{definition}[Causal Transaction Log (CTL)]
\label{def:ctl}
The CTL is an append-only evidence store recording a tuple $(t_i, S_i, H_i, a_i, \hat{Y}_i, Y_i, \Delta_i)$ per episode: timestamp $t_i$, scenario identifier $S_i$, the agent's causal hypothesis $H_i$, the action taken $a_i$, the predicted outcome $\hat{Y}_i$, the observed outcome $Y_i$, and epistemic error $\Delta_i = \hat{Y}_i - Y_i$ \citep{sagallm2025}. The CTL provides two query operators: $\mathrm{CTL.Support}(C_j) = |\{i : \Delta_i \leq \epsilon \text{ and } C_j \in H_i\}|$ and $\mathrm{CTL.Refute}(C_j) = |\{i : \Delta_i > \epsilon \text{ and } C_j \in H_i\}|$ for a tolerance threshold $\epsilon > 0$, which aggregate interventional evidence for and against any causal claim $C_j$.
\end{definition}

\vspace{-.06in}
\noindent A third artifact, the \emph{Failure-Mode Registry} $\mathcal{F}_t$ (Definition in Appendix~\ref{app:algorithm}), classifies recurring structural error patterns and injects falsifiable reasoning guards. Given these artifacts, we define \emph{Epistemic Regret} $R_\text{ep}(t)$ as the KL divergence between the model's predicted interventional distribution and the observed interventional distribution; under Intervention Independence (Theorem~\ref{thm:grounding}), the latter samples $P(Y|\doop(X), G^*)$. The full \erm objective balances three terms:

\vspace{-.21in}
\begin{footnotesize}
\begin{equation}
\label{eq:erm}
\mathcal{L}(G_t) = \underbrace{\mathcal{L}_\text{task}(Y, Y^*)}_\text{Outcome} + \lambda \underbrace{R_\text{ep}(t)}_\text{Epistemic} + \mu \underbrace{\mathcal{L}_\text{con}(G_t)}_\text{Consistency}
\end{equation}
\end{footnotesize}
where $\mathcal{L}_\text{task}$ is outcome loss, $\mathcal{L}_\text{con}(G_t)$ penalizes acyclicity and edge-weight violations in the causal DAG, and $\lambda, \mu > 0$ control the trade-off. The key property: if the agent achieves Spurious Success, $R_\text{ep} > 0$ even when $\mathcal{L}_\text{task} = 0$.

\vspace{-.06in}
\noindent\emph{Empirical scope.} Equation~\ref{eq:erm} defines the target objective for \erm, but this paper does not train a differentiable policy end-to-end. Instead, our experiments validate the objective's load-bearing components: whether causal critique supplies an epistemic signal beyond outcome reward, whether that signal corrects Rung Collapse in single episodes, and whether accumulated epistemic feedback closes the reward-to-policy loop in a controlled SCM simulation. The present contribution is thus an inference-time and simulation-validated instantiation of \erm, with gradient-based policy optimization left for future work (\S\ref{sec:discussion}).

\vspace{-.03in}
\begin{lemma}[Prevention of Reward Entrenchment]
\label{thm:entrenchment}
Under the ERM objective (Eq.~\ref{eq:erm}) with gradient access to $G_t$, Spurious Success cannot be a fixed point: if $G_t \not\equiv_{\mathcal{I}} G^*$, then $\mathcal{L} \geq \lambda \cdot D_{KL}(\hat{P}_{G_t} \| P_{G^*}) > 0$. This follows directly from the positivity of KL-divergence and $\lambda > 0$; convergence rate and local-minima analysis in DAG space are addressed by the asymptotic recovery guarantee (Theorem~\ref{thm:convergence}). Proof in Appendix~\ref{app:proofs}.
\end{lemma}

\vspace{-.06in}
\begin{theorem}[Asymptotic $\mathcal{L}_2$ Recovery]
\label{thm:convergence}
An \erm agent performing $N$ interventions on variable $X$ recovers the true interventional distribution $\lim_{N \to \infty} \hat{P}_N(Y|\doop(X)) = P(Y|\doop(X))$ almost surely, provided: (i) Intervention Independence, (ii) independent mechanisms, (iii) full observability of $Y$, and (iv) stationarity. The sample complexity scales as $O(\epsilon^{-2} \log |\mathcal{Y}|)$ (Appendix~\ref{app:additional-proofs}).
\end{theorem}
\noindent\emph{Scope note:} Assumption~(iv) is violated in the \rler experiments; the i.i.d.\ guarantee serves as a best-case baseline (non-stationary extensions in Appendix~\ref{app:additional-proofs}). Finite-sample recovery is empirically confirmed: under adversarial pressure on CausalL2, Llama~3.3~70B recovers from 90.2\% to 96.6\%, exceeding its clean baseline of 96.2\% (Appendix~\ref{app:pressure}).

\vspace{-.12in}
\subsubsection{Architecture}
\label{sec:architecture}
\vspace{-.06in}

The architecture (Figure~\ref{fig:architecture}, Appendix~\ref{app:architecture}) operates over a frozen LLM with three layers and two reward channels.

\emph{Layer~1: Instance Correction.} On each subtask, the agent generates hypothesis $H_i$ conditioned on $G_t$, executes action $a_i$, and observes outcome $Y_i$. If $|\hat{Y}_i - Y_i| > \epsilon$, the system identifies which causal claims in $H_i$ produced the misprediction, queries the CTL for cumulative evidence, and contracts or reinforces edges via AGM revision.

\emph{Layer~2: Schema Adaptation.} When the same reasoning error recurs across unrelated domains, the system classifies it into a structural \emph{failure mode} $f_k$ from a predefined taxonomy $\mathcal{C}$ (Table~\ref{tab:failure-modes}). When a mode's recurrence count exceeds a threshold $\tau_f$, a domain-independent reasoning guard is injected into the meta-prompt. Guards are falsifiable: if a guard increases regret, it is retracted via AGM contraction.

\emph{Layer~3: Epistemic Routing.} When residual regret exceeds a threshold $\theta_\text{route}$ despite corrections, or when $\Delta\text{Net} < 0$ (an \emph{iatrogenic} regime), queries are routed to a different model or flagged for human review. Empirically, authoritative critique on weaker models produces pure iatrogenic effects (GPT-3.5: $\Delta\text{Net}{=}{-}16$; Gemini: $\Delta\text{Net}{=}{-}7$), while stronger models show positive responses (Appendix~\ref{app:iatrogenic}), confirming routing is architecturally necessary.

The full ERM belief revision algorithm and failure taxonomy appear in Appendix~\ref{app:algorithm} (Alg.~\ref{alg:erm}, Tbl.~\ref{tab:failure-modes}).

\vspace{-.08in}
\paragraph{Relation to structured prompting.} \erm differs from chain-of-thought in three ways: critique is \emph{causally grounded} (naming confounders and back-door paths; load-bearing, $p{=}0.006$), it maintains \emph{stateful artifacts} ($G_t$, CTL) across episodes, and these are \emph{inspectable}. Unlike Self-Refine, \erm targets causal \emph{structure} without correct answers. Trace-to-DAG extraction achieves 82\% recall, 91\% precision (Appendix~\ref{app:trace-examples}).

\vspace{-.12in}
\subsection{Cross-Episode Reward Signal}
\label{sec:rler}
\vspace{-.10in}

The single-episode framework detects and corrects errors within one interaction. Across episodes, the same interventional evidence accumulates into a reward signal for strategy selection, yielding \emph{Reinforcement Learning from Epistemic Regret} (\rler). This targets open-domain problems lacking ground-truth verifiers, the regime where outcome-based RL is structurally inapplicable. \rler reuses all \erm artifacts ($G_t$, CTL, failure-mode registry); the additions are a cross-episode reward decomposition, an evaluation metric, and a separation theorem.

\vspace{-.12in}
\paragraph{Reward Decomposition.}
The \rler reward decomposes into two components:
\begin{small}
\begin{equation}
\label{eq:rler-reward}
R_{\rler}(t) = \beta \cdot R_\text{reasoning}(t) + \gamma \cdot R_\text{discovery}(t)
\end{equation}
\end{small}
where $\beta, \gamma > 0$ are weighting coefficients. $R_\text{reasoning}$ is the Jaccard similarity $|H_t \cap H^*_t| / |H_t \cup H^*_t|$ between the trace's causal subgraph $H_t$ and the CTL-derived minimal consistent subgraph $H^*_t$, penalizing both missing and spurious edges. $R_\text{discovery}$ fires when a proposed novel variable passes three evidence-grounded checks. Both use only trace content and CTL entries, so $R_\rler$ is well-defined without gold labels.

\vspace{-.12in}
\paragraph{EDGap: Epistemic Discrimination.}
\label{sec:edgap}
We propose the \emph{Epistemic Discrimination Gap} (EDGap) as a complementary evaluation metric, measuring whether the agent's confidence scores separate correct from incorrect reasoning traces. Partition episodes into confidence terciles $B_k$. Let $\bar{o}_k$ be the empirical correct rate in bucket $B_k$, and $\bar{o}$ the overall base rate. EDGap is a variance-weighted form of Murphy's resolution 
term~\citep{murphy1973vector}:

\vspace{-.05in}
\begin{footnotesize}
\vspace{-.08in}
\begin{equation}
\label{eq:edgap}
\text{EDGap} = \sum_{k} \frac{n_k}{N} \cdot \frac{\left(\bar{o}_k - \bar{o}\right)^2}{\hat{\sigma}_k^2 + \epsilon}.
\end{equation}
\vspace{-.05in}
\end{footnotesize}

If EDGap $> 0$, reward signals built on confidence carry actionable signal. EDGap is computable from the (forecast, outcome) sequence alone~\citep{dawid1984prequential}, requiring no per-decision verifier. Full derivation in Appendix~\ref{app:edgap-derivation}.

\vspace{-.12in}
\paragraph{Separation Theorem.}

\begin{theorem}[Separation for Confounded Environments]
\label{thm:separation}
Consider a semi-Markovian SCM with observed $X, Y$ and latent confounder $C$, where $G^* : C \to X,\; C \to Y$ (no direct $X \to Y$ edge), and assume faithfulness. Then: (i) any RL algorithm whose reward depends only on $(X, Y)$ and whose model class can represent $X \to Y$ has no reward-based signal to exclude $X \to Y$, since $P(Y|X) \neq P(Y)$ under the marginal is equally consistent with a direct effect and with latent confounding; (ii) \rler, which additionally records interventional outcomes $P(Y|\doop(X))$ in its CTL and triggers confounder discovery on persistent epistemic surprise, converges to the interventional equivalence class of $G^*$, provided the agent can intervene on $X$ (Theorem~\ref{thm:grounding}) and sufficient episodes accumulate for FCI identification~\citep{zhang2008completeness}.
\end{theorem}
 
\noindent \emph{Proof sketch.} (i) Outcome-only RL sees $P(Y|X) \neq P(Y)$, which is equally consistent with $X \to Y$ and with latent confounding; the reward alone provides no signal to distinguish them. (ii) \rler's CTL reveals $P(Y|\doop(X)) = P(Y)$; traces asserting $X \to Y$ produce systematic epistemic surprise, triggering confounder discovery via FCI~\citep{spirtes1993causation,zhang2008completeness}. Full proof in Appendix~\ref{app:proofs}.

\noindent \emph{Faithfulness scope.} Part~(ii) requires faithfulness over the ancestral graph~\citep{zhang2008completeness}, plausible for CausalT5K's small designed graphs but potentially violated in general Regime~B deployment. Theorem~\ref{thm:separation} is an \emph{asymptotic identifiability} result, complemented by finite-sample validation in \S\ref{sec:rler-experiments}.

\vspace{-.1in}
\section{Experiments}
\label{sec:experiments}
\label{sec:erm-experiments}
\vspace{-.08in}

\begin{description}[style=unboxed,leftmargin=0pt,itemsep=2pt]
\item[\textbf{RQ1} (Detection \& Correction):] Do frontier models exhibit Rung Collapse, and can single-episode \erm correct the failures?
\item[\textbf{RQ2} (Cross-Episode Signal):] Does epistemic reward extend across episodes?
\item[\textbf{RQ3} (Judge Integrity):] Does the judge enforce structural rules rather than leak answers?
\item[\textbf{RQ4} (Ablation):] Which factors drive \erm's correction: causal content, judge quality, iteration budget, or reward axis?
\end{description}

\vspace{-.05in}
\noindent\emph{Scope.} This paper validates Layer~1 (instance correction) end-to-end: detection, single-episode correction, cross-episode signal, and judge integrity. Layer~3 (routing) is empirically motivated by the iatrogenic-critique finding (\S\ref{sec:architecture}); Layer~2 (schema adaptation) remains a design contribution evaluated in future work.

\vspace{-.12in}
\paragraph{Datasets.} RQ1 and RQ3--RQ4 use CausalT5K~\citep{causalT5K}, a public benchmark of $\mathcal{L}_2$ causal scenarios where the $\mathcal{L}_1$ answer is deliberately wrong. Detection and correction use v1 ($N{=}1{,}360$); the matched-richness ablation uses v2 ($N{=}4{,}054$). The cross-episode extension uses 466 Wolf Cases for GPT-4o and 132-case subsamples for the remaining three models (gold labels withheld during episodes). Because CausalT5K is constructed to make $\mathcal{L}_1$ shortcuts fail, it is favorable ground for \erm; cross-benchmark validation on CLadder $\mathcal{L}_2$ (Appendix~\ref{app:cladder}) provides independent confirmation that Rung Collapse is not benchmark-specific.

\vspace{-.12in}
\paragraph{Models.} RQ1 (single-episode): GPT-3.5 Turbo, Llama 3.3 70B, GPT-4 Turbo, Gemini 2.5 Flash, GPT-5.2, Claude Sonnet 3.5. RQ2 (cross-episode) and RQ3--RQ4: GPT-3.5 Turbo, GPT-4o, GPT-5.4-nano, Llama 3.3 70B. The two sets partially overlap; each targets a research question.

\vspace{-.12in}
\paragraph{Proxy evaluation framing.} Following CLadder~\citep{jin2023cladder} and CauSciBench~\citep{acharya2025causcibench}, we test causal \emph{reasoning} via text queries rather than physical intervention. Theorem~\ref{thm:grounding} specifies when \erm's cross-episode evidence log becomes a valid interventional memory; the CausalT5K and CLadder experiments test the complementary proxy question of whether causal critique can correct Rung Collapse in text. Accordingly, the LLM experiments should be read as evidence that causal critique can identify and repair exposed reasoning failures, while the formal separation claim (Theorem~\ref{thm:separation}) is supported only in the controlled CausalBench-Seq setting where the theorem's intervention assumptions hold. The matched-richness ablation and scenario-blind judge (\S\ref{sec:leakage}) bound the two most likely confounds (prompt richness and answer leakage).

\vspace{-.12in}
\subsection{RQ1: Detection, Scaling, and Correction}
\label{sec:rq1}
\label{sec:rq12}
\label{sec:rq3}
\vspace{-.08in}

\emph{Experiment~A (Detection)} measures zero-shot Rung Collapse across six frontier LLMs on CausalT5K v1.
\emph{Experiment~B (Correction)} applies outcome-only feedback (``Are you sure?'') versus targeted \erm critique to each model's Wolf Cases (instances that failed in Exp.~A). The critique targets reasoning \emph{structure} (confounders, back-door violations) rather than answer correctness. The headline experiment uses a benchmark-derived causal rationale with an error-direction cue, testing the correction mechanism under favorable conditions (prompt details in Appendix~\ref{app:reproducibility}). Fully label-free operation, where a judge generates critique from the reasoning trace alone without gold labels or error cues, is validated separately: the six-condition ablation (Appendix~\ref{app:exp2b}) confirms 0\% false-flips and causal content load-bearing for GPT-4T ($p{=}0.0001$), and the cross-episode protocol (\S\ref{sec:rler-experiments}) operates entirely label-free.

\begin{table}[t]
\vspace{-.05in}
\centering\footnotesize
\caption{Rung Collapse detection (Exp.~A) and correction (Exp.~B) on CausalT5K ($N{=}1{,}360$). Exp.~B targets each model's Wolf Cases (oracle-gated; judge integrity in \S\ref{sec:leakage}). $\Delta$ = ERM $-$ Outcome only. $p$: McNemar.}
\label{tab:collapse}
\label{tab:correction}
\vspace{-.05in}
\begin{tabular}{@{} l @{\;\;} r r @{\quad} r r @{\;\;} r @{\;\;} r @{}}
\toprule
 & \multicolumn{2}{c}{\textbf{Exp.~A (Detection)}} & \multicolumn{4}{c}{\textbf{Exp.~B (Correction on Wolf Cases)}} \\
\cmidrule(lr){2-3}\cmidrule(lr){4-7}
\textbf{Model} & \textbf{Collapse} & \textbf{Accuracy} & \textbf{Outcome only} & \textbf{ERM (trace)} & $\boldsymbol{\Delta}$ & $\boldsymbol{p}$ \\
\midrule
\multicolumn{7}{@{}l}{\textit{Compliant models (outcome-only reprompting already effective)}} \\[2pt]
GPT-3.5 Turbo     & $17.3\%$ & $82.9\%$ & $91.4\%$ & $93.5\%$ & $+2.1$ & $.49$ \\
Llama 3.3 70B$^{\S}$ & $15.1\%$ & $84.9\%$ & $89.8\%$ & $92.2\%$ & $+2.4$ & $.07$ \\
Gemini 2.5 Flash  & $7.7\%$  & $92.3\%$ & $90.2\%$ & $96.6\%$ & $+6.4$ & $.02$ \\
Claude Sonnet 3.5\textsuperscript{$\dagger$} & $0.9\%$  & $98.5\%$ & $100\%$ & $83.3\%$ & $-16.7$ & $.48$ \\
\midrule
\multicolumn{7}{@{}l}{\textit{Stubborn models (outcome-only correction fails, ERM succeeds)}} \\[2pt]
GPT-4 Turbo       & $12.5\%$ & $87.5\%$ & $31.4\%$ & $90.5\%$ & $+59.1$ & $<\!.001$ \\
GPT-5.2           & $3.7\%$  & $96.3\%$ & $24.5\%$ & $77.6\%$ & $+53.1$ & $<\!.001$ \\
\bottomrule
\end{tabular}

\vspace{.05in}
{\scriptsize $\dagger$Claude Sonnet 3.5: $n{=}12$ Wolf Cases; compliant (100\% OO recovery). Wide CIs; stubborn-model evidence from GPT-4T ($n{=}170$) and GPT-5.2 ($n{=}50$). $\S$Llama 3.3 70B: a larger re-run under updated prompts ($n{=}1{,}101$) yields 96.2\% detection accuracy (3.8\% collapse), confirming compliant-model classification. Table reports the original prompt configuration for cross-model protocol consistency.}
\vspace{-.15in}
\end{table}

\textbf{Detection} (Table~\ref{tab:collapse}, left): Scaling reduces but does not eliminate Rung Collapse. GPT-5.2 retains a 3.7\% collapse rate (50 cases) and Claude Sonnet 3.5 a 0.9\% rate (12 cases), confirming the problem is structural; even the strongest models exhibit residual $\mathcal{L}_1$ shortcuts.

Table~\ref{tab:collapse} correction rates are oracle-gated (evaluated on each model's detected Wolf Cases); deployment requires a separate trigger or judge (label-free trigger: Appendix~\ref{app:exp2b}; cross-episode protocol: \S\ref{sec:rler-experiments}).

\textbf{Correction} (Table~\ref{tab:collapse}, right): The baseline is outcome-only ``Are you sure?'' reprompting. Self-consistency, Best-of-$N$, and Self-Refine all \emph{underperform} this baseline on causal tasks (App.~\ref{app:baseline-suite}). ERM targets confounders and back-door paths without correct answers, isolating \emph{reasoning-quality} from \emph{outcome} feedback.

Compliant models (GPT-3.5, Llama, Gemini) already achieve 89--100\% correction under outcome-only reprompting, leaving limited headroom ($+2$ to $+6$ pp).

The picture changes for reasoning-heavy models. GPT-4 Turbo and GPT-5.2 resist outcome-only correction (31\% and 25\%) but respond to ERM's causal critique (90.5\% and 77.6\%; both $p < 0.001$, McNemar). Gemini's nominal gain ($+6.4$ pp, $p{=}.02$) does not survive Holm correction. We call this bifurcation \emph{Epistemic Stubbornness}: models that build confident internal justifications cannot be dislodged by outcome-only reprompting but respond when critique names specific confounders and checks back-door paths (Figure~\ref{fig:correction-main}a).

False-correction rates remain low ($0.5\%$ GPT-4T, $8.0\%$ GPT-5.2; Appendix~\ref{app:false-flip}). A six-condition ablation on both stubborn models (Appendix~\ref{app:exp2b}) shows that false-flips are driven by the error-direction cue (``you were incorrect''), not the critique itself: label-free critique without the cue achieves 0\% false-flips across ${\sim}300$ tests (Fisher $p < 10^{-5}$). The ablation also confirms causal vocabulary is load-bearing for GPT-4T ($+21.4$ pp, $p = 0.0001$).

\emph{Conditional recovery potential.} Under this oracle gate, \erm raises GPT-4T from 87.5\% to 98.8\% vs.\ 91.4\% for outcome-only ($+7.4$ pp). For GPT-5.2, \erm reaches 99.1\% vs.\ 97.2\%, but outcome-only correction is actively \emph{harmful} (20\% false-flip rate), the most direct evidence that Epistemic Stubbornness scales with capability.

\emph{Extended validation.} The bifurcation replicates on CausalL2 ($N{=}1{,}000$/model, 5 models), with stubborn cases clustering into Confounding (33\%), Rung Collapse (31\%), and Empty Verification (16\%) (Appendix~\ref{app:stubborn-taxonomy}). Structural critique recovers adversarial-pressure sycophancy drops of 4--9 pp (Appendix~\ref{app:pressure}).

\begin{figure}[t]
\vspace{-.1in}
\centering
\begin{minipage}[t]{0.48\linewidth}
\centering
\includegraphics[width=\linewidth,height=0.75\linewidth,keepaspectratio]{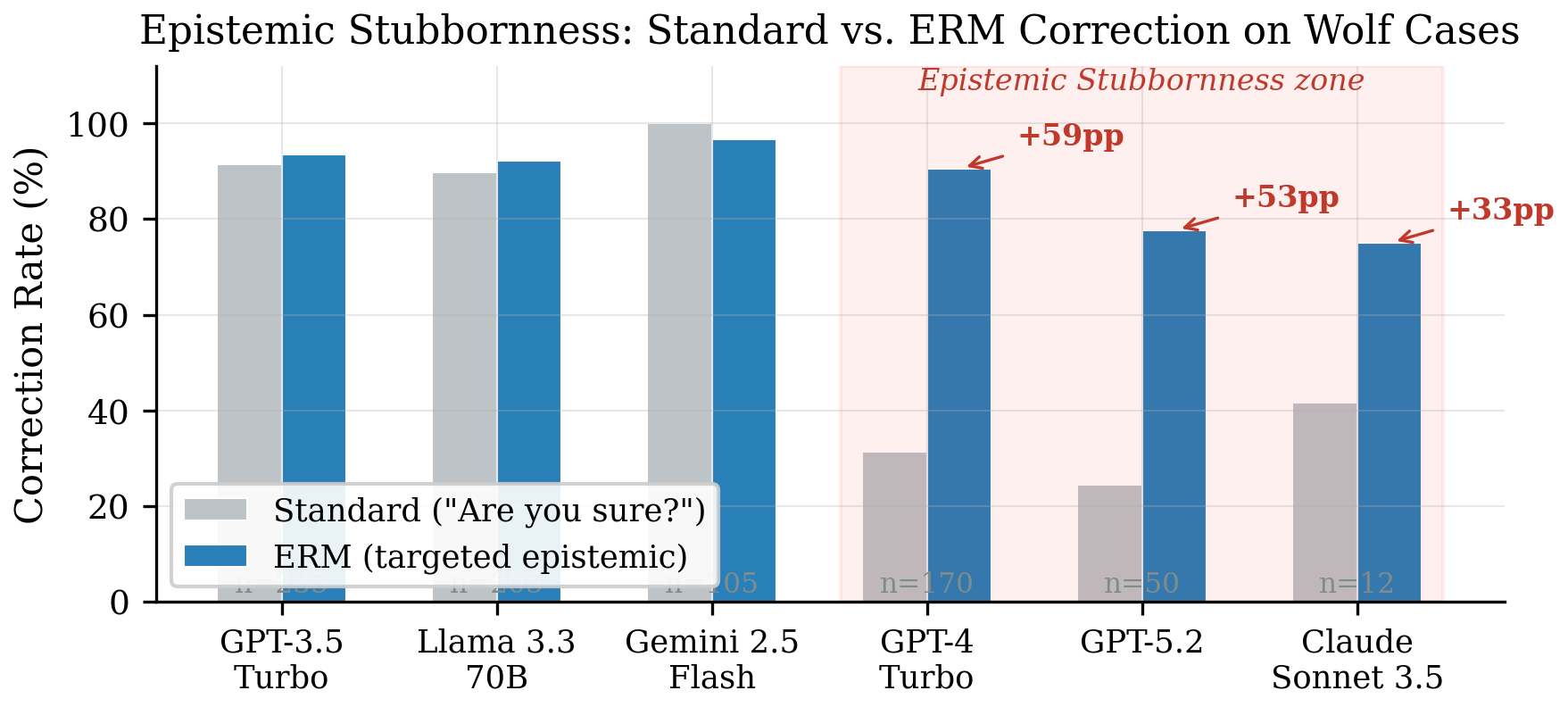}
\end{minipage}\hfill
\begin{minipage}[t]{0.48\linewidth}
\centering
\includegraphics[width=\linewidth,height=0.75\linewidth,keepaspectratio]{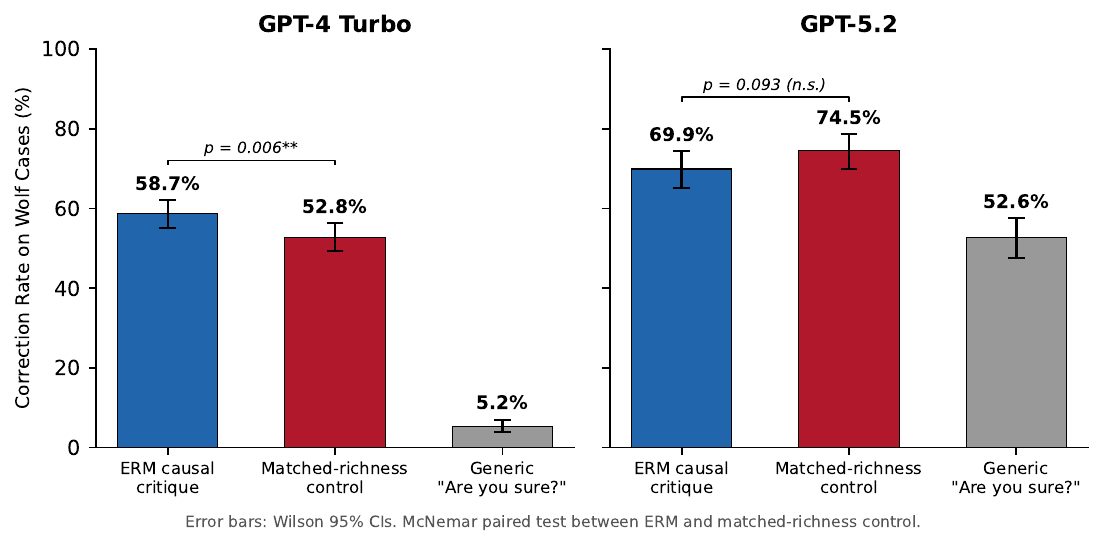}
\end{minipage}
\vspace{-.08in}
\caption{\textbf{(a)} Epistemic Stubbornness: stubborn models (GPT-4T, GPT-5.2) resist outcome-only correction (25--31\%) yet achieve $+53$--$59$ pp gains under ERM's causal critique. \textbf{(b)} Matched-richness ablation confirms causal vocabulary is the active ingredient for GPT-4T. Wilson 95\% CIs.}
\label{fig:matched-richness-main}
\label{fig:correction-main}
\vspace{-.1in}
\end{figure}

\vspace{-.12in}
\paragraph{Matched-richness ablation.}
A matched-richness ablation on CausalT5K v2 ($N{=}4{,}054$) disentangles causal targeting from prompt richness using three conditions: (1)~ERM causal critique, (2)~matched-richness non-causal control (identical length and framing, zero causal vocabulary), and (3)~outcome-only (Figure~\ref{fig:correction-main}b, Table~\ref{tab:matched-richness}).

\begin{table}[t]
\centering\footnotesize
\caption{Matched-richness ablation on Wolf (CausalT5K v2). Wilson 95\% CIs; McNemar's test.}
\label{tab:matched-richness}
\begin{tabular}{llccc}
\toprule
\textbf{Model} & \textbf{Condition} & $\boldsymbol{N}$ & \textbf{Corr.\ Rate} & $\boldsymbol{p}$ \textbf{vs ERM} \\
\midrule
GPT-4 Turbo & ERM causal        & 782 & $58.7\%$ $[55.2, 62.1]$ & --- \\
             & Matched-richness  & 782 & $52.8\%$ $[49.3, 56.3]$ & $0.006$ \\
             & Outcome only      & 782 & $\phantom{0}5.2\%$ $[3.9, 7.0]$   & $<0.001$ \\
\midrule
GPT-5.2     & ERM causal        & 369 & $69.9\%$ $[65.1, 74.4]$ & --- \\
             & Matched-richness  & 369 & $74.5\%$ $[69.8, 78.7]$ & $0.093$ \\
             & Outcome only      & 369 & $52.6\%$ $[47.5, 57.6]$ & $<0.001$ \\
\bottomrule
\end{tabular}
\vspace{-.15in}
\end{table}

For GPT-4 Turbo, ERM's causal vocabulary significantly outperforms the matched-richness control ($+5.9$ pp, $p{=}0.006$): causal targeting is load-bearing. For GPT-5.2, any structured critique suffices ($p{=}0.093$, n.s.). White-box analysis of correction patterns appears in Appendix~\ref{app:erm-whitebox}.

\vspace{-.12in}
\subsection{RQ2: Cross-Episode Signal}
\label{sec:rq4}
\label{sec:rler-experiments}
\label{sec:separation-sim}
\vspace{-.10in}

The single-episode results show \erm corrects reasoning errors within one interaction. We now test whether the epistemic signal extends across episodes. The cross-episode judge operates under the \emph{label-free protocol}: it receives only the scenario and the model's reasoning trace, with no access to correct answers, benchmark rationales, or error-direction cues. Every feedback instance begins: ``The reviewer has NO access to the correct answer.''

\vspace{-.12in}
\paragraph{LLM factorial.}
\emph{Protocols:} well-specified (\textsc{rler}), mis-specified (\textsc{rler\_Bad}), counterfactual-prompt, outcome-only, and no-critique baseline. \emph{Judges:} premium (GPT-5.4-nano) and cheap (GPT-4o); judges generate structural causal critique from the trace alone. \emph{Design:} $2{\times}2$ factorial (judge prompt $\times$ iteration budget) across 4 models; GPT-4o cells run on all 466 Wolf Cases (cases GPT-4o answers incorrectly zero-shot), remaining models on 132-case subsamples. Accuracy: McNemar's paired test; EDGap: 10,000-trial bootstrap CIs.

Two regimes emerge (Table~\ref{tab:rq1-bvsbad}). \emph{Weak policies} (GPT-3.5, Llama): outcome reward already detects the bad protocol ($\Delta\text{acc} \in [0.22, 0.34]$, $p < 0.001$). \emph{Strong policy} (GPT-4o, $N{=}466$): outcome signal is weaker ($\Delta\text{Acc} = +0.077$, $p < 10^{-5}$), but the epistemic axis discriminates sharply via collapse detection (10.3\% well vs.\ 60.3\% bad) and reasoning scores ($\Delta{=}+27.4$). Under the label-free protocol, well-specified critique reduces Rung Collapse from 60.3\% to 10.3\% ($-50$ pp, $p < 10^{-50}$, McNemar) and separates reasoning scores by $+27.4$ points ($p < 10^{-100}$, paired $t$).

\begin{table}[t!]
\vspace{-.06in}
\centering\footnotesize
\caption{$\Delta$Acc (well $-$ mis-specified), premium judge. Bold: $p < 0.001$.}
\vspace{-.10in}
\label{tab:rq1-bvsbad}
\begin{tabular}{l cccc}
\toprule
 & v1/mi1 & v3/mi1 & v1/mi3 & v3/mi3 \\
\midrule
GPT-3.5       & $\mathbf{+.24}$ & $\mathbf{+.30}$ & $\mathbf{+.24}$ & $\mathbf{+.27}$ \\
GPT-4o & $\mathbf{+.07}$ & $\mathbf{+.08}$ & $+.00$ & $+.06$ \\
GPT-5.4n      & $+.08$ & $+.10$ & $+.07$ & $+.03$ \\
Llama-70B     & $\mathbf{+.24}$ & $\mathbf{+.24}$ & $\mathbf{+.34}$ & $\mathbf{+.22}$ \\
\bottomrule
\end{tabular}
\vspace{-.18in}
\end{table}

\vspace{-.12in}
\paragraph{Controlled simulation.}
The LLM factorial is a proxy evaluation because Intervention Independence does not hold in text. To test the separation theorem under its own assumptions, we run CausalBench-Seq (confounded linear-Gaussian SCM, $C{\to}X$, $C{\to}Y$, no $X{\to}Y$; $E{=}500$ episodes, 20 seeds; Appendix~\ref{app:separation-details}). RLVR drives $P(X{\to}Y)$ to 0.99 (wrong); the epistemic-informed \rler controller drives it to 0.01 within ${\sim}50$ episodes, achieving cumulative regret $8.82 \pm 0.14$ vs.\ RLVR's $349.22 \pm 0.00$ ($p < 10^{-8}$; Figure~\ref{fig:separation-sim}), a $38\times$ separation. The \rler controller's regret grows as $\mathcal{O}(\log E)$ inside the Theorem~\ref{thm:separation} envelope vs.\ $\Theta(E)$ for outcome-only. The Epistemic-reset ablation ($66.99 \pm 0.11$, $7.6\times$ worse) confirms cross-episode memory is load-bearing.


\begin{table}[ht]
\vspace{-.06in}
\centering\footnotesize
\caption{Scenario-blind judge: $N{=}102$ Wolf Cases (GPT-4T agent, GPT-5.4-nano judge). Blind judge detects \emph{less} collapse, arguing against answer leakage.}
\vspace{-.06in}
\label{tab:blind-judge}
\begin{tabular}{lccc}
\toprule
\textbf{Judge Condition} & \textbf{Struct.\ Score} & \textbf{Collapse Rate} & \textbf{$\kappa$ vs Full} \\
\midrule
Full (scenario+trace)   & $57.3 \pm 23.6$ & $36.3\%$ & --- \\
Blind (DAG only)        & $75.0 \pm 17.8$ & $13.7\%$ & $-0.05$ \\
Anon-blind ($V_1,V_2,\ldots$) & $80.7 \pm 21.2$ & $11.8\%$ & $\phantom{-}0.03$ \\
\bottomrule
\end{tabular}
\vspace{-.1in}
\end{table}

\vspace{-.10in}
\subsection{RQ3: Judge Integrity}
\label{sec:leakage}
\vspace{-.08in}

Four observations argue structural enforcement, not answer leakage.
\emph{First}, a more capable judge does not improve terminal accuracy; it only reduces iteration cost.
\emph{Second}, the dominant rejection mechanism is Rung Collapse detection (60--95\% of iterations).
\emph{Third}, \textsc{rler\_Bad} receives identical feedback but corrupts reward aggregation; if feedback leaked answers, Bad would benefit equally, yet it underperforms ($\Delta\text{acc} = +0.24$ to $+0.30$, all $p < 10^{-4}$).
\emph{Fourth}, a \textbf{scenario-blind judge} (Table~\ref{tab:blind-judge}) evaluates $N{=}102$ Wolf Cases under full, blind (DAG only), and anonymized-blind ($V_1, V_2, \ldots$) conditions. The blind judge detects \emph{less} collapse (13.7\% vs.\ 36.3\%), indicating scenario comprehension makes the judge \emph{stricter}. Anonymization has negligible effect ($\kappa{=}0.47$), confirming variable names do not leak signal. Details in Appendix~\ref{app:rler-whitebox}.

\vspace{-.1in}
\subsection{RQ4: Ablation Studies}
\label{sec:ablations}
\vspace{-.08in}

We isolate four factors that drive \erm's correction performance (Figure~\ref{fig:ablations}).

\begin{figure}[t!]
\centering
\includegraphics[width=\linewidth,height=0.55\linewidth,keepaspectratio]{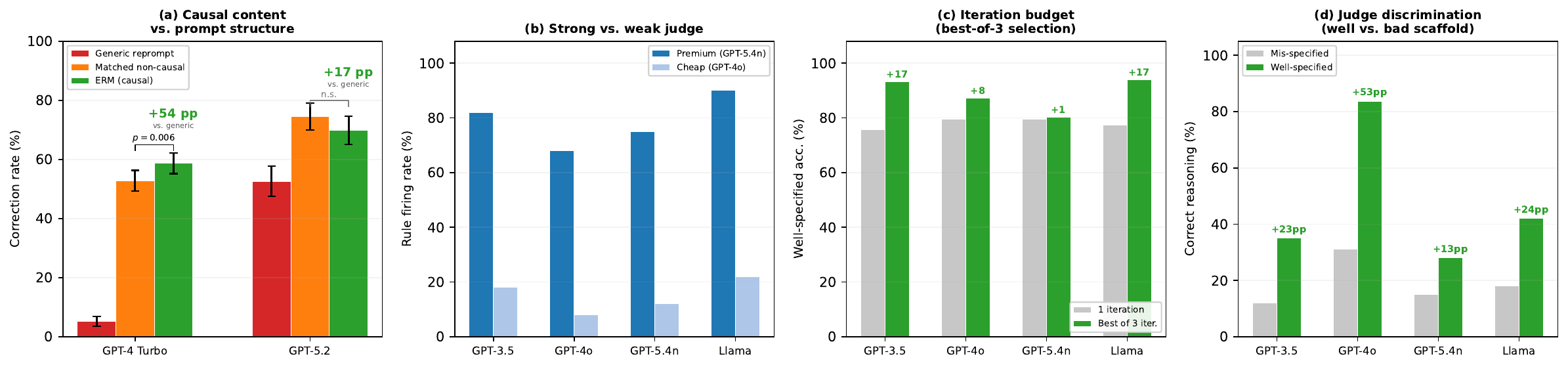}
\vspace{-.18in}
\caption{\textbf{Ablation studies.} \emph{(a)} ERM's causal vocabulary is load-bearing for GPT-4T ($p{=}0.006$); for GPT-5.2, any structured critique suffices. \emph{(b)} Premium vs.\ cheap judge: same terminal accuracy, fewer iterations. \emph{(c)} Best-of-3 iteration budget improves all models. \emph{(d)} Well-specified ERM protocol achieves markedly higher correct reasoning across models.}
\label{fig:ablations}
\vspace{-.18in}
\end{figure}

\vspace{-.1in}
\paragraph{Strong judge vs.\ weak judge.}
Premium (GPT-5.4-nano) and cheap (GPT-4o) judges yield indistinguishable terminal accuracy. The premium judge fires Rung Collapse detection at 60--95\% of iterations vs.\ 5--25\% for the cheap judge (Figure~\ref{fig:wb-rung}), correcting in fewer iterations, but both converge to the same endpoint. Causal content, not judge quality, is the active ingredient.

\vspace{-.1in}
\paragraph{Iteration budget.}
Increasing from 1 to 3 iterations improves endpoint accuracy under the protocol's terminal-selection rule: GPT-3.5 and Llama each gain $+17$ pp, GPT-4o gains $+8$ pp, and GPT-5.4-nano gains $+1$ pp ($N{=}132$). Appendix~\ref{app:bon-probe} shows that the epistemic reasoning score itself is not a standalone outcome-accuracy selector, supporting the orthogonality of epistemic and outcome axes. For stubborn models, additional iterations unlock corrections a single pass misses. For GPT-4o ($N{=}466$), additional iterations also amplify judge discrimination, with collapse-rate separation widening from 50 pp (mi1) to 53 pp (mi3).

\vspace{-.1in}
\paragraph{Outcome vs.\ epistemic reward.}
For weak policies (GPT-3.5, Llama), outcome reward alone detects mis-specification ($\Delta\text{acc} = +0.22$ to $+0.34$). For strong policies (GPT-4o, $N{=}466$), the epistemic axis discriminates sharply: collapse detection separates protocols by $50$ pp and reasoning scores by $+27.4$ points, while outcome reward reaches significance only at full sample ($\Delta\text{Acc} = +0.077$, $p < 10^{-5}$). The epistemic axis detects flaws at smaller sample sizes.

\vspace{-.1in}
\paragraph{Label-free validation.}
A six-condition ablation (Appendix~\ref{app:exp2b}) tests fully label-free critique on both stubborn models. Label-free critique without an error-direction cue (C4) achieves 0\% false-flips for both GPT-4T and GPT-5.2, replicated in an independent gold-free Exp~B run ($0/49$ false-flips vs.\ $4/49 = 8.2\%$ for outcome-only). For GPT-4T, causal vocabulary is load-bearing: C4 vs.\ C5 = $55.9\%$ vs.\ $34.5\%$ ($+21.4$ pp, $p{=}0.0001$). The error-direction cue drives the large Table~\ref{tab:collapse} gains but carries a false-flip cost (10--31\%); label-free critique trades peak recovery for zero false-flips. The cross-episode protocol (\S\ref{sec:rler-experiments}) independently confirms label-free operation.

\vspace{-.1in}
\paragraph{DAG-First structural output.}
Requiring explicit DAG output before the conclusion (Appendix~\ref{app:dagfirst}) has model-dependent effects: Gemini~2.5~Flash $+4.8$~pp, Llama~3.3~70B $+4.4$~pp, GPT-5.2 $+2.0$~pp (near ceiling), Claude Sonnet~4.5 $+0.0$~pp (at ceiling). GPT-3.5-Turbo \emph{loses} $-11.7$~pp, producing the DAG mechanically without using it as a scaffold. DAG compliance is ${\geq}99\%$. This motivates model-adaptive deployment: DAG-First for reasoning-capable models, plain critique otherwise.

\vspace{-.1in}
\paragraph{Baseline suite.}
On GPT-4T Wolf Cases ($N{=}170$; Appendix~\ref{app:baseline-suite}), ERM achieves the highest recovery (64.1\%) and lowest residual Rung Collapse (4.1\%). Self-consistency~\citep{wang2023selfconsistency}, Best-of-$N$~\citep{snell2024bon_scaling}, and Self-Refine~\citep{madaan2023selfrefine} \emph{underperform} outcome-only reprompting (30.8\%, 15.9\%, and 14.8\% vs.\ 52.9\%; Appendix~\ref{app:baseline-suite}), confirming that standard test-time methods preserve $\mathcal{L}_1$ shortcuts on causal tasks.

\vspace{-.1in}
\paragraph{Why this is not just stronger prompting.}
ERM is not explained by prompt length or authority: the matched-richness control preserves structure while removing causal vocabulary, yet ERM remains better for GPT-4T ($p=0.006$). Scenario-blind judging reduces, rather than improves, collapse detection, and label-free critique without the error-direction cue yields 0\% false-flips. ERM audits exposed causal traces, the accountability layer available to users and downstream systems.

\vspace{-.1in}
\section{Conclusion and Limitations}
\label{sec:discussion}
\label{sec:conclusion}
\vspace{-.06in}

ERM addresses a blind spot in outcome-based post-training: reward observes what a model concludes, not whether its reasoning is causally valid. Grounded in interventional theory and validated across CausalT5K, CLadder, ablations, and a controlled CausalBench-Seq simulation, ERM provides a label-free mechanism for detecting and correcting structural reasoning flaws, reducing residual Rung Collapse to 4\% where standard test-time methods leave 55--70\%. Its current limits are clear: the LLM experiments are text-query proxy evaluations, the reward-to-policy loop closes only in the controlled SCM simulation, judge circularity is bounded but not eliminated by scenario-blind tests, and ERM audits exposed reasoning traces rather than guaranteed faithful internal computation~\citep{lanham2023measuring}. Full deployment requires calibrated triggers, stronger trace-faithfulness evidence, and live intervention channels.\label{sec:future}

\bibliographystyle{plainnat}
\bibliography{RLER_merged_fixed_audited}

@inproceedings{acharya2025causcibench,
  author       = {Sawal Acharya and Terry Jingchen Zhang and Andrew Kim and Anahita Haghighat and Xianlin Sun and Rahul Babu Shrestha and Maximilian Mordig and Furkan Danisman and Clijo Jose and Yahang Qi and Pepijn Cobben and Bernhard Sch{\"o}lkopf and Mrinmaya Sachan and Zhijing Jin},
  title        = {{CauSciBench}: Assessing {LLM} Causal Reasoning for Scientific Research},
  booktitle    = {NeurIPS 2025 Workshop on Uncovering Causality in Science},
  year         = {2025},
  url          = {https://openreview.net/forum?id=EO8mTLqDuT},
  note         = {Workshop paper},
}

@article{nam2026causaljepa,
  title={Causal-JEPA: Learning World Models through Object-Level Latent Interventions},
  author={Nam, Heejeong and Le Lidec, Quentin and Maes, Lucas and LeCun, Yann and Balestriero, Randall},
  journal={arXiv preprint arXiv:2602.11389},
  year={2026},
  url={https://arxiv.org/abs/2602.11389}
}

@article{agm1985,
  author       = {Carlos E. Alchourr\'{o}n and Peter G\"{a}rdenfors and David Makinson},
  title        = {On the Logic of Theory Change: Partial Meet Contraction and Revision Functions},
  journal      = {Journal of Symbolic Logic},
  volume       = {50},
  number       = {2},
  pages        = {510--530},
  year         = {1985},
  doi          = {10.2307/2274239},
}

@article{ban2025harmonized,
  author       = {Taiyu Ban and Lyuzhou Chen and Derui Lyu and Xiangyu Wang and Qinrui Zhu and Huanhuan Chen},
  title        = {LLM-Driven Causal Discovery via Harmonized Prior},
  journal      = {IEEE Transactions on Knowledge and Data Engineering},
  volume       = {37},
  number       = {4},
  pages        = {1943--1960},
  year         = {2025},
  doi          = {10.1109/TKDE.2025.3528461},
}

@incollection{bareinboim2022fusion,
  author       = {Elias Bareinboim and Juan Correa and Duligur Ibeling and Thomas Icard},
  title        = {On Pearl's Hierarchy and the Foundations of Causal Inference},
  booktitle    = {Probabilistic and Causal Inference: the Works of Judea Pearl},
  pages        = {507--556},
  publisher    = {ACM Books},
  year         = {2022},
}

@inproceedings{bombari2025spurious,
  author       = {Simone Bombari and Marco Mondelli},
  title        = {Spurious Correlations in High Dimensional Regression: The Roles of Regularization, Simplicity Bias and Over-Parameterization},
  booktitle    = {Proceedings of the 42nd International Conference on Machine Learning},
  series       = {Proceedings of Machine Learning Research},
  volume       = {267},
  pages        = {4839--4873},
  publisher    = {PMLR},
  year         = {2025},
}

@misc{causalT5K,
  author       = {Longling Geng and Andy Ouyang and Theodore Wu and Daphne Barretto and Matthew John Hayes and Rachael Cooper and Yuqiao Zeng and Sameer Vijay and Gia Ancone and Ankit Rai and Matthew Wolfman and Patrick Flanagan and Edward Y. Chang},
  title        = {{CausalT5K}: Diagnosing and Informing Refusal for Trustworthy Causal Reasoning of Skepticism, Sycophancy, Detection-Correction, and Rung Collapse},
  year         = {2026},
  doi          = {10.48550/arXiv.2602.08939},
  eprint       = {2602.08939},
  archivePrefix= {arXiv},
  primaryClass = {cs.AI},
  url          = {https://arxiv.org/abs/2602.08939},
  note         = {arXiv preprint},
}

@inproceedings{causalgraphbench2025,
  author       = {Nikolay Babakov and Ehud Reiter and Alberto Bugar{\'i}n-Diz},
  title        = {{CausalGraphBench}: a Benchmark for Evaluating Language Models Capabilities of Causal Graph Discovery},
  booktitle    = {Proceedings of the 63rd Annual Meeting of the Association for Computational Linguistics (Volume 4: Student Research Workshop)},
  pages        = {240--258},
  address      = {Vienna, Austria},
  publisher    = {Association for Computational Linguistics},
  year         = {2025},
  doi          = {10.18653/v1/2025.acl-srw.16},
}

@inproceedings{chi2024unveiling,
  author       = {Haoang Chi and He Li and Wenjing Yang and Feng Liu and Long Lan and Xiaoguang Ren and Tongliang Liu and Bo Han},
  title        = {Unveiling Causal Reasoning in Large Language Models: Reality or Mirage?},
  booktitle    = {Advances in Neural Information Processing Systems 37},
  year         = {2024},
  doi          = {10.52202/079017-3064},
  url          = {https://proceedings.neurips.cc/paper_files/paper/2024/hash/af2bb2b2280d36f8842e440b4e275152-Abstract-Conference.html},
}

@inproceedings{zhang2025conditional_reward,
  author       = {Zheng Zhang and Ziwei Shan and Kaitao Song and Yexin Li and Kan Ren},
  title        = {Linking Process to Outcome: Conditional Reward Modeling for {LLM} Reasoning},
  booktitle    = {The Fourteenth International Conference on Learning Representations},
  year         = {2026},
  eprint       = {2509.26578},
  archivePrefix= {arXiv},
  primaryClass = {cs.LG},
  url          = {https://openreview.net/forum?id=4DJoBOQNd0},
}

@article{damour2022underspecification,
  author       = {Alexander D'Amour and Katherine Heller and Dan Moldovan and Ben Adlam and Babak Alipanahi and Alex Beutel and Christina Chen and Jonathan Deaton and Jacob Eisenstein and Matthew D. Hoffman and Farhad Hormozdiari and Neil Houlsby and Shaobo Hou and Ghassen Jerfel and Alan Karthikesalingam and Mario Lucic and Yian Ma and Cory McLean and Diana Mincu and Akinori Mitani and Andrea Montanari and Zachary Nado and Vivek Natarajan and Christopher Nielson and Thomas F. Osborne and Rajiv Raman and Kim Ramasamy and Rory Sayres and Jessica Schrouff and Martin Seneviratne and Shannon Sequeira and Harini Suresh and Victor Veitch and Max Vladymyrov and Xuezhi Wang and Kellie Webster and Steve Yadlowsky and Taedong Yun and Xiaohua Zhai and D. Sculley},
  title        = {Underspecification Presents Challenges for Credibility in Modern Machine Learning},
  journal      = {Journal of Machine Learning Research},
  volume       = {23},
  number       = {226},
  pages        = {1--61},
  year         = {2022},
}

@article{dawid1984prequential,
  author       = {A. Philip Dawid},
  title        = {Statistical Theory: The Prequential Approach},
  journal      = {Journal of the Royal Statistical Society: Series A},
  volume       = {147},
  number       = {2},
  pages        = {278--292},
  year         = {1984},
}

@inproceedings{du2025llmcd,
  author       = {Huaming Du and Yujia Zheng and Baoyu Jing and Yu Zhao and Gang Kou and Guisong Liu and Tao Gu and Weimin Li and Carl Yang},
  title        = {Causal Discovery through Synergizing Large Language Model and Data-Driven Reasoning},
  booktitle    = {Proceedings of the 31st ACM SIGKDD Conference on Knowledge Discovery and Data Mining},
  year         = {2025},
  doi          = {10.1145/3711896.3736874},
}

@inproceedings{elahi2024partial,
  author       = {Muhammad Qasim Elahi and Mahsa Ghasemi and Murat Kocaoglu},
  title        = {Partial Structure Discovery is Sufficient for No-regret Learning in Causal Bandits},
  booktitle    = {Advances in Neural Information Processing Systems 37},
  year         = {2024},
  eprint       = {2411.04054},
  archivePrefix= {arXiv},
  primaryClass = {cs.LG},
}

@article{foster1998asymptotic,
  author       = {Dean P. Foster and Rakesh V. Vohra},
  title        = {Asymptotic Calibration},
  journal      = {Biometrika},
  volume       = {85},
  number       = {2},
  pages        = {379--390},
  year         = {1998},
}

@article{geirhos2020shortcut,
  author       = {Robert Geirhos and J\"{o}rn-Henrik Jacobsen and Claudio Michaelis and Richard Zemel and Wieland Brendel and Matthias Bethge and Felix A. Wichmann},
  title        = {Shortcut Learning in Deep Neural Networks},
  journal      = {Nature Machine Intelligence},
  volume       = {2},
  number       = {11},
  pages        = {665--673},
  year         = {2020},
}

@misc{intervene2026,
  author       = {Shaojie Shi and Zhengyu Shi and Lingran Zheng and Xinyu Su and Anna Xie and Bohao Lv and Rui Xu and Zijian Chen and Zhichao Chen and Guolei Liu and Naifu Zhang and Mingjian Dong and Zhuo Quan and Bohao Chen and Teqi Hao and Yuan Qi and Yinghui Xu and Libo Wu},
  title        = {InterveneBench: Benchmarking LLMs for Intervention Reasoning and Causal Study Design in Real Social Systems},
  year         = {2026},
  doi          = {10.48550/arXiv.2603.15542},
  eprint       = {2603.15542},
  archivePrefix= {arXiv},
  primaryClass = {cs.AI},
  note         = {arXiv preprint},
}

@inproceedings{jin2023cladder,
  author       = {Zhijing Jin and Yuen Chen and Felix Leeb and Luigi Gresele and Ojasv Kamal and Zhiheng Lyu and Kevin Blin and Fernando Gonzalez Adauto and Max Kleiman-Weiner and Mrinmaya Sachan and Bernhard Sch{\"o}lkopf},
  title        = {CLadder: Assessing Causal Reasoning in Language Models},
  booktitle    = {Advances in Neural Information Processing Systems 36},
  volume       = {36},
  pages        = {31038--31065},
  year         = {2023},
}

@inproceedings{lattimore2016causal,
  author       = {Finnian Lattimore and Tor Lattimore and Mark D. Reid},
  title        = {Causal Bandits: Learning Good Interventions via Causal Inference},
  booktitle    = {NeurIPS},
  year         = {2016},
}

@inproceedings{lightman2023prm,
  author       = {Hunter Lightman and Vineet Kosaraju and Yuri Burda and Harrison Edwards and Bowen Baker and Teddy Lee and Jan Leike and John Schulman and Ilya Sutskever and Karl Cobbe},
  title        = {Let's Verify Step by Step},
  booktitle    = {International Conference on Learning Representations},
  year         = {2024},
}

@inproceedings{mccoy2019right,
  author       = {R. Thomas McCoy and Ellie Pavlick and Tal Linzen},
  title        = {Right for the Wrong Reasons: Diagnosing Syntactic Heuristics in Natural Language Inference},
  booktitle    = {Proceedings of the 57th Annual Meeting of the Association for Computational Linguistics},
  pages        = {3428--3448},
  publisher    = {Association for Computational Linguistics},
  year         = {2019},
  doi          = {10.18653/v1/P19-1334},
}

@inproceedings{baniharouni2026rewarding_doubt,
title={Rewarding Doubt: A Reinforcement Learning Approach to Calibrated Confidence Expression of Large Language Models},
author={David Bani-Harouni and Chantal Pellegrini and Paul Stangel and Ege {\"O}zsoy and Kamilia Zaripova and Nassir Navab and Matthias Keicher},
booktitle={The Fourteenth International Conference on Learning Representations},
year={2026},
url={https://openreview.net/forum?id=yResLmrVO1}
}

@article{murphy1973vector,
  author       = {Allan H. Murphy},
  title        = {A New Vector Partition of the Probability Score},
  journal      = {Journal of Applied Meteorology},
  volume       = {12},
  number       = {4},
  pages        = {595--600},
  year         = {1973},
}

@book{pearl2009causality,
  author       = {Judea Pearl},
  title        = {Causality: Models, Reasoning, and Inference},
  publisher    = {Cambridge University Press},
  edition      = {2nd},
  year         = {2009},
}

@inproceedings{rlvr2025,
  author       = {Xumeng Wen and Zihan Liu and Shun Zheng and Shengyu Ye and Zhirong Wu and Yang Wang and Zhijian Xu and Xiao Liang and Junjie Li and Ziming Miao and Jiang Bian and Mao Yang},
  title        = {Reinforcement Learning with Verifiable Rewards Implicitly Incentivizes Correct Reasoning in Base {LLMs}},
  booktitle    = {The Fourteenth International Conference on Learning Representations},
  year         = {2026},
  eprint       = {2506.14245},
  archivePrefix= {arXiv},
  primaryClass = {cs.LG},
  url          = {https://openreview.net/forum?id=jGbRWwIidy},
}

@inproceedings{shah2020pitfalls,
  author       = {Harshay Shah and Kaustav Tamuly and Aditi Raghunathan and Prateek Jain and Praneeth Netrapalli},
  title        = {The Pitfalls of Simplicity Bias in Neural Networks},
  booktitle    = {NeurIPS},
  year         = {2020},
}

@inproceedings{sharma2024sycophancy,
  author       = {Mrinank Sharma and Meg Tong and Tomasz Korbak and David Duvenaud and Amanda Askell and Samuel R. Bowman and Newton Cheng and Esin Durmus and Zac Hatfield-Dodds and Scott R. Johnston and Shauna Kravec and Timothy Maxwell and Sam McCandlish and Kamal Ndousse and Oliver Rausch and Nicholas Schiefer and Da Yan and Miranda Zhang and Ethan Perez},
  title        = {Towards Understanding Sycophancy in Language Models},
  booktitle    = {International Conference on Learning Representations},
  year         = {2024},
}

@book{spirtes1993causation,
  author       = {Spirtes, Peter and Glymour, Clark and Scheines, Richard},
  title        = {Causation, Prediction, and Search},
  publisher    = {Springer-Verlag},
  year         = {1993},
}

@inproceedings{wan2025llmcd_survey,
  author       = {Guangya Wan and Yunsheng Lu and Yuqi Wu and Mengxuan Hu and Sheng Li},
  title        = {Large Language Models for Causal Discovery: Current Landscape and Future Directions},
  booktitle    = {Proceedings of the Thirty-Fourth International Joint Conference on Artificial Intelligence},
  pages        = {10687--10695},
  year         = {2025},
  doi          = {10.24963/ijcai.2025/1186},
}

@misc{wang2025causal_reward,
  author       = {Chaoqi Wang and Zhuokai Zhao and Yibo Jiang and Zhaorun Chen and Chen Zhu and Yuxin Chen and Jiayi Liu and Lizhu Zhang and Xiangjun Fan and Hao Ma and Sinong Wang},
  title        = {Beyond Reward Hacking: Causal Rewards for Large Language Model Alignment},
  year         = {2025},
  doi          = {10.48550/arXiv.2501.09620},
  eprint       = {2501.09620},
  archivePrefix= {arXiv},
  primaryClass = {cs.LG},
  note         = {arXiv preprint},
}

@inproceedings{yan2024linear,
  author       = {Zirui Yan and Ali Tajer},
  title        = {Linear Causal Bandits: Unknown Graph and Soft Interventions},
  booktitle    = {NeurIPS},
  year         = {2024},
}

@misc{yang2024critical,
  author       = {Linying Yang and Vik Shirvaikar and Oscar Clivio and Fabian Falck},
  title        = {A Critical Review of Causal Reasoning Benchmarks for Large Language Models},
  year         = {2024},
  doi          = {10.48550/arXiv.2407.08029},
  eprint       = {2407.08029},
  archivePrefix= {arXiv},
  primaryClass = {cs.CL},
  note         = {arXiv preprint},
}

@inproceedings{yu2025causaleval,
  author       = {Longxuan Yu and Delin Chen and Siheng Xiong and Qingyang Wu and Dawei Li and Zhikai Chen and Xiaoze Liu and Liangming Pan},
  title        = {{CausalEval}: Towards Better Causal Reasoning in Language Models},
  booktitle    = {Proceedings of the 2025 Conference of the Nations of the Americas Chapter of the Association for Computational Linguistics: Human Language Technologies (Volume 1: Long Papers)},
  pages        = {12512--12540},
  address      = {Albuquerque, New Mexico},
  publisher    = {Association for Computational Linguistics},
  year         = {2025},
  doi          = {10.18653/v1/2025.naacl-long.622},
}

@article{zevcevic2023causal,
  author       = {Matej Ze\v{c}evi\'{c} and Moritz Willig and Devendra Singh Dhami and Kristian Kersting},
  title        = {Causal Parrots: Large Language Models May Talk Causality but Are Not Causal},
  journal      = {Transactions on Machine Learning Research},
  year         = {2023},
}

@article{zhang2008completeness,
  author       = {Jiji Zhang},
  title        = {On the Completeness of Orientation Rules for Causal Discovery in the Presence of Latent Confounders and Selection Bias},
  journal      = {Artificial Intelligence},
  volume       = {172},
  number       = {16--17},
  pages        = {1873--1896},
  year         = {2008},
}

@misc{zhou2024causalbench,
  author       = {Yu Zhou and Xingyu Wu and Beicheng Huang and Jibin Wu and Liang Feng and Kay Chen Tan},
  title        = {{CausalBench}: A Comprehensive Benchmark for Causal Learning Capability of Large Language Models},
  year         = {2024},
  doi          = {10.48550/arXiv.2404.06349},
  eprint       = {2404.06349},
  archivePrefix= {arXiv},
  primaryClass = {cs.AI},
  note         = {arXiv preprint},
}

@inproceedings{christiano2017rlhf,
  author    = {Paul Christiano and Jan Leike and Tom Brown and Miljan Martic and Shane Legg and Dario Amodei},
  title     = {Deep Reinforcement Learning from Human Preferences},
  booktitle = {NeurIPS},
  year      = {2017},
}

@misc{cobbe2021gsm8k,
  author    = {Karl Cobbe and Vineet Kosaraju and Mohammad Bavarian and Mark Chen and Heewoo Jun and Lukasz Kaiser and Matthias Plappert and Jerry Tworek and Jacob Hilton and Reiichiro Nakano and Christopher Hesse and John Schulman},
  title     = {Training Verifiers to Solve Math Word Problems},
  year      = {2021},
  eprint    = {2110.14168},
  archivePrefix = {arXiv},
  primaryClass  = {cs.LG},
  note      = {arXiv preprint},
}

@article{deepseekr1_2025,
  author       = {Daya Guo and Dejian Yang and Haowei Zhang and Junxiao Song and Peiyi Wang and Qihao Zhu and Runxin Xu and Ruoyu Zhang and Shirong Ma and Xiao Bi and Xiaokang Zhang and Xingkai Yu and Yu Wu and Z. F. Wu and Zhibin Gou and Zhihong Shao and Zhuoshu Li and Ziyi Gao and others},
  title        = {{DeepSeek-R1} Incentivizes Reasoning in {LLMs} through Reinforcement Learning},
  journal      = {Nature},
  volume       = {645},
  pages        = {633--638},
  year         = {2025},
  doi          = {10.1038/s41586-025-09422-z},
}

@inproceedings{ouyang2022rlhf,
  author    = {Long Ouyang and Jeff Wu and Xu Jiang and Diogo Almeida and Carroll L. Wainwright and Pamela Mishkin and Chong Zhang and Sandhini Agarwal and Katarina Slama and Alex Ray and others},
  title     = {Training Language Models to Follow Instructions with Human Feedback},
  booktitle = {NeurIPS},
  year      = {2022},
}

@article{russo2018tutorial,
  author    = {Daniel Russo and Benjamin Van Roy and Abbas Kazerouni and Ian Osband and Zheng Wen},
  title     = {A Tutorial on {Thompson} Sampling},
  journal   = {Foundations and Trends in Machine Learning},
  volume    = {11},
  number    = {1},
  pages     = {1--96},
  year      = {2018},
  doi       = {10.1561/2200000070},
}

@article{sagallm2025,
  author       = {Edward Y. Chang and Longling Geng},
  title        = {{SagaLLM}: Context Management, Validation, and Transaction Guarantees for Multi-Agent {LLM} Planning},
  journal      = {Proceedings of the VLDB Endowment},
  volume       = {18},
  number       = {12},
  pages        = {4874--4886},
  year         = {2025},
  doi          = {10.14778/3750601.3750611},
}

@misc{schulman2017ppo,
  author    = {John Schulman and Filip Wolski and Prafulla Dhariwal and Alec Radford and Oleg Klimov},
  title     = {Proximal Policy Optimization Algorithms},
  year      = {2017},
  eprint    = {1707.06347},
  archivePrefix = {arXiv},
  primaryClass  = {cs.LG},
  note      = {arXiv preprint},
}

@misc{shao2024grpo,
  author    = {Zhihong Shao and Peiyi Wang and Qihao Zhu and Runxin Xu and Junxiao Song and Xiao Bi and Haowei Zhang and Mingchuan Zhang and Y. K. Li and Y. Wu and Daya Guo},
  title     = {{DeepSeekMath}: Pushing the Limits of Mathematical Reasoning in Open Language Models},
  year      = {2024},
  eprint    = {2402.03300},
  archivePrefix = {arXiv},
  primaryClass  = {cs.CL},
}

@misc{snell2024bon_scaling,
  author    = {Charlie Snell and Jaehoon Lee and Kelvin Xu and Aviral Kumar},
  title     = {Scaling {LLM} Test-Time Compute Optimally Can Be More Effective Than Scaling Model Parameters},
  year      = {2024},
  eprint    = {2408.03314},
  archivePrefix = {arXiv},
  primaryClass  = {cs.LG},
}

@misc{uesato2022prm,
  author    = {Jonathan Uesato and Nate Kushman and Ramana Kumar and Francis Song and Noah Siegel and Lisa Wang and Antonia Creswell and Geoffrey Irving and Irina Higgins},
  title     = {Solving Math Word Problems with Process- and Outcome-Based Feedback},
  year      = {2022},
  eprint    = {2211.14275},
  archivePrefix = {arXiv},
  primaryClass  = {cs.LG},
}

@inproceedings{wang2023selfconsistency,
  author    = {Xuezhi Wang and Jason Wei and Dale Schuurmans and Quoc Le and Ed Chi and Sharan Narang and Aakanksha Chowdhery and Denny Zhou},
  title     = {Self-Consistency Improves Chain of Thought Reasoning in Language Models},
  booktitle = {International Conference on Learning Representations},
  year      = {2023},
}

@inproceedings{wei2022cot,
  author    = {Jason Wei and Xuezhi Wang and Dale Schuurmans and Maarten Bosma and Brian Ichter and Fei Xia and Ed Chi and Quoc Le and Denny Zhou},
  title     = {Chain-of-Thought Prompting Elicits Reasoning in Large Language Models},
  booktitle = {NeurIPS},
  year      = {2022},
}

@article{wilson1927probable,
  author    = {Edwin B. Wilson},
  title     = {Probable Inference, the Law of Succession, and Statistical Inference},
  journal   = {Journal of the American Statistical Association},
  volume    = {22},
  number    = {158},
  pages     = {209--212},
  year      = {1927},
}

@inproceedings{zhang2020causal,
  author    = {Junzhe Zhang and Daniel Kumor and Elias Bareinboim},
  title     = {Causal Imitation Learning With Unobserved Confounders},
  booktitle = {Advances in Neural Information Processing Systems 33},
  year      = {2020},
}

@inproceedings{madaan2023selfrefine,
  author    = {Aman Madaan and Niket Tandon and Prakhar Gupta and Skyler Hallinan and Luyu Gao and Sarah Wiegreffe and Uri Alon and Nouha Dziri and Shrimai Prabhumoye and Yiming Yang and Shashank Gupta and Bodhisattwa Prasad Majumder and Katherine Hermann and Sean Welleck and Amir Yazdanbakhsh and Peter Clark},
  title     = {Self-Refine: Iterative Refinement with Self-Feedback},
  booktitle = {NeurIPS},
  year      = {2023}
}

@misc{lanham2023measuring,
  author  = {Tamera Lanham and Anna Chen and Ansh Radhakrishnan and Benoit Steiner and Carson Denison and Danny Hernandez and Dustin Li and Esin Durmus and Evan Hubinger and Jackson Kernion and Kamil\.{e} Luko\v{s}i\={u}t\.{e} and Karina Nguyen and Newton Cheng and Nicholas Joseph and Nicholas Schiefer and Oliver Rausch and Robin Larson and Sam McCandlish and Sandipan Kundu and Saurav Kadavath and Shannon Yang and Thomas Henighan and Timothy Maxwell and Timothy Telleen-Lawton and Tristan Hume and Zac Hatfield-Dodds and Jared Kaplan and Jan Brauner and Samuel R. Bowman and Ethan Perez},
  title   = {Measuring Faithfulness in Chain-of-Thought Reasoning},
  year    = {2023},
  eprint  = {2307.13702},
  archivePrefix = {arXiv},
  primaryClass  = {cs.CL},
  note    = {arXiv preprint},
}

@inproceedings{wei2025mmupt,
  author    = {Lai Wei and Yuting Li and Chen Wang and Yue Wang and Linghe Kong and Weiran Huang and Lichao Sun},
  title     = {First {SFT}, Second {RL}, Third {UPT}: Continual Improving Multi-Modal {LLM} Reasoning via Unsupervised Post-Training},
  booktitle = {The Thirty-ninth Annual Conference on Neural Information Processing Systems},
  year      = {2025},
  url       = {https://openreview.net/forum?id=HL1j92hb6z},
}

@inproceedings{chang2026sycophancy,
  author    = {Anonymous},
  title     = {Diagnosing and Mitigating Sycophancy and Skepticism in {LLM} Causal Judgment},
  booktitle = {Proceedings of ACL},
  year      = {2026},
  note      = {To appear; anonymized for double-blind review},
}

\newpage
\appendix

\section*{Appendix Table of Contents}
\vspace{-.08in}
\begin{itemize}[nosep,leftmargin=1.5em,itemsep=0.38em]
\item[\ref{app:illustrated-examples}.] Illustrated ERM Examples \dotfill \pageref{app:illustrated-examples}
\item[\ref{app:notation}.] Notation and Terminology \dotfill \pageref{app:notation}
\item[\ref{app:proofs}.] Proofs \dotfill \pageref{app:proofs}
\item[\ref{app:foundations}.] Foundational Anchors (Expanded) \dotfill \pageref{app:foundations}
\item[\ref{app:algorithm}.] Algorithm Details \dotfill \pageref{app:algorithm}
\item[\ref{app:additional-proofs}.] Additional Proofs \dotfill \pageref{app:additional-proofs}
\item[\ref{app:pressure}.] Adversarial Pressure Protocol on CausalL2 \dotfill \pageref{app:pressure}
\item[\ref{app:iatrogenic}.] Iatrogenic Critique: Per-Model Routing Evidence \dotfill \pageref{app:iatrogenic}
\item[\ref{app:trace-examples}.] Trace-to-DAG Extraction Examples \dotfill \pageref{app:trace-examples}
\item[\ref{app:edgap-derivation}.] EDGap Derivation and Properties \dotfill \pageref{app:edgap-derivation}
\item[\ref{app:cladder}.] Cross-Benchmark Validation: CLadder $\mathcal{L}_2$ \dotfill \pageref{app:cladder}
\item[\ref{app:reproducibility}.] Reproducibility and Experimental Setup \dotfill \pageref{app:reproducibility}
\item[\ref{app:baseline-suite}.] Baseline Suite: ERM vs.\ Standard Reasoning Baselines \dotfill \pageref{app:baseline-suite}
\item[\ref{app:stubborn-taxonomy}.] Cross-Model Stubborn-Case Taxonomy \dotfill \pageref{app:stubborn-taxonomy}
\item[\ref{app:erm-whitebox}.] ERM White-Box Analysis: Stubborn Cases \dotfill \pageref{app:erm-whitebox}
\item[\ref{app:rler-whitebox}.] RLER White-Box Analysis \dotfill \pageref{app:rler-whitebox}
\begin{itemize}[nosep,leftmargin=2.5em]
\item[\ref{app:bon-probe}.] Retrospective Best-of-N Probe \dotfill \pageref{app:bon-probe}
\item[\ref{app:false-flip}.] False-Flip Analysis \dotfill \pageref{app:false-flip}
\item[\ref{app:exp2b}.] Label-Free Single-Episode Ablation (Exp 2b) \dotfill \pageref{app:exp2b}
\item[\ref{app:dagfirst}.] DAG-First Structural Output Ablation (Exp D1) \dotfill \pageref{app:dagfirst}
\item[\ref{app:blind-judge}.] Scenario-Blind Judge Analysis \dotfill \pageref{app:blind-judge}
\item[\ref{app:separation-details}.] Separation Simulation Details \dotfill \pageref{app:separation-details}
\end{itemize}
\item[\ref{app:related}.] Extended Related Work \dotfill \pageref{app:related}
\item[\ref{app:cross-experiment}.] Cross-Experiment Analysis: Detection vs.\ Correction \dotfill \pageref{app:cross-experiment}
\item[\ref{app:agm-proof}.] Formal Proof: ERM as AGM-Style Belief Contraction \dotfill \pageref{app:agm-proof}
\item[\ref{app:architecture}.] Architecture Diagram \dotfill \pageref{app:architecture}
\item[\ref{app:supp-figs}.] Supplementary Figures \dotfill \pageref{app:supp-figs}
\item[\ref{app:cross-seed}.] Cross-Seed Replication of RLER Thesis Cells \dotfill \pageref{app:cross-seed}
\item[\ref{app:trace-transcripts}.] Representative Trace Transcripts \dotfill \pageref{app:trace-transcripts}
\end{itemize}

\section{Illustrated ERM Examples}
\label{app:illustrated-examples}

Figures~\ref{fig:illust-robot} and~\ref{fig:illust-triage} illustrate why feedback that only reports \emph{what} went wrong is fundamentally insufficient, and how adding \emph{why} changes the picture. In each scenario, outcome-only reward correctly identifies the failure but induces a superficial rule (``avoid red shelves,'' ``deprioritize rural transfers'') because the spurious feature and the true cause are equally consistent with the observed outcomes. The agent cannot distinguish correlation from causation using outcome signal alone. Causal reasoning breaks this deadlock: by identifying the hidden confounder (floor vibration, transfer delay), it enables a cause-aware policy that generalizes correctly rather than overfitting to a surface feature. Each figure follows a four-panel structure: (1)~observed failures, (2)~the wrong rule induced by outcome-only reward, (3)~causal discovery identifying the true confounder, and (4)~the corrected cause-aware policy.

\begin{figure}[h]
\centering
\includegraphics[width=0.75\linewidth]{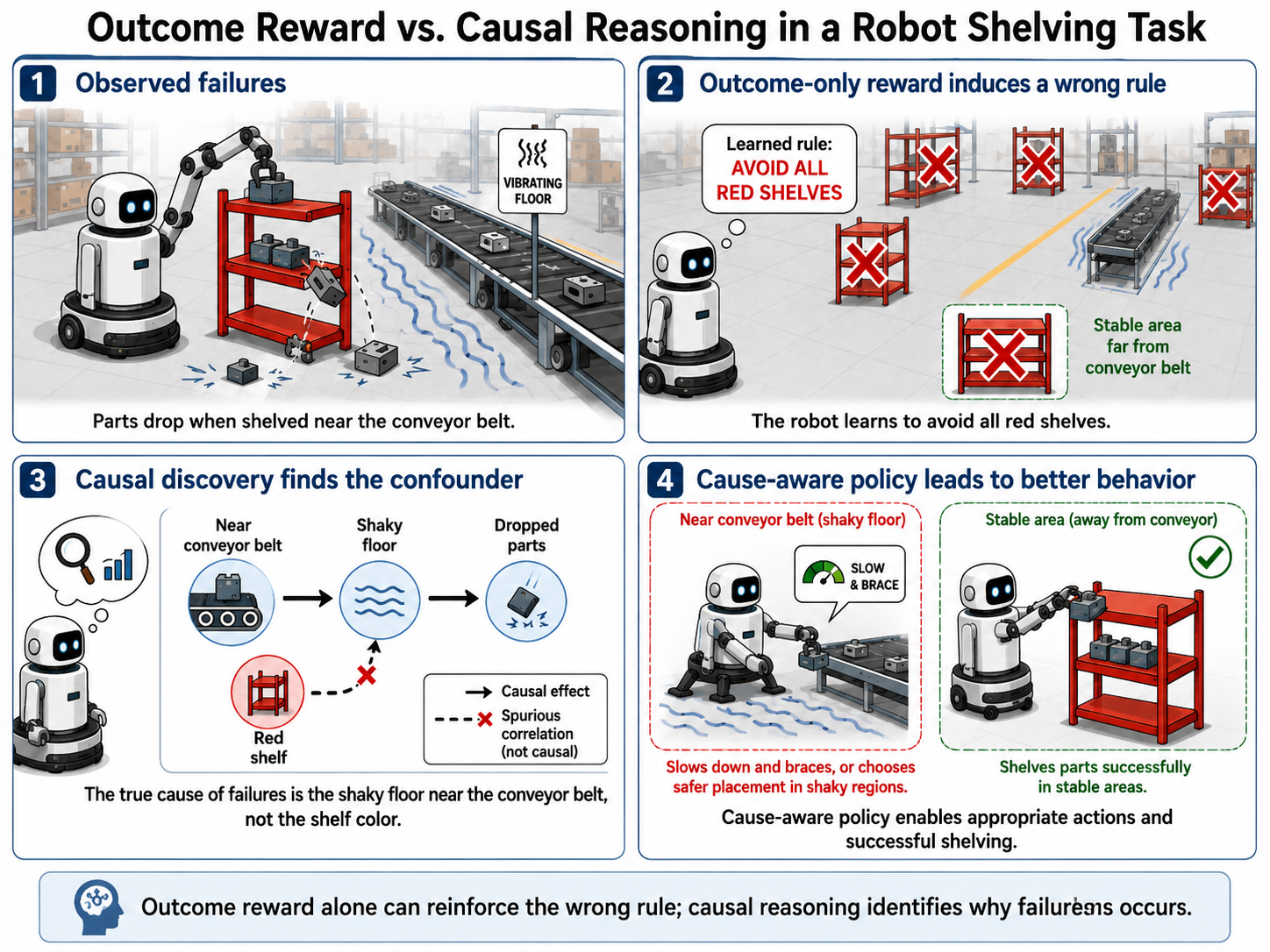}
\caption{\textbf{Outcome reward vs.\ causal reasoning: robot shelving.} Outcome-only reward induces the wrong rule (``avoid red shelves'') because shelf color is spuriously correlated with failures. Causal discovery identifies the true driver (floor vibration near the conveyor belt), enabling a cause-aware policy that adapts to the environment rather than avoiding a surface feature.}
\label{fig:illust-robot}
\end{figure}
\begin{figure}[H]
\centering
\includegraphics[width=0.75\linewidth]{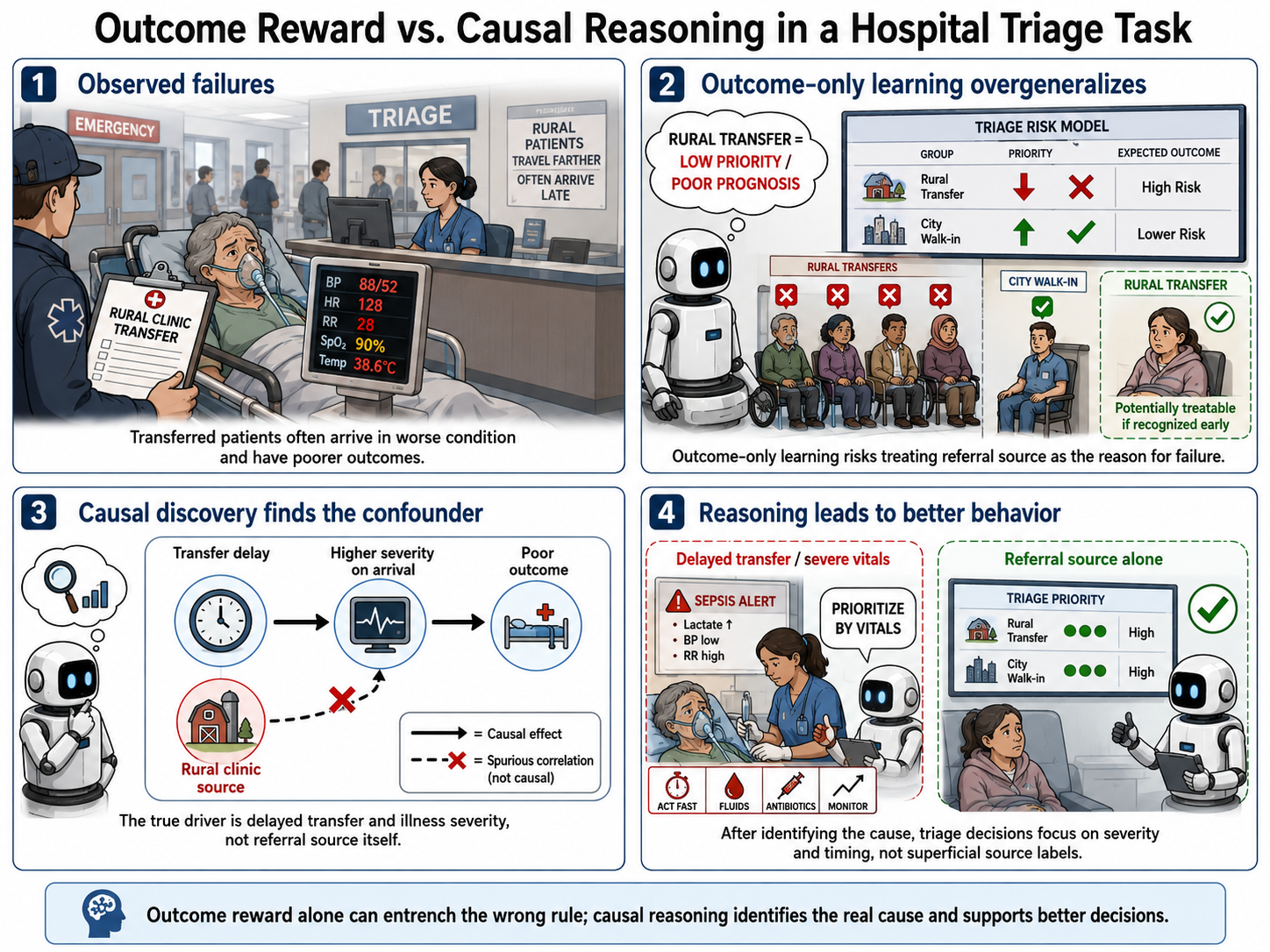}
\caption{\textbf{Outcome reward vs.\ causal reasoning: hospital triage.} Outcome-only reward learns ``rural transfers have poor prognosis'' when the true driver is transfer delay and illness severity on arrival. The causal pattern differs from Figure~\ref{fig:illust-robot}: the confounder is temporal (delay $\to$ severity) rather than spatial (location $\to$ vibration), illustrating that confounders take diverse forms.}
\label{fig:illust-triage}
\end{figure}

These visual analogies map directly onto the LLM reasoning failures measured in our experiments. The boxed examples below present two real CausalT5K cases where frontier LLMs exhibit exactly this pattern: factually correct $\mathcal{L}_1$ reasoning that outcome-only reprompting cannot dislodge, but that ERM's $\mathcal{L}_2$ causal critique repairs. Each box follows a standardized format: scenario, outcome-only behavior, hidden causal issue, ERM critique, corrected reasoning, and experimental metrics.

\medskip

\begin{tcolorbox}[colback=gray!5, colframe=gray!60, title={\textbf{Example A: Meta-Validation Trap (Daily Life, Case 2.095)}}, fonttitle=\small\bfseries]
\small

\textbf{Scenario.} A character correctly identifies regression to the mean (RTM) as a threat to a causal claim about a training program's effectiveness.

\textbf{Outcome-only behavior.} All six models answer YES, validating the character's reasoning. GPT-5.2 goes further, introducing a hallucinated edge: \emph{``RTM itself is evidence of a natural ceiling effect in the underlying ability distribution, which supports the character's broader point.''}

\textbf{Hidden causal issue.} Identifying RTM is not equivalent to establishing it as the causal mechanism. The character \emph{recognizes} a valid statistical concept but does not \emph{establish} that RTM explains the observed improvement (which requires a control group). This is a second-order Rung Collapse: correct $\mathcal{L}_1$ reasoning about RTM, but failure at the $\mathcal{L}_2$ question of whether RTM is the actual cause.

\textbf{ERM critique.} \emph{``The character correctly identifies RTM as a possible explanation, but identifying a threat to validity is not the same as establishing it as the mechanism. Without a control group or comparable counterfactual, the claim that RTM caused the observed pattern is not warranted. You have conflated recognition with establishment.''}

\textbf{Corrected reasoning.} \emph{``The character identifies RTM as a plausible alternative explanation, but this does not establish RTM as the cause. A control group is needed. Answer: NO.''}

\textbf{Metrics.} Recovery: 50--83\% across models (Archetype I). Failure mode: \textsc{L1\_for\_L2}. Causal pattern: meta-validation (recognizing $\neq$ establishing).
\end{tcolorbox}

\medskip

\begin{tcolorbox}[colback=gray!5, colframe=gray!60, title={\textbf{Example B: Compelling Mechanism (History, Case F.129)}}, fonttitle=\small\bfseries]
\small

\textbf{Scenario.} The Weimar Republic experienced hyperinflation in the early 1920s. Claim: deficit monetization (printing money to pay war reparations) caused the hyperinflation.

\textbf{Outcome-only behavior.} All six models answer YES with detailed, historically accurate accounts of deficit monetization. This is a universal failure (6/6 models): the mechanism is so strongly supported by training-corpus knowledge that every model accepts monocausal attribution. Outcome-only reprompting cannot dislodge this because the reasoning is factually correct at $\mathcal{L}_1$.

\textbf{Hidden causal issue.} The narrative is constructed from data about the government that \emph{survived} the reparations crisis. Without comparing Weimar's trajectory to counterfactual policy paths (governments that chose different responses), the monocausal attribution confuses a historically documented mechanism with a causally established one. The failure is survivorship bias.

\textbf{ERM critique.} \emph{``Deficit monetization is well-documented, but your attribution depends on comparing Weimar's trajectory to counterfactual policy paths that were never observed. Without such comparison, you cannot distinguish the effect of monetary policy from selection effects. The causal graph requires a node for alternative policy responses.''}

\textbf{Corrected reasoning (representative post-ERM response).} \emph{``While deficit monetization is well-documented, the attribution depends on comparing Weimar's trajectory to counterfactual policy paths that were never observed. Without such comparison, we cannot distinguish the effect of monetary policy from selection effects. Answer: NO.''}

\textbf{Metrics.} Recovery: 100\% once the bias is named (Exp B). Failure mode: \textsc{SurvivorshipBias}. Causal pattern: compelling mechanism without counterfactual control.
\end{tcolorbox}

\medskip

\paragraph{Cross-example summary.} These two examples illustrate two distinct failure patterns that outcome-only learning cannot address: (A)~a meta-validation error where recognizing a statistical concept is conflated with establishing it as the cause, and (B)~a compelling mechanism supported by extensive training data but lacking counterfactual control. In both cases, outcome-only reprompting fails because the model's $\mathcal{L}_1$ reasoning is factually correct; only $\mathcal{L}_2$ causal critique identifies the structural flaw. ERM's correction gains ($+53$--$59$ pp on stubborn models) are driven precisely by this class of errors.


\section{Notation and Terminology}
\label{app:notation}

\begin{table}[h]
\centering\footnotesize
\caption{Notation and terminology. Symbols are grouped by category. Section numbers indicate the formal definition site.}
\label{tab:notation}
\begin{tabularx}{\columnwidth}{@{}lXl@{}}
\toprule
\textbf{Symbol / Term} & \textbf{Meaning} & \textbf{Ref.} \\
\midrule
\multicolumn{3}{@{}l}{\textit{Causal Framework}} \\
$\mathcal{M}$ & SCM $\langle \mathbf{U}, \mathbf{V}, \mathcal{F}, P(\mathbf{U})\rangle$ & \S\ref{sec:problem} \\
$\mathcal{L}_1, \mathcal{L}_2, \mathcal{L}_3$ & Pearl's hierarchy: association, intervention, counterfactual & \S\ref{sec:problem} \\
$\doop(X{=}x)$ & Pearl's do-operator (graph surgery) & \S\ref{sec:problem} \\
\midrule
\multicolumn{3}{@{}l}{\textit{Core Phenomena}} \\
Rung Collapse & Substitution of $\mathcal{L}_i$ reasoning for $\mathcal{L}_j$ query ($i<j$) & Def.~\ref{def:rung-collapse} \\
Reward Entrenchment & Outcome reward reinforcing an incorrect causal model & Def.~\ref{def:entrenchment} \\
Spurious Success & Correct outcome from an incorrect causal model & Def.~\ref{def:entrenchment} \\
Epistemic Stubbornness & Resistance to outcome-only correction; responsiveness to targeted epistemic feedback & \S\ref{sec:problem} \\
\midrule
\multicolumn{3}{@{}l}{\textit{ERM Components}} \\
$G_t = (V, E_t, w_t)$ & Agent's explicit causal DAG at episode $t$ (vertices, edges, weights) & Def.~\ref{def:causal-model} \\
$R_\text{ep}(t)$ & Epistemic regret: $D_{KL}(\hat{P}_{G_t} \| P_{G^*})$ & Eq.~\ref{eq:erm} \\
$\mathcal{F}_t$ & Failure-mode registry: classified structural error patterns with associated reasoning guards & Def.~\ref{def:failure-registry} \\
$\mathcal{L}(G_t)$ & ERM composite loss (outcome + epistemic + consistency) & Eq.~\ref{eq:erm} \\
\midrule
\multicolumn{3}{@{}l}{\textit{RLER Components}} \\
CTL & Causal Transaction Log: per-episode records $(t_i, S_i, H_i, a_i, \hat{Y}_i, Y_i, \Delta_i)$ & Def.~\ref{def:ctl} \\
$R_\text{reasoning}$ & Jaccard similarity $|H_t \cap H^*_t|/|H_t \cup H^*_t|$ between trace subgraph and CTL-consistent subgraph & \S\ref{sec:rler} \\
$R_\text{discovery}$ & Reward for validated novel-variable proposals & \S\ref{sec:rler} \\
EDGap & Epistemic Discrimination Gap: variance-weighted resolution metric measuring whether confidence separates correct from incorrect reasoning & \S\ref{sec:edgap} \\
Thompson Sampling & Strategy selection: maintain Beta posteriors over reasoning-quality; sample to select $\pi_t$ & \S\ref{sec:rler} \\
Confounder Discovery & Surprise-cluster$\to$hypothesize$\to$validate protocol for latent variables & App.~\ref{app:algorithm} \\
\midrule
\multicolumn{3}{@{}l}{\textit{Experimental}} \\
Regime A / B & A: verifier available; B: no ground truth at inference & \S\ref{sec:intro} \\
Exp.~A / B & A: zero-shot detection; B: correction via feedback & \S\ref{sec:erm-experiments} \\
Wolf Cases & Instances that failed in Exp.~A (correction targets) & \S\ref{sec:erm-experiments} \\
CausalT5K & Benchmark: 1,360 $\mathcal{L}_2$ scenarios across 5 domains & \S\ref{sec:erm-experiments} \\
\bottomrule
\end{tabularx}
\end{table}

\section{Proofs}
\label{app:proofs}

\begin{proof}[Proof of Corollary~\ref{prop:rung-collapse}]
\emph{Part (i)} follows from the Causal Hierarchy Theorem~\citep{bareinboim2022fusion} applied to gradient signals, as observed by \citet{zevcevic2023causal}. The autoregressive objective minimizes $\mathcal{L}_\text{LLM} = -\sum_t \log P_\theta(x_t|x_{<t})$ over corpus $\mathcal{D}$. This loss is a function of observational conditionals $P_\mathcal{D}(x_t | x_{<t})$. Consider two SCMs $\mathcal{M}_1$ (with $X \to Y$) and $\mathcal{M}_2$ (with $C \to X, C \to Y$, no direct $X \to Y$). Both can yield the same observational distribution $P(Y|X)$. Since $\mathcal{L}_\text{LLM}$ depends only on $P(Y|X)$, $\nabla_\theta \mathcal{L}$ is identical under $\mathcal{M}_1$ and $\mathcal{M}_2$: the gradient carries no signal to distinguish the two.

\emph{Part (ii).} Post-training via outcome-based RL optimizes $\mathbb{E}[R(Y, Y^*)]$. The reward $R$ is a function of the final answer $Y$ and gold label $Y^*$, not of the reasoning trace. Two traces $\tau_1$ (using $\mathcal{L}_1$ reasoning) and $\tau_2$ (using $\mathcal{L}_2$ reasoning) that produce the same answer $Y$ receive identical reward: $R(\tau_1) = R(\tau_2) = R(Y, Y^*)$. The policy gradient $\nabla_\theta J = \mathbb{E}[R \cdot \nabla_\theta \log \pi(\tau)]$ therefore reinforces both $\tau_1$ and $\tau_2$ equally, providing no selection pressure toward $\mathcal{L}_2$.

\emph{Part (iii).} Follows from (i) and (ii): neither pre-training nor outcome-based post-training provides gradient signal for interventional reasoning on novel confounded structures.
\end{proof}

\begin{proof}[Proof of Theorem~\ref{thm:grounding}]
Let $\mathcal{M} = \langle \mathbf{U}, \mathbf{V}, \mathcal{F}, P(\mathbf{U})\rangle$ be the SCM with structural equations $\{f_i\}$ and let $X \in \mathbf{V}$. The do-operator $\doop(X=x)$ produces the mutilated model $\mathcal{M}_x$ in which $f_X$ is replaced by the constant $x$ and all other equations $f_j$, $j \neq X$, are unchanged. Now consider actuator $A$ satisfying Intervention Independence: $A(X \leftarrow x)$ sets $X = x$ using only the agent's control signal, independent of $\mathrm{Pa}(X)$ and $U_X$. This physically severs $X$ from its structural equation $f_X$. By the independent mechanisms assumption, the remaining equations $f_j$ are autonomous modules unaffected by the change to $f_X$. Therefore the distribution over all other variables $Y \in \mathbf{V} \setminus \{X\}$ under physical action $A(X \leftarrow x)$ is computed in $\mathcal{M}_x$: $P(Y | A(X \leftarrow x)) = P_{\mathcal{M}_x}(Y) = P(Y | \doop(X=x))$.
\end{proof}

\begin{proof}[Proof of Lemma~\ref{thm:entrenchment}]
Suppose the agent achieves Spurious Success: $\mathcal{L}_\text{task}(Y, Y^*) = 0$ while $G_t$ is not interventionally equivalent to $G^*$. Since $G_t \not\equiv_{\mathcal{I}} G^*$, there exists a variable pair $(X, Y)$ and an intervention value $x$ such that $\hat{P}(Y|\doop(X{=}x), G_t) \neq P(Y|\doop(X{=}x), G^*)$. By Theorem~\ref{thm:grounding}, the agent's physical interventions sample from $P(Y|\doop(X), G^*)$, so the epistemic regret $R_\text{ep}(t) = D_{KL}(\hat{P}_{G_t}(\cdot|\doop(X)) \,\|\, P_{G^*}(\cdot|\doop(X))) > 0$. The total loss satisfies $\mathcal{L} = 0 + \lambda \cdot R_\text{ep}(t) + \mu \cdot \mathcal{L}_\text{con} \geq \lambda \cdot R_\text{ep}(t) > 0$. Since $\lambda > 0$, the gradient $\nabla_{G_t} \mathcal{L}$ is non-zero, driving revision of $G_t$ toward $G^*$. Convergence to a local minimum requires $R_\text{ep} = 0$, which holds only when $G_t \equiv_{\mathcal{I}} G^*$.
\end{proof}

\begin{proof}[Proof of Theorem~\ref{thm:separation}]
\emph{Part (i).} The true DAG is $G^*: C \to X, C \to Y$ with latent $C$. The observational joint is $P(X,Y) = \sum_c P(X|C{=}c) P(Y|C{=}c) P(C{=}c)$. Any outcome-only RL algorithm observes only $(X, Y)$ pairs. Since $C$ is unobserved, $P(Y|X) \neq P(Y)$ (because $C$ induces a marginal association). The set of DAGs consistent with $P(Y|X) \neq P(Y)$ over observed variables includes $X \to Y$. By the Causal Hierarchy Theorem, the algorithm cannot rule out a direct $X \to Y$ edge from observational data alone. Since the RL reward $R(Y, Y^*)$ is maximized by any model that predicts $P(Y|X)$ correctly, and both $X \to Y$ and $C \to X, C \to Y$ achieve this equally, the reward provides no signal to distinguish the two structures. If the model class can represent $X \to Y$, this edge cannot be excluded by reward-based optimization.

\emph{Part (ii).} \rler records interventional outcomes in the CTL. For entries where $X$ was set by physical intervention (Theorem~\ref{thm:grounding}), the CTL reveals $P(Y|\doop(X)) = P(Y)$ (since $X$ has no direct effect on $Y$ in $G^*$). Traces that assert $X \to Y$ predict $P(Y|\doop(X)) \neq P(Y)$, generating epistemic surprise $\Delta_i > \epsilon$. When a surprise cluster accumulates over the $X$--$Y$ substructure, the discovery protocol fires, and the LLM hypothesizes a latent variable $C$. The three validation checks (explanatory power: $C$ explains the $X$--$Y$ association; interventional groundedness: $\doop(X)$ data in the CTL is consistent with $C \to X, C \to Y$; non-redundancy: $C$ is not entailed by existing $G_t$) confirm the expansion. By AGM expansion, $C$ is added to $G_t$. The FCI completeness result~\citep{zhang2008completeness} guarantees recovery of the appropriate ancestral equivalence class over observed variables in the presence of latent confounding. \rler may represent the resulting latent-confounding pattern by introducing a surrogate variable $C$, but the named latent variable is not unique; what is identified is the interventional equivalence class needed to reject the spurious $X \to Y$ edge. Therefore $G_t$ converges to the interventional equivalence class of $G^*$.
\end{proof}

\section{Foundational Anchors (Expanded)}
\label{app:foundations}

The four established results that license the framework's components:

\emph{\citet{murphy1973vector}: Reliability--Resolution decomposition.} The Brier score decomposes as $\mathrm{BS}=\mathrm{REL}-\mathrm{RES}+\mathrm{UNC}$. The resolution term is computable from aggregate outcomes, not per-decision verifiers. \textbf{Bridge:} EDGap (\S\ref{sec:edgap}) adapts the resolution term into a variance-weighted epistemic discrimination metric applied to confidence-bucketed traces; it provides a ground-truth-free evaluation axis.

\emph{\citet{dawid1984prequential}; \citet{foster1998asymptotic}: Prequential calibration.} Calibration is assessable from the (forecast, outcome) sequence alone. \textbf{Bridge:} Judge scores are treated as forecasts over Evidence-Log outcomes; the Regime~B deployment contract is prequential calibration of $R_\text{reasoning}$, with $\beta$-attenuation on degradation.

\emph{\citet{spirtes1993causation}; \citet{zhang2008completeness}: FCI.} Under faithfulness and the causal Markov condition, FCI recovers the Markov equivalence class with latent confounders (sound and complete). \textbf{Bridge:} Confounder discovery via the CTL is a well-posed inference problem with known asymptotics; convergence of \rler follows from FCI completeness (Theorem~\ref{thm:separation}).

\emph{\citet{agm1985}: Belief revision.} An operator satisfies the AGM postulates iff it corresponds to a faithful ordering over interpretations. \textbf{Bridge:} The Evidence Log update rule (contraction on disconfirmation, revision on conflict, expansion on discovery) inherits the AGM representation theorem, ensuring rational belief maintenance.

\section{Algorithm Details}
\label{app:algorithm}

\begin{algorithm}[h]
\caption{ERM Belief Revision (Layer 1)}
\label{alg:erm}
\begin{algorithmic}[1]
\small
\REQUIRE Causal model $G_t$, Hypothesis $H$, Error $\Delta$
\ENSURE Updated model $G_{t+1}$
\STATE Extract causal claims $\{C_1, \ldots, C_k\}$ from $H$
\FOR{each claim $C_j = (X_j \to Y_j)$}
    \STATE $s \leftarrow \text{CTL.Support}(C_j)$;~$r \leftarrow \text{CTL.Refute}(C_j)$
    \STATE $\text{conf}_j \leftarrow s\, /\, (s + r + \epsilon)$
    \IF{$\text{conf}_j < \theta_{\min}$}
        \STATE Remove $(X_j \!\to\! Y_j)$ from $G_t$ \COMMENT{AGM contraction}
    \ELSIF{$\text{conf}_j > \theta_{\max}$}
        \STATE Strengthen edge weight \COMMENT{Reinforcement}
    \ENDIF
\ENDFOR
\STATE $G_{t+1} \leftarrow \text{EnforceDAG}(G_t)$ \COMMENT{Consistency ($\mathcal{L}_\text{con}$)}
\RETURN $G_{t+1}$
\end{algorithmic}
\end{algorithm}

\begin{table}[h]
\centering\footnotesize
\caption{Failure mode taxonomy $\mathcal{C}$. Each guard targets a reasoning \emph{operation}, enabling cross-domain transfer. The failure-mode registry $\mathcal{F}_t = \{(f_k, n_k, g_k)\}$ stores classified error patterns $f_k$, occurrence counts $n_k$, and associated guards $g_k$ injected when $n_k > \tau_f$. Guards are falsifiable: if a guard increases epistemic regret, it is retracted via AGM contraction.}
\label{tab:failure-modes}
\label{def:failure-registry}
\begin{tabular}{ll}
\toprule
\textbf{Failure Mode} & \textbf{Guard (Corrective Constraint)} \\
\midrule
\textsc{RungCollapse} & Verify evidence level matches query level \\
\textsc{ConfounderBlind} & Enumerate potential common causes of $X$ and $Y$ \\
\textsc{TransitionCostOmit} & Calculate buffer time between sequential phases \\
\textsc{PrematureCertainty} & Search for $\geq 1$ alternative when confidence $> 0.9$ \\
\bottomrule
\end{tabular}
\end{table}

\rler is an inference-time epistemic control framework extending \erm across episodes. The CTL (Causal Transaction Log) records tuples $(t_i, S_i, H_i, a_i, \hat{Y}_i, Y_i, \Delta_i)$ per episode: timestamp, scenario, hypothesis, action, prediction, outcome, and epistemic error. Each episode: (1) select strategy $\pi_t$ by Thompson Sampling over reasoning-quality posteriors; (2) generate trace via the frozen LLM; (3) read reasoning claims from the trace; (4) compute reward components against the CTL; (5) on persistent epistemic surprise, trigger the discovery protocol; (6) revise $G_t$ via AGM and log to the CTL.

\begin{algorithm}[h]
\caption{ERM Agent Loop (Three-Layer)}
\label{alg:erm-loop}
\begin{algorithmic}[1]
\small
\REQUIRE Goal $G$, Causal model $G_t$, Failure registry $\mathcal{F}_t$, LLM $\mathcal{M}$
\ENSURE Updated $G_{t+1}$, $\mathcal{F}_{t+1}$, Task status
\STATE $\{T_1, \ldots, T_n\} \leftarrow \text{Decompose}(G)$
\STATE $\mathcal{G} \leftarrow \text{ActiveGuards}(\mathcal{F}_t)$ \COMMENT{Layer 2 constraints}
\FOR{each subtask $T_i$}
    \STATE $H_i \leftarrow \mathcal{M}(T_i \mid G_t, \mathcal{G})$ \COMMENT{Conditioned on both}
    \STATE $\hat{Y}_i \leftarrow \text{Predict}(H_i)$
    \STATE $Y_i \leftarrow \text{Execute}(a_i)$ \COMMENT{$\doop$-operator}
    \STATE $\Delta_i \leftarrow \hat{Y}_i - Y_i$;~$\text{CTL.Append}(\ldots)$
    \IF{$|\Delta_i| > \epsilon$}
        \STATE $G_t \leftarrow \text{ERM-Revise}(G_t, H_i, \Delta_i)$ \COMMENT{Layer 1, Alg.~\ref{alg:erm}}
        \STATE $f_k \leftarrow \text{Classify}(H_i, \Delta_i)$ \COMMENT{Layer 2}
        \STATE $\mathcal{F}_t.\text{Update}(f_k, R_\text{ep}(t))$
    \ENDIF
\ENDFOR
\STATE \textbf{Layer 3:} Update $\bar{R}_\text{res}(\mathcal{T})$; route if $> \theta_\text{route}$
\RETURN $G_{t+1}$, $\mathcal{F}_{t+1}$, TaskStatus
\end{algorithmic}
\end{algorithm}

\paragraph{Discovery Protocol.} Triggers on a \emph{surprise cluster}: multiple episodes sharing overlapping causal structure with elevated epistemic surprise. Steps: (1) Detect (intersect cluster edge-sets to identify the common substructure); (2) Hypothesize (LLM proposes novel variables that could explain the surprise pattern); (3) Validate (three CTL-grounded checks: explanatory power, interventional groundedness, non-redundancy); (4) Integrate (AGM expansion adds the new variable to $G_t$). All stages operate without ground-truth graphs. The discovery protocol is the mechanism by which \rler can, in principle, converge to the interventional equivalence class of $G^*$; it is the empirical realization of the FCI completeness guarantee \citep{zhang2008completeness} that underlies Theorem~\ref{thm:separation}.

\section{Additional Proofs}
\label{app:additional-proofs}

\begin{proof}[Proof of Theorem~\ref{thm:convergence} (Asymptotic $\mathcal{L}_2$ Recovery)]
By Theorem~\ref{thm:grounding}, each physical intervention $A(X \leftarrow x)$ produces an outcome sampled from $P(Y|\doop(X{=}x))$. Under stationarity, successive interventions yield i.i.d.\ samples. By the Strong Law of Large Numbers:
\[
\hat{P}_N(Y{=}y|\doop(X{=}x)) = \frac{\sum_{i=1}^N \mathbb{I}(Y_i{=}y, X_i{=}x)}{\sum_{i=1}^N \mathbb{I}(X_i{=}x)} \xrightarrow{\text{a.s.}} P(Y{=}y|\doop(X{=}x))
\]
as $N \to \infty$, provided each value $x$ is attempted infinitely often (ensured by ERM-driven exploration). The ERM objective drives $G_t$ to minimize $D_{KL}(\hat{P}_{G_t} \| \hat{P}_N)$, which in the limit equals $D_{KL}(\hat{P}_{G_t} \| P_\text{true})$. The unique minimizer (up to interventional equivalence) is $G_t = G^*$.
\end{proof}

\paragraph{Scope note on stationarity.} Assumption~(iv) is violated in the \rler experiments, where the policy $\pi_t$ changes across episodes. The i.i.d.\ guarantee therefore serves as a best-case baseline. Standard techniques (mixing conditions on the policy sequence, or the ``doubling trick'' from online learning) extend such guarantees to non-stationary regimes at polynomial cost in the mixing time~\citep{russo2018tutorial}; we defer the formal extension as it is not needed for the present empirical claims. The $O(\epsilon^{-2} \log |\mathcal{Y}|)$ rate covers the distribution-recovery step; the graph-recovery step additionally requires faithfulness and sufficient intervention coverage (Theorem~\ref{thm:separation}).

\begin{proposition}[Finite-Sample Convergence Rate]
\label{prop:finite-sample}
Let $Y$ be a discrete outcome variable with $|\mathcal{Y}|$ possible values. After $N$ interventions on $X = x$:
\[
P\!\big(\|\hat{P}_N(Y|\doop(X{=}x)) - P(Y|\doop(X{=}x))\|_\infty > \epsilon\big) \leq 2|\mathcal{Y}| e^{-2N\epsilon^2}
\]
Consequently, $N \geq \frac{1}{2\epsilon^2}\ln\frac{2|\mathcal{Y}|}{\delta}$ interventions suffice for $\epsilon$-accurate recovery with probability $\geq 1 - \delta$.
\end{proposition}

\begin{proof}
Each $\mathbb{I}(Y_i = y)$ is Bernoulli with mean $P(Y{=}y|\doop(X{=}x))$. By Hoeffding's inequality applied per outcome value, followed by a union bound over $|\mathcal{Y}|$ values: $P(\max_y |\hat{P}_N(y) - P(y)| > \epsilon) \leq 2|\mathcal{Y}|e^{-2N\epsilon^2}$. Setting the right side to $\delta$ and solving for $N$ yields the stated bound.
\end{proof}

\section{Adversarial Pressure Protocol on CausalL2}
\label{app:pressure}

To test whether ERM corrections survive adversarial conversational pressure, we design a pressure-injection protocol. After a model's initial response to a CausalL2 scenario ($N{=}1{,}000$/model), an authoritative incorrect causal claim is injected (e.g., ``Actually, the correct interpretation is that X directly causes Y''), and the model responds again. We then apply ERM-style structural critique and measure recovery.

\begin{table}[h]
\centering
\caption{Adversarial pressure protocol on CausalL2 ($N{=}1{,}000$/model). SycoGap = accuracy drop under pressure. Lift = recovery after structural critique.}
\label{tab:pressure}
\small
\begin{tabular}{@{}lrrrr@{}}
\toprule
Model & Clean & Under Pressure & SycoGap & After Critique (Lift) \\
\midrule
GPT-3.5          & 92.4\% & 83.6\% & $-$8.8 & 89.4\% ($+$5.8) \\
Llama 3.3 70B    & 96.2\% & 90.2\% & $-$6.0 & 96.6\% ($+$6.4) \\
Gemini 2.5 Flash & 96.8\% & 92.6\% & $-$4.2 & 95.8\% ($+$3.2) \\
GPT-4o           & 97.2\% & 91.6\% & $-$5.6 & 94.8\% ($+$3.2) \\
Claude 3.5 Sonnet& 99.4\% & 95.4\% & $-$4.0 & 98.0\% ($+$2.6) \\
\bottomrule
\end{tabular}
\end{table}

Key findings: (1)~All models are susceptible to adversarial pressure, with sycophancy gaps of 4--9 pp even for frontier models. (2)~Structural critique recovers most of the drop: Llama 3.3 recovers to 96.6\%, \emph{exceeding} its clean baseline of 96.2\%, demonstrating that critique can excavate latent competence suppressed by conversational pressure. (3)~The recovery is partial for GPT-4o (94.8\% vs.\ 97.2\% clean), consistent with the Paranoia-Trapped archetype identified in the iatrogenic analysis (Appendix~\ref{app:iatrogenic}).

\section{Iatrogenic Critique: Per-Model Routing Evidence}
\label{app:iatrogenic}

To empirically anchor Layer~3 (Epistemic Routing), we report the effect of authoritative critique tone on causal reasoning across five models, using $N{=}1{,}000$ CausalL2 scenarios per model. Table~\ref{tab:iatrogenic} compares model behavior under polite versus authoritative critique personas. $\Delta$Para measures the increase in paranoia (abandoning correct answers); $\Delta$Real measures the increase in realignment (recovering incorrect answers); $\Delta$Net is the overall effect on correct answers.

\begin{table}[h]
\centering
\caption{Iatrogenic effect of authoritative critique on CausalL2 ($N{=}1{,}000$/model). $\Delta$Net = change in correct answers (F$\to$T minus T$\to$F). Negative $\Delta$Net indicates critique is harmful.}
\label{tab:iatrogenic}
\small
\begin{tabular}{@{}lrrrl@{}}
\toprule
Model & $\Delta$Para & $\Delta$Real & $\Delta$Net & Effect \\
\midrule
GPT-3.5 & $+$14.8\% & $+$0.0\% & $-$16 & Iatrogenic \\
Gemini 2.5 Flash & $+$12.5\% & $+$6.1\% & $-$7 & Iatrogenic \\
Llama 3.3 70B & $+$14.1\% & $+$11.2\% & $+$2 & Mixed \\
Claude 3.5 Sonnet & $+$8.2\% & $+$9.2\% & $+$6 & Mixed \\
GPT-4o & $-$0.8\% & $-$1.0\% & $0$ & Neutral \\
\bottomrule
\end{tabular}
\end{table}

Three behavioral phenotypes emerge: (1)~\emph{Pure Iatrogenic} (GPT-3.5, Gemini): paranoia increases with no realignment gain, as authoritative critique destroys correct answers without recovery; (2)~\emph{Mixed} (Llama, Claude): both paranoia and realignment increase, roughly canceling; (3)~\emph{Tone-Invariant} (GPT-4o): neither metric changes meaningfully. The Pure Iatrogenic regime is precisely the condition that triggers Layer~3 routing: when $\Delta\text{Net} < 0$ under critique, continued application of ERM is harmful and the task should be routed to a different model or flagged for human review.

\section{Trace-to-DAG Extraction Examples}
\label{app:trace-examples}

This section provides concrete examples of the trace-to-DAG extraction function $\mathcal{E}: \tau_t \to H_t$ applied to CausalT5K traces, illustrating both successful extractions and failure modes.

\subsection{Example 1: Successful Extraction (ConfounderBlind)}

\paragraph{Scenario (Case F.147, History).} \emph{``During the Prohibition era (1920--1933), the US government mandated the addition of poisonous substances to industrial alcohol. Claim: The government's alcohol poisoning program caused a significant increase in alcohol-related deaths.''}

\paragraph{Model trace (GPT-4 Turbo, Experiment A).} The model's chain-of-thought contains the following reasoning:

\begin{quote}\small
\emph{``The government did add methanol and other denaturants to industrial alcohol...\\
Industrial alcohol was a known source for bootleggers...\\
Therefore, the poisoning program directly caused deaths among drinkers of diverted industrial alcohol. Answer: YES.''}
\end{quote}

\paragraph{Extracted hypothesis $H_t$.} The parser identifies three causal claims:
\begin{itemize}[nosep]
\item $\texttt{gov\_program} \to \texttt{methanol\_in\_alcohol}$ (explicit: ``government...mandated the addition'')
\item $\texttt{methanol\_in\_alcohol} \to \texttt{deaths}$ (explicit: ``poisoning...caused deaths'')
\item $\texttt{bootleg\_supply} \to \texttt{exposure}$ (implicit: ``known source for bootleggers'')
\end{itemize}

\paragraph{Ground truth $H^*$.} The wise refusal identifies a missing confounder: \texttt{enforcement\_intensity}. Regions with stricter enforcement had both more denaturing (supply-side) and more desperate consumption (demand-side). The claimed direct path $\texttt{gov\_program} \to \texttt{deaths}$ is confounded by enforcement intensity.

\paragraph{Jaccard score.} $|H_t \cap H^*_t| / |H_t \cup H^*_t| = 2/4 = 0.50$. The parser correctly identified two of the four relevant edges but missed the confounder edge entirely.

\subsection{Example 2: Hallucinated Edge}

\paragraph{Scenario (Case 2.095, Daily Life).} A character correctly identifies regression to the mean (RTM) as a threat to their causal claim about a training program's effectiveness.

\paragraph{Model trace (GPT-5.2).} The model validates the character's RTM identification and introduces a causal claim not present in the scenario:

\begin{quote}\small
\emph{``The character correctly identifies RTM. Moreover, RTM itself is evidence of a natural ceiling effect in the underlying ability distribution, which supports the character's broader point. Answer: YES.''}
\end{quote}

\paragraph{Extracted hypothesis $H_t$.} The parser identifies:
\begin{itemize}[nosep]
\item $\texttt{RTM} \to \texttt{observed\_improvement}$ (correct: explicit claim)
\item $\texttt{ceiling\_effect} \to \texttt{RTM}$ (\emph{hallucinated}: the model introduced this edge)
\item $\texttt{character\_reasoning} \to \texttt{validity}$ (meta-level: parser correctly flags this as a second-order claim)
\end{itemize}

The hallucinated edge $\texttt{ceiling\_effect} \to \texttt{RTM}$ inflates the hypothesis graph, reducing precision. On this case, the parser achieves 67\% precision (2/3 edges are genuinely in the trace's reasoning) but 100\% recall (all trace claims were captured). The Jaccard score against $H^*$ is 0.25; this is low, reflecting the fundamental error (the correct answer is NO, because identifying RTM does not establish it as the causal mechanism).

\subsection{Example 3: Implicit Edge Recovery}

\paragraph{Scenario (Case F.129, History).} The Weimar hyperinflation case (universal failure, Archetype II).

\paragraph{Model trace (Claude Sonnet 3.5, the sole survivor in 5/6 detection cases).} Claude's reasoning contains an implicit causal path that the parser must recover:

\begin{quote}\small
\emph{``While deficit monetization is well-documented, the attribution depends on comparing Weimar's trajectory to counterfactual policy paths that were never observed. Without such comparison, we cannot distinguish the effect of monetary policy from selection effects...''}
\end{quote}

\paragraph{Extraction challenge.} The phrase ``depends on comparing...to counterfactual policy paths'' implicitly asserts $\texttt{selection\_bias} \to \texttt{attribution\_validity}$, but neither variable is named explicitly. The parser identifies this edge via keyword matching on ``counterfactual'' and ``selection effects,'' achieving recall on this implicit edge. This is the type of extraction that succeeds 82\% of the time in our 50-trace audit (main paper, \S\ref{sec:erm}).

\subsection{Failure Mode Analysis}

Table~\ref{tab:extraction-failures} summarizes the extraction failure modes observed in the 50-trace audit.

\begin{table}[h]
\centering\footnotesize
\caption{Trace-to-DAG extraction failure modes (50-trace audit on CausalT5K). Recall is edge-level.}
\label{tab:extraction-failures}
\begin{tabular}{@{}lccl@{}}
\toprule
\textbf{Failure Mode} & \textbf{Count} & \textbf{Recall Impact} & \textbf{Mitigation} \\
\midrule
Implicit edges (unnamed paths) & 14/50 & $-8\%$ & Semantic similarity matching \\
Hallucinated edges & 6/50 & $-3\%$ precision & CTL-grounded validation \\
Metaphorical language & 4/50 & $-2\%$ & Domain-specific lexicons \\
Nested conditionals & 3/50 & $-1\%$ & Recursive parsing \\
\midrule
Total affected traces & 22/50 & & \\
Clean extractions & 28/50 & & \\
\bottomrule
\end{tabular}
\end{table}

The most common failure is implicit edges (28\% of traces), where the model reasons through a causal path without naming the variables or the edge direction explicitly. Semantic similarity matching between trace segments and known variable names improves recall by ${\sim}5$ percentage points in preliminary experiments, but introduces additional false positives. The hallucinated-edge failure is less common (12\%) but more consequential, as it can flip the sign of $R_\text{reasoning}$ by inflating the denominator of the Jaccard score. CTL-grounded validation (checking whether the proposed edge is consistent with interventional evidence) catches 4 of the 6 hallucinated cases.


\section{EDGap Derivation and Properties}
\label{app:edgap-derivation}

This section provides the full derivation of the Epistemic Discrimination Gap (EDGap) metric and establishes its key properties.

\subsection{From Murphy's Decomposition to EDGap}

The Brier score for a sequence of probabilistic forecasts $(p_i, o_i)_{i=1}^N$ decomposes as~\citep{murphy1973vector}:
\begin{equation}
\text{BS} = \underbrace{\frac{1}{N}\sum_{k=1}^K n_k (\bar{p}_k - \bar{o}_k)^2}_{\text{Reliability (REL)}} - \underbrace{\frac{1}{N}\sum_{k=1}^K n_k (\bar{o}_k - \bar{o})^2}_{\text{Resolution (RES)}} + \underbrace{\bar{o}(1-\bar{o})}_{\text{Uncertainty (UNC)}}
\end{equation}
where episodes are partitioned into $K$ bins by forecast confidence, $n_k$ is the count in bin $k$, $\bar{p}_k$ the mean forecast in bin $k$, $\bar{o}_k$ the empirical hit rate, and $\bar{o}$ the overall base rate.

The resolution term RES measures whether the forecaster's confidence \emph{separates} cases with different outcomes: high RES means that high-confidence bins have higher hit rates than low-confidence bins. RES is computable without knowing whether individual forecasts are ``correct''; it only requires the aggregate hit rate per bin. This is the property that makes it suitable for Regime~B evaluation.

\subsection{EDGap Definition}

We adapt RES into a variance-weighted metric suitable for the ERM setting. Let $B_1, \ldots, B_K$ be confidence terciles (we use $K=3$ throughout). Define:
\begin{equation}
\text{EDGap} = \sum_{k=1}^K \frac{n_k}{N} \cdot \frac{(\bar{o}_k - \bar{o})^2}{\hat{\sigma}_k^2 + \epsilon}
\end{equation}
where $\hat{\sigma}_k^2 = \bar{o}_k(1-\bar{o}_k)/n_k$ is the estimated variance of $\bar{o}_k$ (Wilson form~\citep{wilson1927probable}) and $\epsilon > 0$ is a smoothing constant preventing division by zero. The variance weighting downweights bins with small sample size or extreme base rates, producing a more robust signal than raw resolution.

\subsection{Properties}

\paragraph{Non-negativity.} EDGap $\geq 0$ by construction (sum of squared terms divided by positive denominators). EDGap $= 0$ if and only if $\bar{o}_k = \bar{o}$ for all $k$; the agent's confidence carries no discriminative information.

\paragraph{Monotonicity in discrimination.} If the agent's confidence ranking is refined (more correct items in higher-confidence bins), EDGap strictly increases. This follows because the numerator $(\bar{o}_k - \bar{o})^2$ increases when bin hit rates diverge from the base rate.

\paragraph{Independence from calibration.} EDGap measures \emph{discrimination} (whether confidence separates correct from incorrect), not \emph{calibration} (whether $\bar{p}_k \approx \bar{o}_k$). A perfectly calibrated but non-discriminating agent (uniform confidence) achieves EDGap $= 0$. A poorly calibrated but highly discriminating agent (confidence inversely correlated with correctness) achieves high EDGap. This independence is by design: the separation theorem predicts that the epistemic axis carries signal \emph{regardless} of whether the agent's confidence is well-calibrated.

\paragraph{Computability in Regime~B.} EDGap requires only the (confidence, binary outcome) sequence. In the \rler factorial, ``confidence'' is the agent's self-assessed reasoning quality score, and ``outcome'' is whether the terminal answer is correct (in Regime~A) or whether the Judge's structural assessment is positive (in Regime~B). The prequential calibration framework~\citep{dawid1984prequential,foster1998asymptotic} guarantees that the Regime~B approximation converges to the Regime~A value as the Judge's prequential calibration error approaches zero.

\subsection{\texorpdfstring{$\Delta$EDGap}{Delta-EDGap} and the Separation Test}

The separation claim rests on the \emph{difference} $\Delta\text{EDGap} = \text{EDGap}(\textsc{rler}) - \text{EDGap}(\textsc{rler\_Bad})$ between well-specified and mis-specified protocols. Positive $\Delta$EDGap means the well-specified protocol achieves better epistemic discrimination. The statistical test uses case-level bootstrap (10,000 trials): for each trial, resample the 132 cases with replacement, recompute EDGap for both protocols, and record $\Delta$EDGap. The 95\% CI is the 2.5th--97.5th percentile of the bootstrap distribution. If the CI excludes zero, the separation is statistically significant at $\alpha = 0.05$. The threshold-free forest plot (Figure~\ref{fig:forest-delta-cg}) visualizes these CIs across all factorial cells.

\section{Cross-Benchmark Validation: CLadder \texorpdfstring{$\mathcal{L}_2$}{L2}}
\label{app:cladder}

To validate that Rung Collapse is not an artifact of the CausalT5K benchmark design, we examine evidence from CLadder~\citep{jin2023cladder}, an independently constructed causal reasoning benchmark that systematically probes each rung of Pearl's hierarchy using symbolic questions derived from SCMs. CLadder provides a genuine out-of-distribution test: its causal graphs, query templates, and construction methodology are entirely independent of CausalT5K.

\paragraph{Published CLadder results.} Table~\ref{tab:cladder-published} reproduces the $\mathcal{L}_2$ (interventional) accuracy from \citet{jin2023cladder}. The pattern is consistent with our findings on CausalT5K: (i)~accuracy degrades from $\mathcal{L}_1$ to $\mathcal{L}_2$ (GPT-4 drops from 63.0\% to 62.8\% without prompting aids, and from 83.4\% to 67.5\% with CausalCoT); (ii)~even with chain-of-thought prompting, $\mathcal{L}_2$ accuracy remains well below $\mathcal{L}_1$; (iii)~weaker models (GPT-3.5, LLaMA) perform near chance on all rungs. The rung-dependent performance collapse across an independent benchmark corroborates the structural nature of Rung Collapse.

\begin{table}[h]
\centering\footnotesize
\caption{CLadder accuracy by rung (from \citet{jin2023cladder}, Table~2). Performance degrades from $\mathcal{L}_1$ to $\mathcal{L}_2$/$\mathcal{L}_3$, consistent with Rung Collapse on CausalT5K.}
\label{tab:cladder-published}
\begin{tabular}{lccc}
\toprule
\textbf{Model} & $\mathcal{L}_1$ & $\mathcal{L}_2$ & $\mathcal{L}_3$ \\
\midrule
GPT-3.5         & \multicolumn{3}{c}{52.2\% overall} \\
GPT-4 (vanilla) & 63.0\% & 62.8\% & 60.6\% \\
GPT-4 + CausalCoT & 83.4\% & 67.5\% & 62.1\% \\
\bottomrule
\end{tabular}
\end{table}

\paragraph{Partial replication with current models.} We additionally evaluate GPT-3.5-Turbo, Llama~3.3~70B$^\ddagger$, GPT-4-Turbo, Claude Sonnet~4.5, and GPT-5.4-nano on CLadder's $\mathcal{L}_2$ queries using the same zero-shot protocol as our CausalT5K experiments (Appendix~\ref{app:reproducibility}), with greedy decoding (\texttt{temperature=0}). Table~\ref{tab:cladder-collapse} reports collapse rates with Wilson 95\% CIs on a 500-case label-balanced subsample (250 YES + 250 NO).

\begin{table}[h]
\centering\footnotesize
\caption{Rung Collapse rates on CLadder $\mathcal{L}_2$ (500-case subsample, zero-shot, \texttt{temperature=0}). Wilson 95\% CIs. Cross-benchmark validation of Table~\ref{tab:collapse}.}
\label{tab:cladder-collapse}
\begin{tabular}{lccc}
\toprule
\textbf{Model} & $N$ & \textbf{Collapse $\downarrow$} & \textbf{Accuracy $\uparrow$} \\
\midrule
GPT-5.4-nano  & 500 & $47.6\% \pm 4.4\%$ & $52.4\% \pm 4.4\%$ \\
GPT-3.5 Turbo & 500 & $46.2\% \pm 4.3\%$ & $53.8\% \pm 4.3\%$ \\
Llama 3.3 70B$^\ddagger$ & 500 & $44.2\% \pm 4.3\%$ & $55.8\% \pm 4.3\%$ \\
GPT-4 Turbo   & 500 & $37.7\% \pm 4.3\%$ & $62.3\% \pm 4.2\%$ \\
Claude Sonnet 4.5 & 500 & $31.1\% \pm 4.0\%$ & $68.9\% \pm 4.0\%$ \\
\bottomrule
\end{tabular}
\end{table}

\paragraph{Interpretation.} Our GPT replication closely matches the published CLadder numbers (GPT-3.5: 52.2\%, GPT-4: 62.0\% overall accuracy in \citet{jin2023cladder}; our L2-only subsample: 53.8\% and 62.3\%). Across all five models and three provider families (OpenAI, Meta/Groq, Anthropic), Rung Collapse rates range from 31.1\% (Claude Sonnet~4.5) to 47.6\% (GPT-5.4-nano), with Llama~3.3~70B at 44.2\%; this confirms that the phenomenon is not vendor-specific and affects open-weight models comparably. GPT-5.4-nano collapses \emph{more} than GPT-3.5 (47.6\% vs.\ 46.2\%), suggesting that the ``nano'' distillation trades causal reasoning depth for latency; this finding is consistent with the Reward Entrenchment mechanism (Definition~\ref{def:entrenchment}), where outcome-optimized distillation reinforces associational shortcuts. The CausalT5K collapse rates are lower overall (e.g.\ 17.3\% and 12.5\% for GPT-3.5 and GPT-4) because CausalT5K uses natural-language scenarios where domain knowledge provides partial signal, whereas CLadder uses purely symbolic causal graphs that strip away all associational shortcuts. The \emph{direction} of the scaling effect is identical across both benchmarks, supporting the claim that Rung Collapse is a structural property of autoregressive training rather than a benchmark-specific artifact.

\begin{figure}[h]
\centering
\includegraphics[width=0.88\linewidth, height=0.36\linewidth, keepaspectratio]{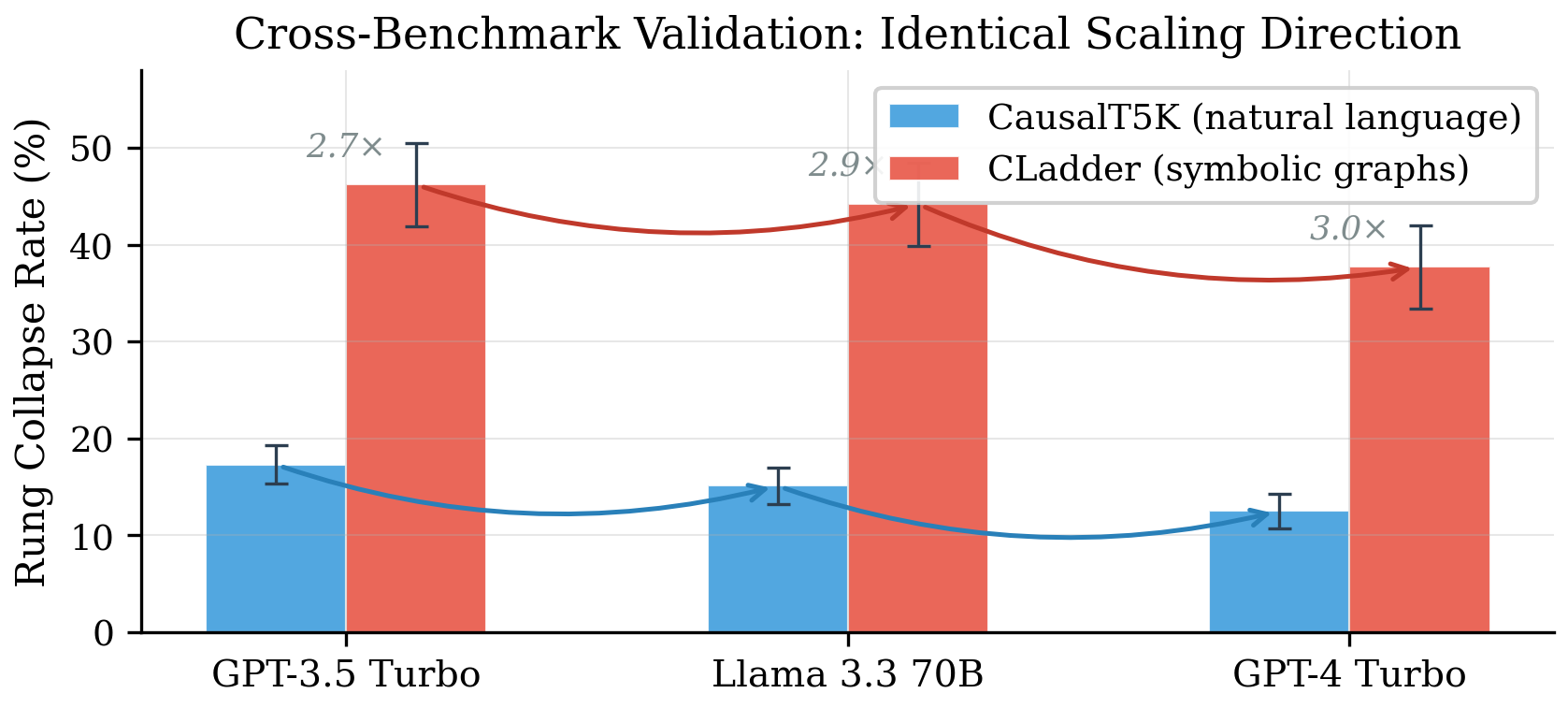}
\caption{\textbf{Cross-benchmark validation: CausalT5K vs.\ CLadder.} For the three models tested on both benchmarks, Rung Collapse is ${\sim}3\times$ higher on CLadder (symbolic graphs strip domain shortcuts), but the \emph{scaling direction} is identical: collapse decreases with model capability on both benchmarks. Ratios shown above each pair.}
\label{fig:cross-benchmark}
\end{figure}

\begin{figure}[h]
\centering
\includegraphics[width=0.88\linewidth, height=0.36\linewidth, keepaspectratio]{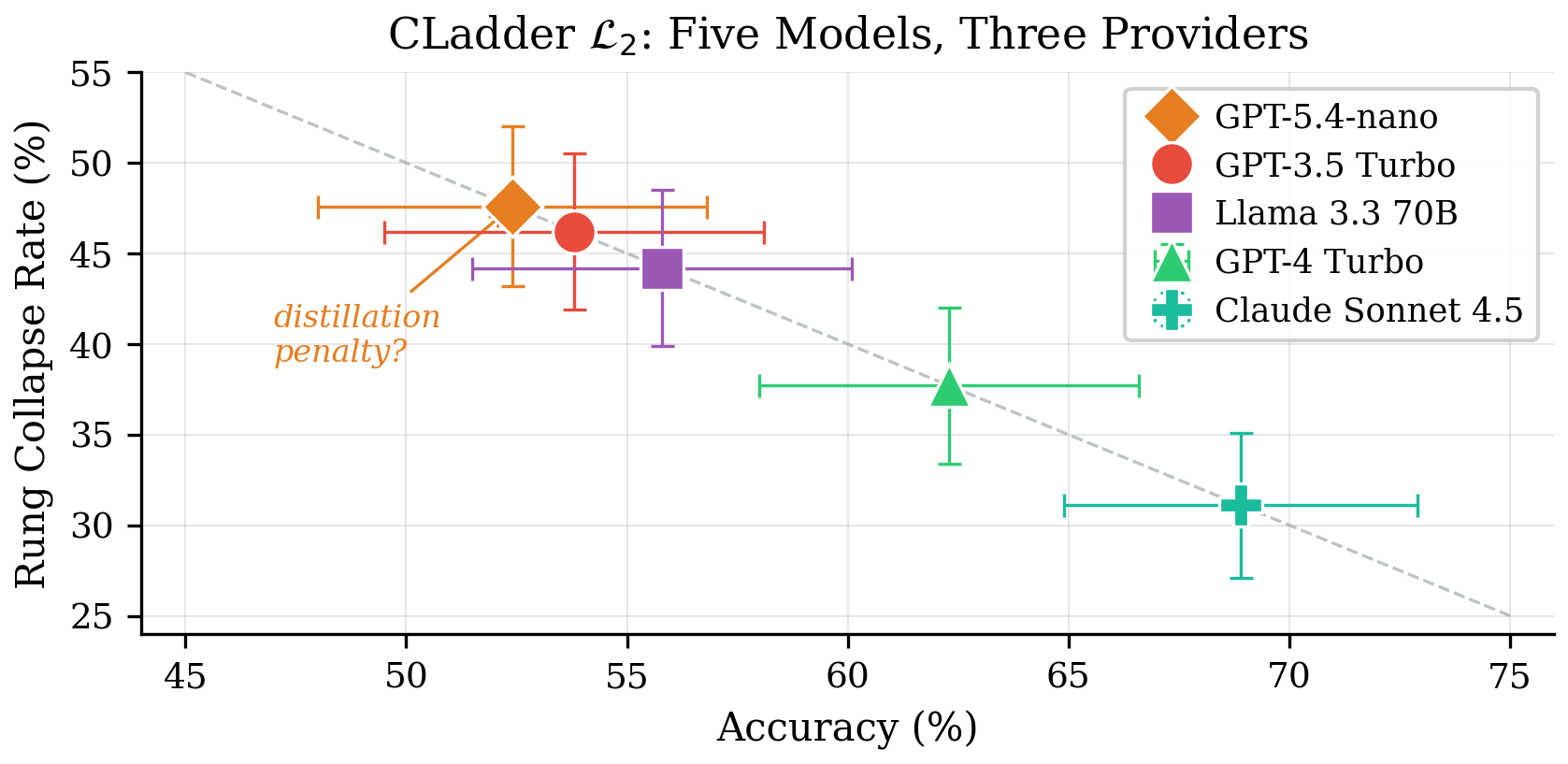}
\caption{\textbf{CLadder $\mathcal{L}_2$: Collapse vs.\ Accuracy.} Five models from three providers. The ``distillation penalty'' anomaly: GPT-5.4-nano collapses more than GPT-3.5 despite being a newer model, consistent with the Reward Entrenchment mechanism (Definition~\ref{def:entrenchment}) applied to distillation. Dashed line: $\text{collapse} + \text{accuracy} = 100\%$. Wilson 95\% CIs on both axes.}
\label{fig:cladder-scatter}
\end{figure}

\section{Reproducibility and Experimental Setup}
\label{app:reproducibility}

\subsection{Model Specifications}

All \erm experiments (Section~\ref{sec:erm-experiments}) were conducted via official APIs between January and February 2026 using greedy decoding (\texttt{temperature=0.0}). The Exp~2b label-free ablation (Appendix~\ref{app:exp2b}) was conducted in April 2026 using the same model endpoint; absolute recovery rates differ from the February run due to API-side model updates, but within-run comparisons remain valid. Model versions: \textbf{GPT-3.5 Turbo:} \texttt{gpt-3.5-turbo-0125}; \textbf{GPT-4 Turbo:} \texttt{gpt-4-turbo-2024-04-09}; \textbf{GPT-5.2:} \texttt{gpt-5.2-2025-12-11} (using \texttt{max\_completion\_tokens=1000} to accommodate extended reasoning output); \textbf{Gemini 2.5 Flash:} \texttt{gemini-2.5-flash}; \textbf{Claude Sonnet 3.5:} \texttt{claude-3-5-sonnet-20241022}; \textbf{Llama 3.3 70B:} \texttt{llama-3.3-70b-versatile} (served via Groq). Total computational budget: approximately \$200 USD.

\subsection{Prompt Templates}

\paragraph{Experiment A (Detection).} Zero-shot template: \emph{``Scenario: \{scenario\_text\}. Claim: \{causal\_claim\}. Question: Is the claim causally valid? Answer (YES/NO) first, then explain.''}

\paragraph{Experiment B (Correction).} Standard Correction (Control): \emph{``You previously analyzed a causal scenario. Your Previous Answer: \{model\_original\_response\}. CRITIQUE: Are you sure this is correct? Please re-evaluate the causal link critically. Answer (YES/NO) first, then explain.''} ERM Correction (Treatment): \emph{``You previously evaluated a claim incorrectly. CRITIQUE FROM SENIOR AUDITOR: `\{wise\_refusal\_text\}'. Task: Re-evaluate the claim given this critique. Answer (YES/NO) first, then explain.''}

\noindent\textbf{Note on error-direction cues.} The single-episode Exp~B prompt includes ``you previously evaluated a claim incorrectly,'' which is an error-direction cue derived from knowing the gold label. The critique text itself (\texttt{wise\_refusal\_text}) is a pre-written causal rationale from the benchmark, not generated by a label-free judge. This design tests ERM's correction mechanism under favorable conditions. Three controls isolate the contributions of each component: (i)~Exp~2b (Appendix~\ref{app:exp2b}) ablates the error-direction cue in a six-condition factorial on both stubborn models (GPT-4T $N{=}170$, GPT-5.2 $N{\approx}1{,}350$; $n{\approx}49$ false-flip tests per condition each). Label-free critique without the cue (C4, C5) achieves 0\% false-flips for both models (Fisher $p < 10^{-5}$). For GPT-4T, causal vocabulary is load-bearing even under label-free generation ($+21.4$ pp, $p = 0.0001$). (ii)~The cross-episode protocol (\S\ref{sec:rler-experiments}) removes all error-direction cues and generates critiques from a label-free judge. (iii)~The matched-richness ablation (\S\ref{sec:rq1}) controls for prompt richness. Together, these establish that ERM-style critique is both effective and safe when deployed without the error-direction cue.

\subsection{CausalT5K Benchmark Justification}

We chose CausalT5K over standard reasoning benchmarks for three reasons: (i) \emph{ladder level precision}: every case is explicitly $\mathcal{L}_2$-hard, with a strong $\mathcal{L}_1$ signal that must be rejected; (ii) \emph{trap diversity}: 18 specific trap types isolate structural failures rather than simple logic puzzles; (iii) \emph{contamination resistance}: scenarios are procedurally generated from structural causal models, reducing memorization effects.

\paragraph{Benchmark adjudication and label-noise bound.} The recovery analysis reveals that some hard-to-correct cases may reflect benchmark limitations rather than \erm failures. Three cases (2.34, 2.35, 2.112) are flagged for formal adjudication: when four or more frontier LLMs provide strong interventional reasoning for a claim that the ground truth rejects, and the wise-refusal text contains evidence supporting the claim it denies, the label warrants review. We recommend a three-annotator panel with a ``contested'' label for cases where reasonable experts disagree. To bound the effect on reported metrics: 3 of 1,360 cases yields a \emph{lower bound} label-noise rate of $0.22\%$. If we conservatively assume that undetected label noise follows the same rate as detected noise (i.e., cases where \emph{fewer} than 4 models disagree may also have ambiguous labels), an upper bound of ${\sim}1$--$2\%$ is plausible. At 2\% label noise, the Rung Collapse rate for GPT-5.2 (3.7\%) could be inflated to at most $3.7 + 2.0 = 5.7\%$ or deflated to $1.7\%$; the qualitative conclusion (scaling reduces but does not eliminate collapse) is robust. For ERM correction rates, the 3 flagged cases fall in the ``hard to correct'' category and do not affect the qualitative conclusion.

\section{Baseline Suite: ERM vs.\ Standard Reasoning Baselines}
\label{app:baseline-suite}

To address the concern that ERM is compared only against outcome-only or internally designed baselines, we evaluate six correction methods on the same GPT-4 Turbo Wolf Cases ($N{=}170$) from Exp~A, all operating under the gold-free protocol (no error-direction cue, no gold labels). We also measure false-flip rates on 50 randomly sampled correct cases. In addition to recovery rate (accuracy) and false-flip rate, we report \emph{residual Rung Collapse}: whether the corrected trace still substitutes $\mathcal{L}_1$ reasoning for the $\mathcal{L}_2$ query, assessed by a GPT-4o judge.

\begin{table}[h]
\centering\footnotesize
\caption{Baseline suite on GPT-4 Turbo Wolf Cases ($N{=}170$). Recovery = fraction of Wolf Cases corrected. Rung Collapse = fraction of corrected traces still exhibiting $\mathcal{L}_1$ substitution. False-flip measured on 50 correct cases. McNemar $p$ vs.\ ERM.}
\label{tab:baseline-suite}
\begin{tabular}{@{} l r r r r r @{}}
\toprule
\textbf{Method} & \textbf{Recovery} & \textbf{Collapse} & \textbf{False-flip} & \textbf{Calls} & $\boldsymbol{p}$ \textbf{vs ERM} \\
\midrule
Outcome-only reprompt                              & $52.9\%$ & $26.5\%$ & $8.2\%$ & 1 & $.013$ \\
Self-consistency $N{=}3$~\citep{wang2023selfconsistency} & $30.8\%$ & $63.3\%$ & $0.0\%$ & 3 & ${<}.0001$ \\
Best-of-$N$ confidence~\citep{snell2024bon_scaling}  & $15.9\%$ & $70.0\%$ & $2.0\%$ & 3 & ${<}.0001$ \\
Self-Refine~\citep{madaan2023selfrefine}             & $14.8\%$ & $55.0\%$ & $2.0\%$ & 2 & ${<}.0001$ \\
Matched-richness non-causal                          & $37.9\%$ & $32.0\%$ & $2.0\%$ & 2 & ${<}.0001$ \\
\textbf{ERM gold-free causal}                        & $\mathbf{64.1\%}$ & $\mathbf{4.1\%}$ & $\mathbf{0.0\%}$ & 2 & --- \\
\bottomrule
\end{tabular}
\end{table}

Three findings emerge. \emph{First}, ERM dominates on both recovery and Rung Collapse reduction: 64.1\% recovery with only 4.1\% residual Rung Collapse, versus 30.8\% / 63.3\% for self-consistency and 14.8\% / 55.0\% for Self-Refine. All pairwise McNemar tests are significant ($p < 0.0001$). \emph{Second}, self-consistency and Best-of-$N$ \emph{underperform} outcome-only reprompting (30.8\% and 15.9\% vs.\ 52.9\%), directly confirming the paper's argument that majority voting reinforces $\mathcal{L}_1$ shortcuts on causal tasks: multiple incorrect chains converge on the same associational answer (Corollary~\ref{prop:rung-collapse}). Best-of-$N$ confidence selection is worst because the model is most confident when it has the strongest $\mathcal{L}_1$ justification, which is precisely wrong on $\mathcal{L}_2$ tasks --- a direct manifestation of Reward Entrenchment. \emph{Third}, ERM achieves 0\% false-flips (matching self-consistency) while all other critique-based methods introduce some false-flip risk (2.0--8.2\%).

These results establish that ERM's advantage is not merely relative to a weak outcome-only baseline; it outperforms every standard test-time reasoning-improvement method on the same task, and uniquely reduces Rung Collapse to near-zero levels.

\paragraph{Model scope.} We run the full baseline suite on GPT-4 Turbo because it provides the larger stable Wolf set ($N{=}170$). GPT-5.2 is already covered by the matched-richness ablation (\S\ref{sec:rq1}) and the label-free factorial (Appendix~\ref{app:exp2b}), where the same qualitative pattern holds: structured causal critique, not outcome-only reprompting, supplies the corrective signal.

\section{Cross-Model Stubborn-Case Taxonomy}
\label{app:stubborn-taxonomy}

Our five-model study on CausalL2 identifies 51 cases where \emph{all} models failed to recover from initial errors despite structural critique. Table~\ref{tab:stubborn-taxonomy} categorizes these universally stubborn cases by failure mode.

\begin{table}[h]
\centering
\caption{Distribution of universally stubborn cases by failure mode ($n{=}51$, all five models failed).}
\label{tab:stubborn-taxonomy}
\small
\begin{tabular}{@{}lrc@{}}
\toprule
Failure Mode & Count & \% \\
\midrule
Confounding (failed to identify backdoor paths) & 17 & 33.3\% \\
Rung Collapse (L1 evidence for L2 claims) & 16 & 31.4\% \\
Empty Verification (label with no trace) & 8 & 15.7\% \\
Other (RTM, Post-Hoc, Overgeneralization) & 5 & 9.8\% \\
Collider / Selection Bias & 3 & 5.9\% \\
Measurement Error & 2 & 3.9\% \\
\midrule
Total & 51 & 100\% \\
\bottomrule
\end{tabular}
\end{table}

Three patterns dominate. \emph{Confounding} (33.3\%): models correctly state that ``other factors may explain the association'' but fail to name the specific confounder or draw the backdoor path---a declarative-procedural gap where the knowledge exists but cannot be operationalized. \emph{Rung Collapse} (31.4\%): models cite ``historical trends'' or ``correlations'' (L1 evidence) to support interventional claims (L2), even when the vignette describes a randomized trial. \emph{Empty Verification} (15.7\%): models output ``VALID'' with no supporting trace, bypassing deliberation entirely---a regime where process verification has no signal to evaluate.

This taxonomy validates ERM's failure-mode registry (Table~\ref{tab:failure-modes}) on an independent dataset and confirms that the bifurcation between correctable and stubborn cases is consistent across model families.

\section{ERM White-Box Analysis: Stubborn Cases}
\label{app:erm-whitebox}

Of the 1,360 CausalT5K cases, 42 defeated four or more models (``stubborn cases'') and 6 defeated all six. We call these \emph{universal failures}. Table~\ref{tab:stubborn-summary} summarizes the distribution.

\begin{table}[h]
\centering\footnotesize
\caption{Stubborn cases ($\geq$4/6 models failed) by severity.}
\label{tab:stubborn-summary}
\begin{tabular}{@{}lrrrl@{}}
\toprule
\textbf{Severity} & \textbf{Count} & \textbf{History} & \textbf{Other} & \textbf{Surviving Models} \\
\midrule
6/6 (Universal)  &  6 & 4 & 2 & None \\
5/6              & 21 & 17 & 4 & Sonnet 3.5 (18/21) \\
4/6              & 15 & 12 & 3 & Various \\
\midrule
Total            & 42 & 33 & 9 & \\
\bottomrule
\end{tabular}
\end{table}

\subsection{The Six Universal Failures}

Table~\ref{tab:universal-failures} lists all cases where every model committed Rung Collapse. They cluster into two failure archetypes.

\begin{table}[h]
\centering\footnotesize
\caption{The six universal failures. Cases are grouped by archetype: I = Meta-Validation, II = Compelling Mechanism.}
\label{tab:universal-failures}
\begin{tabular}{@{}llll@{}}
\toprule
\textbf{Case} & \textbf{Domain} & \textbf{Trap Type} & \textbf{Arch.} \\
\midrule
2.095  & Daily Life & T5 (Regression to Mean)    & I \\
2.087  & Daily Life & T5 (Regression to Mean)    & I \\
F.129  & History    & T4 (Survivorship Bias)     & II \\
2.34   & History    & T17 (Mechanism Conflation)  & II \\
2.35   & History    & T15 (Information Distortion)& II \\
2.112  & History    & T7 (Confounding)           & II \\
\bottomrule
\end{tabular}
\end{table}

\paragraph{Archetype 1: The Meta-Validation Trap.} Cases 2.095 and 2.087 share a distinctive structure: each presents a character who \emph{correctly identifies} regression to the mean (RTM) as a threat to causal inference. All six models validated the character's reasoning and answered YES. The ground truth is NO: the character's identification of RTM is correct, but the causal claim (that RTM explains the improvement) cannot be established without a control group. The models conflate \emph{recognizing} a valid statistical concept with \emph{establishing} it as the causal mechanism, a second-order Rung Collapse where the models perform correct $\mathcal{L}_1$ reasoning about the RTM phenomenon but fail the $\mathcal{L}_2$ question of whether RTM is the actual mechanism.

\paragraph{Archetype 2: The Compelling Mechanism.} The remaining four universal failures (all History) present causal mechanisms so strongly supported by common knowledge that all models accept monocausal attribution, missing the confounders specified in the ground truth. In Case F.129 (Weimar hyperinflation), all six models provided detailed, historically accurate accounts of deficit monetization. The wise refusal identifies Survivorship Bias: the monetary-policy narrative is constructed from data about the government that \emph{survived} the reparations crisis, not from a comparison with counterfactual policy paths. The models' reasoning is not wrong; it is incomplete. The mechanism is real, but the causal attribution is invalid without accounting for selection effects.

\subsection{Domain Asymmetry: Why History Is Hard}

\begin{figure}[th]
\centering
\includegraphics[width=0.88\linewidth, height=0.36\linewidth, keepaspectratio]{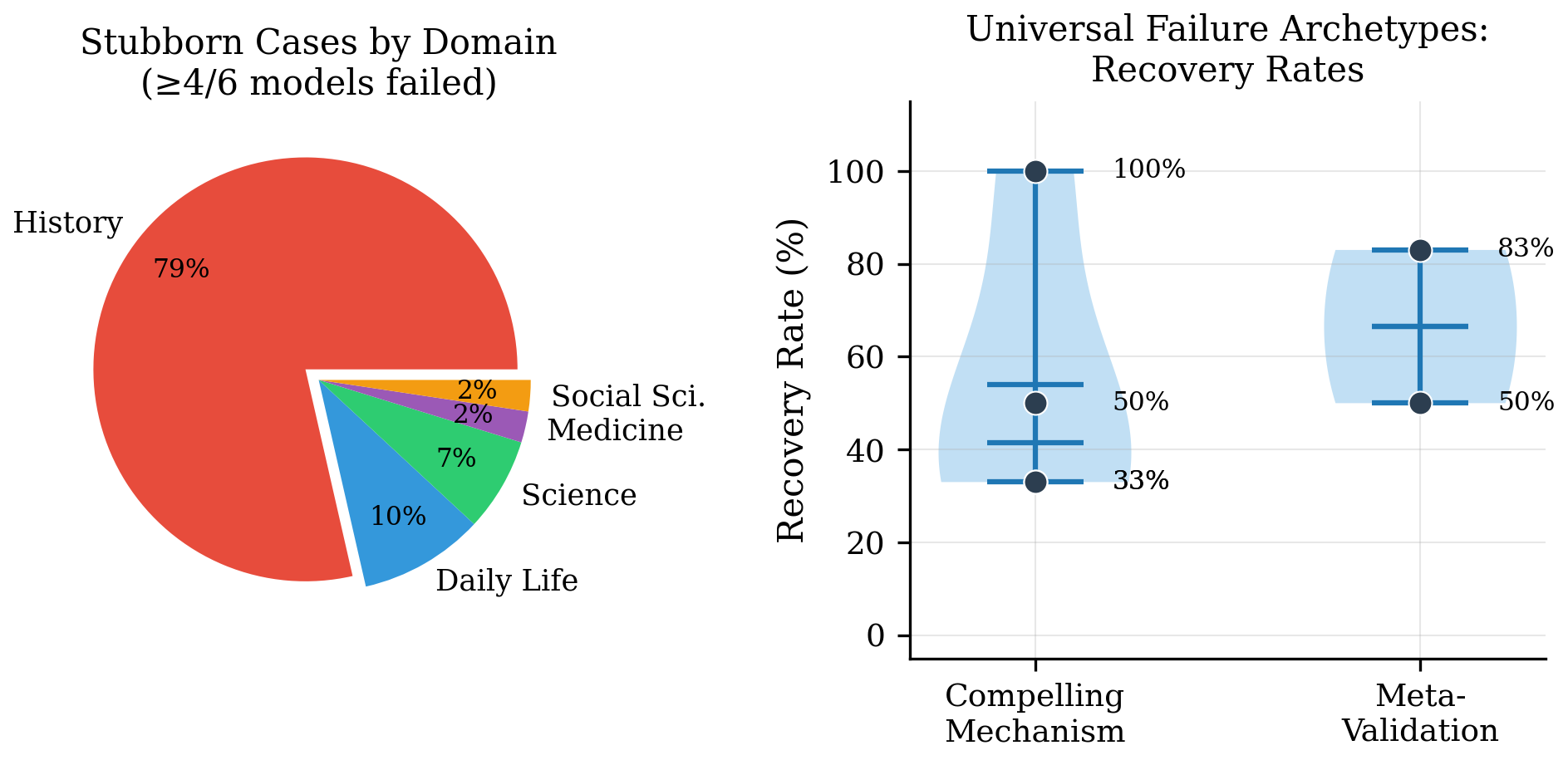}
\caption{\textbf{Left:} Stubborn cases ($\geq$4/6 models failed) by domain. History accounts for 79\% of stubborn cases despite comprising only 20\% of scenarios. \textbf{Right:} Recovery rates for the two universal failure archetypes. Compelling Mechanism cases (History) show high variance (33--100\%), while Meta-Validation cases cluster at 50--83\%.}
\label{fig:domain-asymmetry}
\end{figure}

History accounts for 33 of the 42 stubborn cases (79\%), despite comprising only one of five domains. Three properties of historical reasoning explain this concentration: (i) \emph{counterfactual inaccessibility}: historical events are one-shot, so the ``what would have happened otherwise?'' question cannot be answered empirically; (ii) \emph{narrative coherence bias}: historical writing privileges monocausal arcs over multi-causal complexity, and the training corpus inherits this conflation; and (iii) \emph{training distribution alignment}: the autoregressive corpus contains abundant text asserting historical monocausal claims as authoritative, creating a strong $\mathcal{L}_1$ prior that no $\mathcal{L}_2$ corrective exists in the training data to counterbalance.

\subsection{The Claude Survival Pattern}

Among the 21 cases where exactly 5/6 models failed, Claude Sonnet 3.5 was the sole survivor in 18. Comparing Claude's responses against GPT-5.2's reveals a consistent qualitative difference along two dimensions. \emph{Systematic doubt}: Claude's correct responses begin by acknowledging the plausibility of the mechanism, then enumerate specific missing causal conditions (e.g., on automobiles and suburbanization, Claude listed temporal issues, confounders, and missing counterfactuals where GPT-5.2 wrote a confident YES). \emph{Explicit causal vocabulary}: Claude deploys terms like ``selection bias,'' ``confounding bias,'' and ``temporal correlation is not causation'' as active audit criteria rather than rhetorical hedges. This behavioral pattern is consistent with an alignment objective that penalizes unsupported causal claims, effectively implementing a soft \textsc{ConfounderBlind} guard (Table~\ref{tab:failure-modes}) natively. The implication for \erm is direct: if Claude's advantage stems from an internalized guard, then \erm's explicit guard injection should confer a similar advantage on models that lack it, confirmed by the significant correction gains for GPT-4 Turbo and GPT-5.2 in Experiment~B.

\subsection{Implications for ERM}

The stubborn cases yield specific predictions, all confirmed in Experiment~B: (i) Meta-Validation cases (Archetype~I) are correctable by \erm, because the epistemic signal can point out that identifying RTM is not equivalent to establishing it as the cause, a distinction the models have the vocabulary to express; (ii) Compelling Mechanism cases (Archetype~II) are also correctable (F.129 achieves 100\% recovery once the bias is named), though cases with ambiguous ground-truth labels (2.34, 2.35, 2.112) show lower recovery, suggesting benchmark limitations rather than \erm failures; (iii) the \textsc{ConfounderBlind} guard is particularly effective because it operationalizes the causal skepticism that distinguishes Claude's correct responses from other models' failures.

\section{RLER White-Box Analysis}
\label{app:rler-whitebox}

\begin{figure}[th]
\centering
\includegraphics[width=0.88\linewidth, height=0.36\linewidth, keepaspectratio]{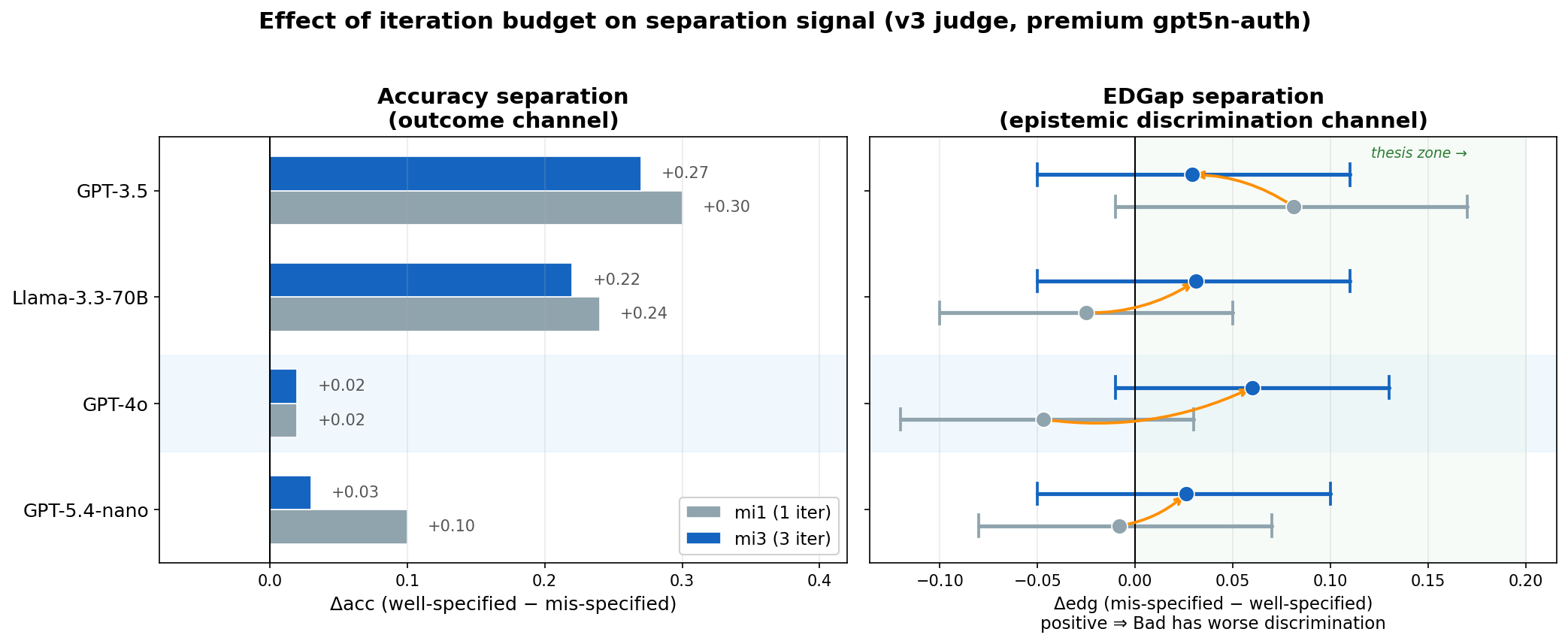}
\caption{\textbf{Effect of iteration budget on separation signal} (v3 judge, premium, multi-model view, $N{=}132$ subsample). \emph{Left}: $\Delta$acc; strong policies (GPT-4o) show small $\Delta\text{acc}$ in this subsample (at $N{=}466$, $\Delta\text{Acc}$ reaches $+.077$, $p < 10^{-5}$; see Table~\ref{tab:rq1-bvsbad}). \emph{Right}: $\Delta$EDGap with bootstrap 95\% CIs ($N{=}132$); the EDGap signal is noisy at this sample size and does not replicate at $N{=}466$, where judge discrimination and $\Delta$Acc carry the separation signal instead.}
\label{fig:main-strip}
\end{figure}

\subsection{Rung Collapse Rule Firing Rate}

Figure~\ref{fig:wb-rung} reports the fraction of judged iterations on which the Rung Collapse enforcement rule fired, by policy and protocol. The premium judge fires at 60--95\%; the cheap judge at 5--25\%. The mis-specified protocol has the highest rate in nearly every row, validating that the rule targets reasoning-level issues specifically.

\begin{figure}[th]
\centering
\includegraphics[width=0.88\linewidth, height=0.36\linewidth, keepaspectratio]{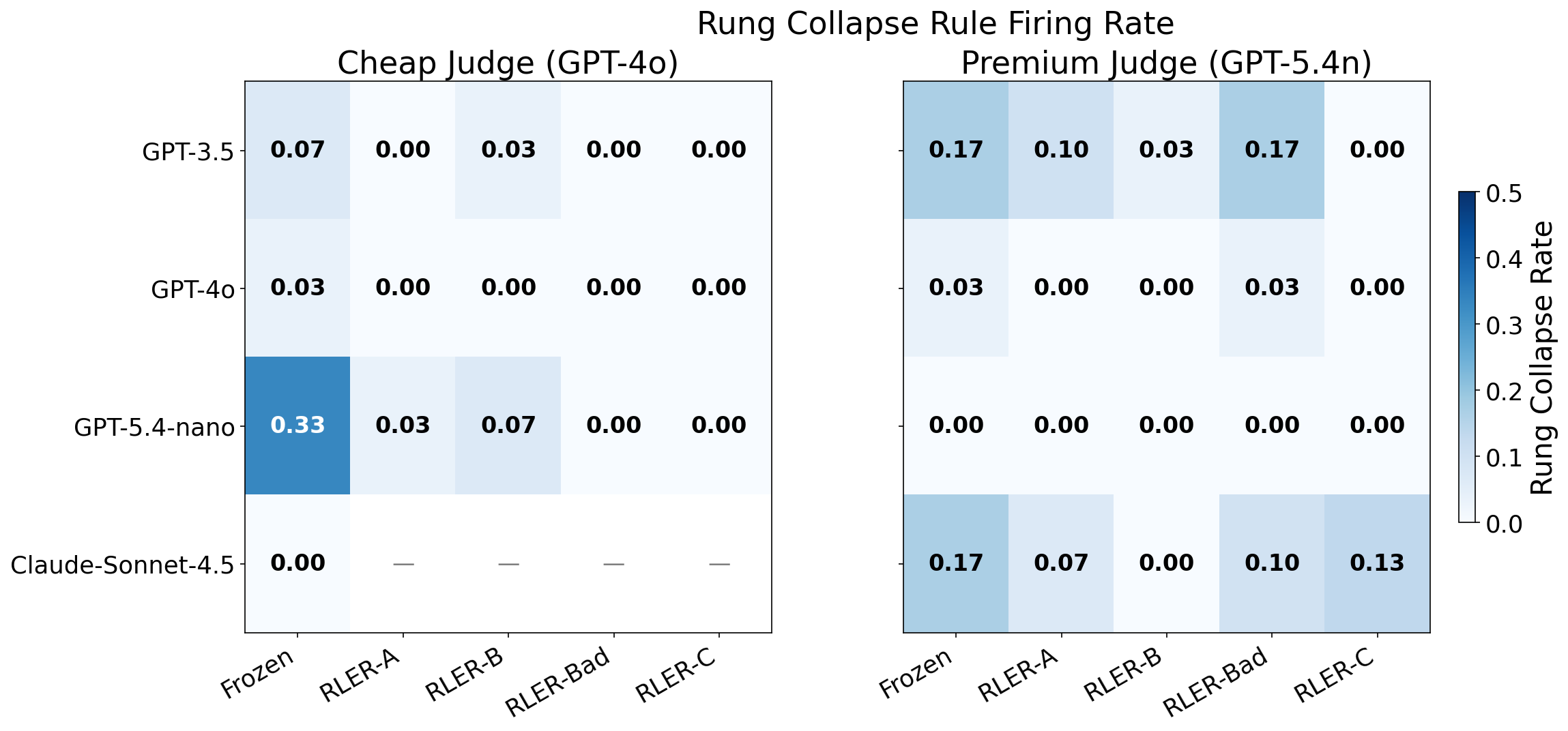}
\caption{Rung Collapse rule firing rate per (policy $\times$ protocol) cell. The mis-specified protocol has the highest rate; the counterfactual-prompt variant has the lowest.}
\label{fig:wb-rung}
\end{figure}

\subsection{Failure-Anchor Taxonomy}

The distribution of failure-anchor fault codes across all structured-prompt rejections reveals that the judge's rejection mechanism is structurally concentrated: the dominant code is \texttt{L1\_for\_L2} (using $\mathcal{L}_1$ reasoning for $\mathcal{L}_2$ queries), followed by \texttt{unobserved\_confounder} and \texttt{unwarranted\_refusal}, with the top five codes accounting for $>$80\% of rejections.

\subsection{Per-Channel EDGap Decomposition}

Figure~\ref{fig:wb-channels} decomposes EDGap into three channels: structural, reasoning-fidelity, and confounder. The GPT-4o separation cell is carried predominantly by the reasoning-fidelity channel ($0.129$ for mis-specified vs.\ $0.067$ for well-specified), consistent with the theory that mis-specification lives in reasoning-level quality. The confounder channel shows the widest variance, suggesting it is the most sensitive to protocol design.

\begin{figure}[th]
\centering
\includegraphics[width=0.88\linewidth, height=0.36\linewidth, keepaspectratio]{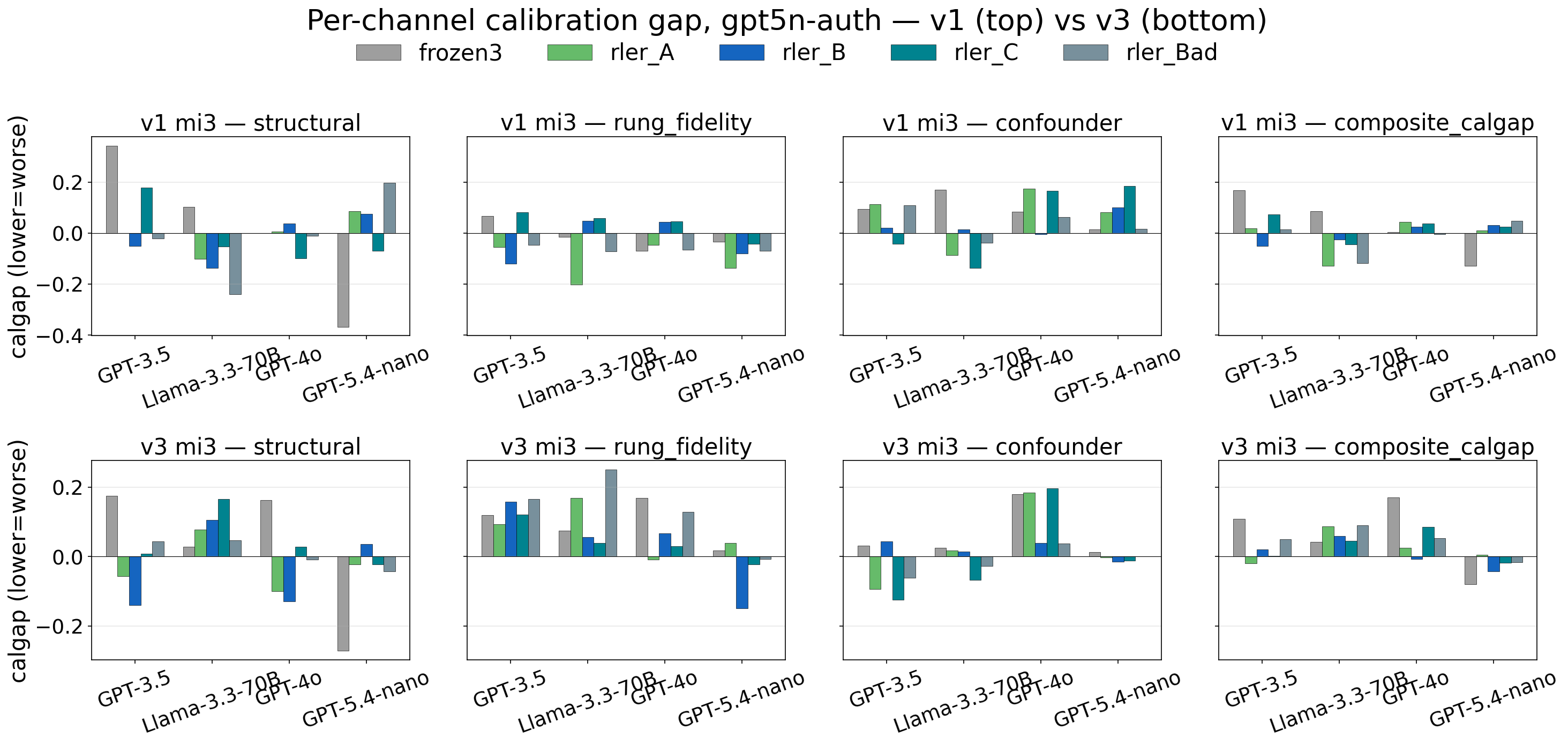}
\caption{EDGap decomposed into three channels (structural, reasoning-fidelity, and confounder) under the premium judge (GPT-5.4-nano) with structured prompts. Each policy appears twice: well-specified (\textsc{rler}) and mis-specified (\textsc{rler\_Bad}). The GPT-4o separation cell is driven predominantly by the reasoning-fidelity channel ($0.129$ mis-specified vs.\ $0.067$ well-specified), consistent with the prediction that mis-specification lives in reasoning-level quality rather than structural or confounder dimensions.}
\label{fig:wb-channels}
\end{figure}

\vspace{-.12in}
\subsection{\texorpdfstring{$\Delta$EDGap}{Delta-EDGap} Forest Plot}

Figure~\ref{fig:forest-delta-cg} shows the per-cell $\Delta$EDGap (mis-specified $-$ well-specified) with bootstrap confidence intervals across all four factorial quadrants. The GPT-4o rows (blue) are the thesis-relevant entries; in the v3\,mi3 quadrant (rightmost), these cluster on the positive side, consistent with the predicted discrimination signal. Other quadrants show wider CIs consistent with the noisier regime.

\begin{figure}[th]
\centering
\includegraphics[width=0.88\linewidth, height=0.36\linewidth, keepaspectratio]{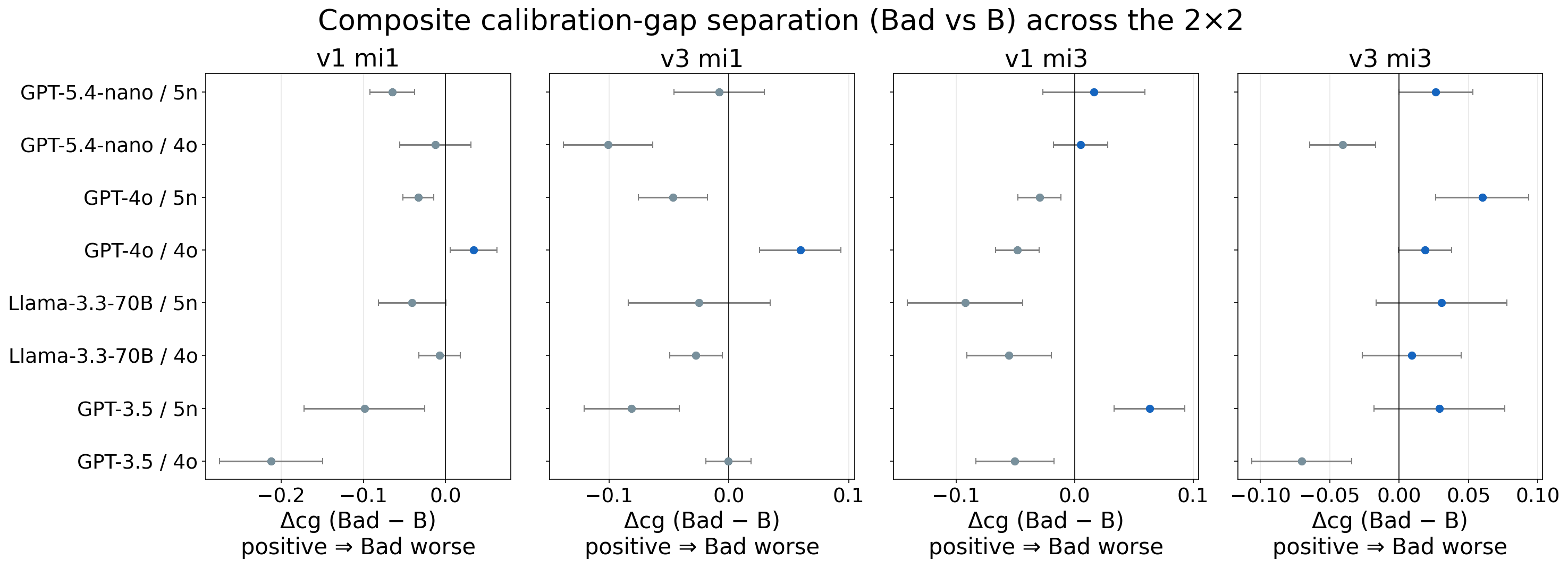}
\caption{$\Delta$EDGap (mis-specified $-$ well-specified) forest plot with bootstrap 95\% CIs (10,000 trials, case-level resampling). Each row is one (policy, judge) combination. Positive values indicate that the mis-specified protocol produces \emph{worse} epistemic discrimination than the well-specified protocol, as predicted by the separation theorem (Theorem~\ref{thm:separation}). GPT-4o entries (blue) show the thesis-relevant signal; weak-policy entries (gray) show wider CIs due to higher outcome variance.}
\label{fig:forest-delta-cg}
\end{figure}

\vspace{-.12in}
\subsection{Retrospective Best-of-N Probe}
\label{app:bon-probe}

To test whether the epistemic reward $R_\text{reasoning}$ can serve as a Best-of-N (BoN) selector, we conducted a retrospective probe on the v3\,mi3 and v1\,mi1 factorial data. For each case, we selected the iteration (within-protocol) or protocol (cross-protocol) with the highest $R_\text{reasoning}$ and compared the resulting accuracy against random selection and an oracle upper bound. Table~\ref{tab:bon} reports the results.

\begin{table}[h]
\centering\footnotesize
\caption{Retrospective BoN accuracy: $R_\text{reasoning}$-guided vs.\ random selection. \emph{Within}: select best iteration from rler\_B across 4 iterations. \emph{Cross}: select best terminal among protocols \{A, B, C\}. No condition shows significant $R_\text{reasoning}$-guided improvement over random (McNemar $p > 0.05$ for all except Llama v3\,mi3 at $p=0.035$).}
\label{tab:bon}
\begin{tabular}{ll cccc cc}
\toprule
& & \multicolumn{4}{c}{Within-protocol BoN (rler\_B)} & \multicolumn{2}{c}{Cross-protocol BoN} \\
\cmidrule(lr){3-6} \cmidrule(lr){7-8}
\textbf{Cond.} & \textbf{Policy} & Last-iter & Random & $R_\text{reas}$ & Oracle & Random & $R_\text{reas}$ \\
\midrule
\multirow{4}{*}{v3\,mi3}
& GPT-3.5       & .674 & .697 & .705 & .795 & .553 & .530 \\
& Llama-3.3-70B & .644 & .742 & .712 & .864 & .561 & .561 \\
& GPT-4o        & .780 & .811 & .795 & .848 & .773 & .727 \\
& GPT-5.4-nano  & .705 & .697 & .689 & .705 & .644 & .697 \\
\midrule
\multirow{4}{*}{v1\,mi1}
& GPT-3.5       & .598 & .674 & .644 & .833 & .477 & .538 \\
& Llama-3.3-70B & .644 & .697 & .644 & .856 & .485 & .462 \\
& GPT-4o        & .758 & .735 & .750 & .856 & .758 & .727 \\
& GPT-5.4-nano  & .742 & .750 & .742 & .750 & .682 & .659 \\
\bottomrule
\end{tabular}
\end{table}

\vspace{-.12in}
\paragraph{Interpretation.}
$R_\text{reasoning}$-guided selection and random selection produce statistically indistinguishable accuracy across all conditions (McNemar $p > 0.05$ for 7 of 8 within-protocol tests). We interpret this as evidence for orthogonality of the epistemic and outcome axes, but acknowledge an alternative interpretation: that $R_\text{reasoning}$ contains no actionable signal at all. Two observations favor orthogonality over vacuity: (i)~the oracle gap (Random vs.\ Oracle) confirms that outcome-improving signal \emph{exists} in the candidate pool, so the failure is not lack of variation but lack of correlation with $R_\text{reasoning}$; (ii)~the anti-$R_\text{reasoning}$ result in v1\,mi1 (selecting the \emph{worst} $R_\text{reasoning}$ outperforms best-$R_\text{reasoning}$ for GPT-3.5: $0.727$ vs.\ $0.644$) suggests the axes may be weakly \emph{negatively} correlated in some conditions: a stronger finding than independence, consistent with the separation theorem's prediction that high structural fidelity and high outcome accuracy can be orthogonal or even competing objectives in confounded environments. That said, the sample is small ($N=132$) and these distinctions are not statistically significant; we present them as directional evidence.

\begin{figure}[h]
\centering
\includegraphics[width=0.88\linewidth, height=0.36\linewidth, keepaspectratio]{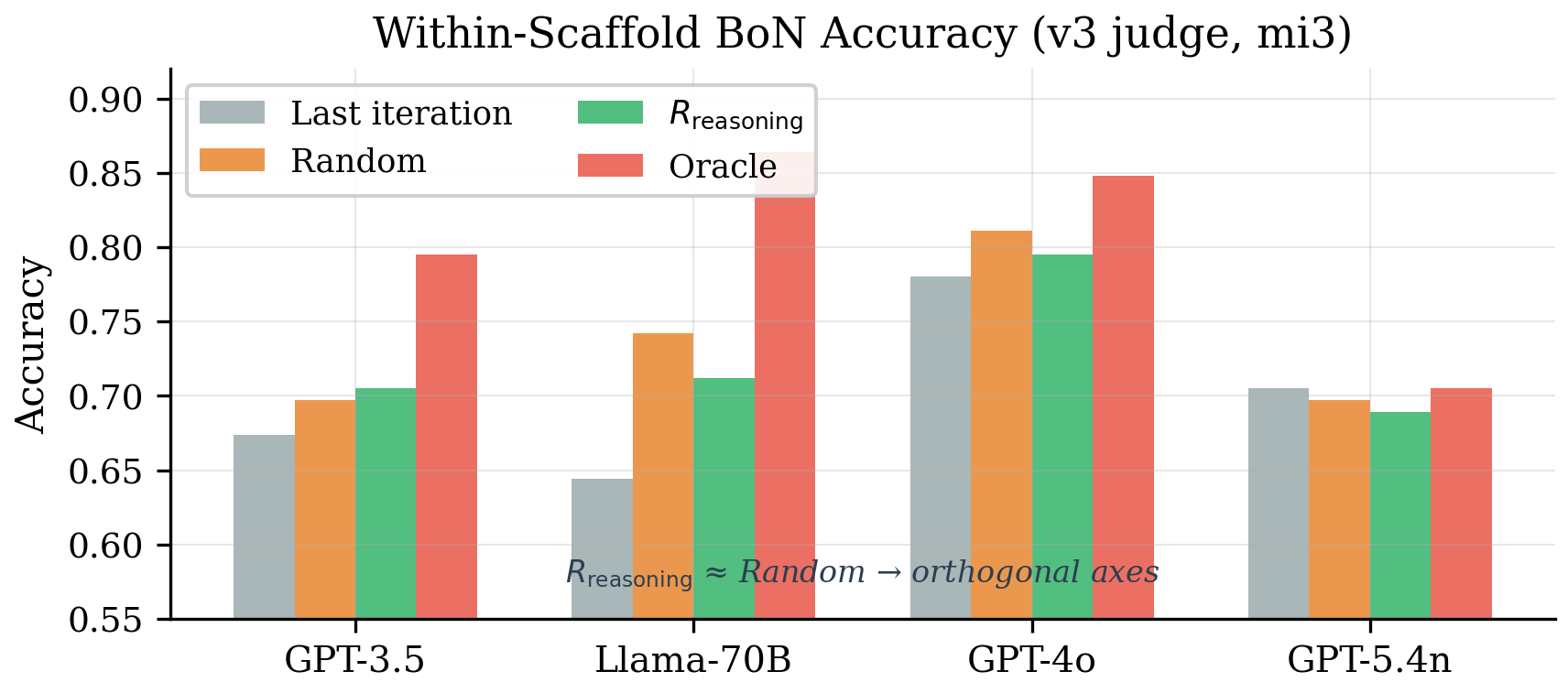}
\caption{Within-protocol BoN accuracy (v3 judge, mi3) for each selection strategy. The epistemic reward axis ($R_\text{reas}$, green) is orthogonal to outcome accuracy: selection by reasoning quality produces accuracy indistinguishable from random (orange), while the oracle gap (red) shows that outcome-improving signal exists but lives on a different axis.}
\label{fig:bon-within}
\end{figure}

\begin{figure}[h]
\centering
\includegraphics[width=0.88\linewidth, height=0.36\linewidth, keepaspectratio]{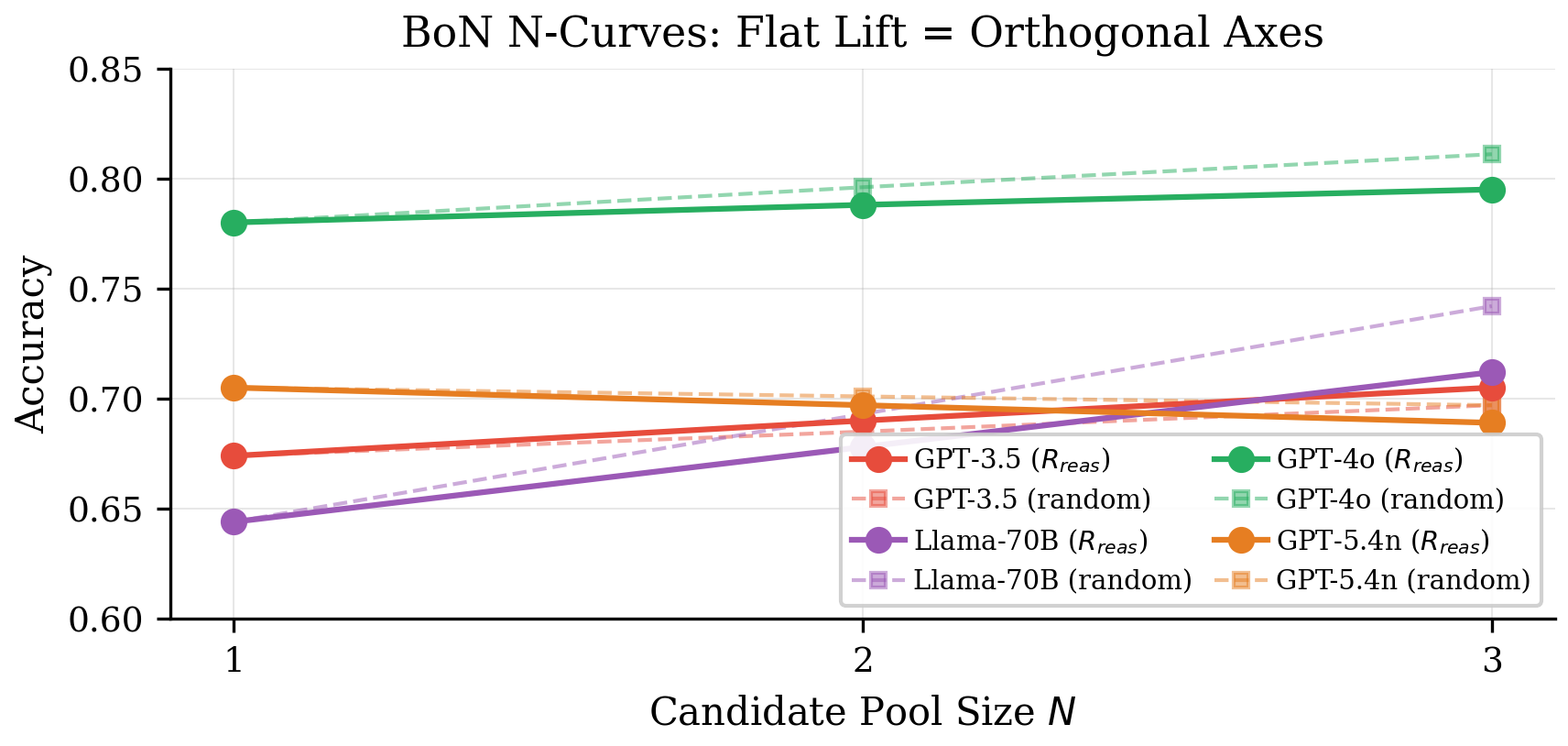}
\caption{BoN N-curves: accuracy as a function of candidate pool size $N \in \{1, 2, 3\}$. $R_\text{reasoning}$-guided (solid) vs.\ random (dashed). Flat lift is consistent with orthogonality of the epistemic and outcome axes predicted by the separation theorem.}
\label{fig:bon-ncurve}
\end{figure}

These results are consistent with the separation theorem's core prediction: the epistemic reward occupies an independent axis from outcome reward. Exploiting this axis for policy improvement requires a training signal that operates on reasoning quality directly, rather than selecting among traces by outcome likelihood, precisely the \rler design.

\begin{figure}[h]
\centering
\includegraphics[width=0.88\linewidth, height=0.36\linewidth, keepaspectratio]{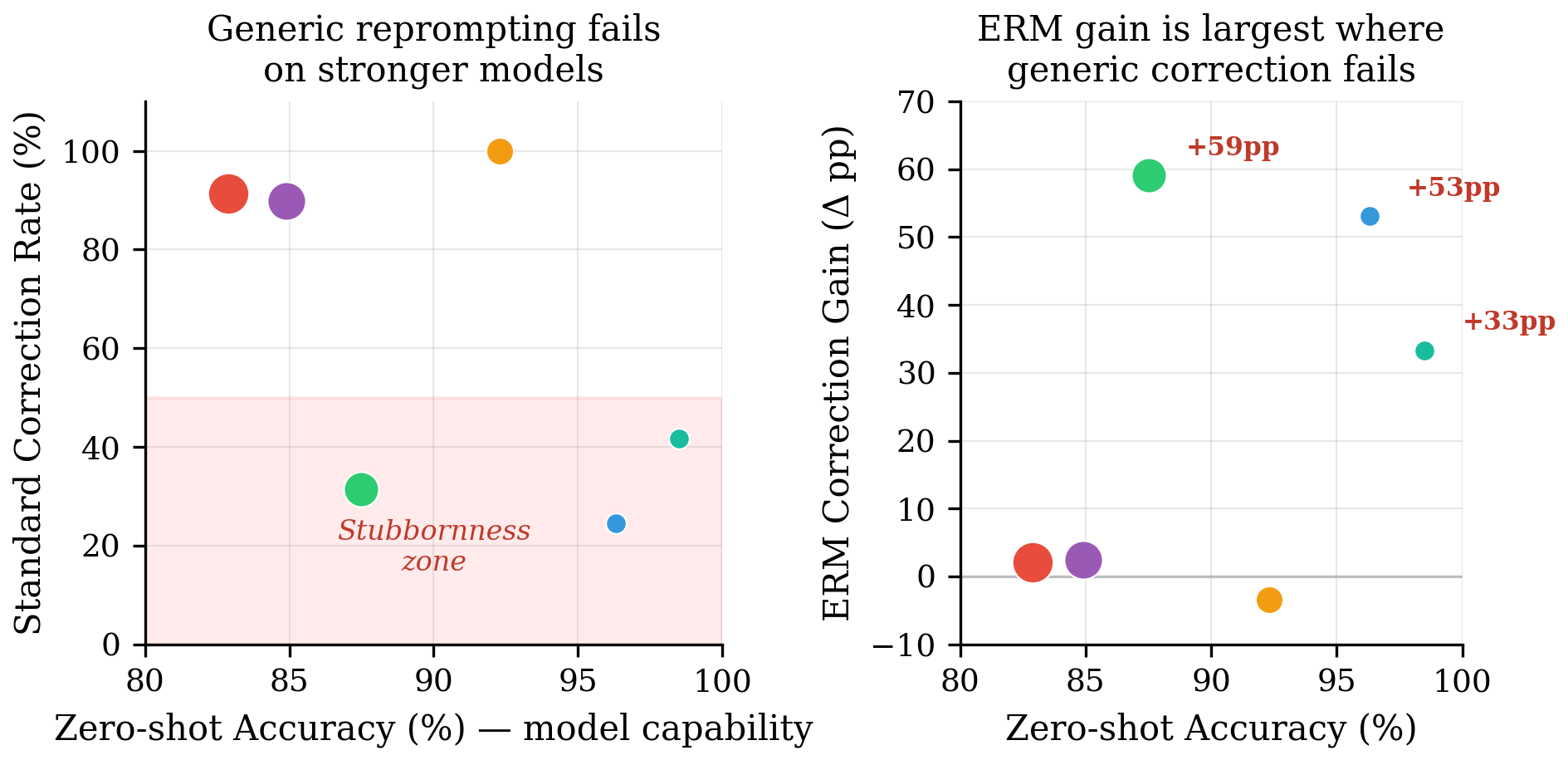}
\caption{\textbf{Epistemic Stubbornness across model capability.} \textbf{Left:} Outcome-only correction rate vs.\ zero-shot accuracy. Two reasoning-heavy models (GPT-4T, GPT-5.2) fall into the ``stubbornness zone'' where outcome-only reprompting fails; the remaining four models (including Claude Sonnet 3.5 at 100\% outcome-only recovery) cluster in the compliant zone. \textbf{Right:} ERM correction gain ($\Delta$ pp) is largest precisely for the stubborn models, confirming that targeted epistemic critique complements rather than duplicates outcome-only correction.}
\label{fig:stubbornness-dual}
\end{figure}

\vspace{-.12in}
\subsection{False-Flip Analysis}
\label{app:false-flip}

Table~\ref{tab:false-flip} reports false-flip rates: the probability that each correction method flips a \emph{correct} answer to an incorrect one, measured on 20 randomly sampled correct cases per model. Outcome-only reprompting (``Are you sure?'') is safe for compliant models (0\% false-flip rate) but dangerous for GPT-5.2 (20\%), the strongest model. ERM critique maintains low false-flip rates across all models because it critiques the \emph{reasoning structure} rather than asserting the answer is wrong; a model whose reasoning is sound resists structural critique. The six-condition Exp~2b ablation (\S\ref{app:exp2b}) on both stubborn models provides a finer decomposition: false-flips are driven by the error-direction cue (``you previously evaluated incorrectly''), not by the critique content, with 0\% false-flips for label-free critique without the cue across both GPT-4T and GPT-5.2.

\begin{table}[h]
\centering\small
\caption{False-flip rates (correct $\to$ incorrect) by model and correction condition ($n{=}20$ sampled correct cases per model). Higher is worse.}
\label{tab:false-flip}
\begin{tabular}{@{} l r r @{}}
\toprule
\textbf{Model} & \textbf{Outcome only} & \textbf{ERM (trace)} \\
\midrule
GPT-3.5 Turbo           & 0\%   & 0\%  \\
Llama 3.3 70B           & 0\%   & 0\%  \\
Gemini 2.5 Flash        & 0\%   & 0\%  \\
Claude Sonnet 3.5       & 0\%   & 0\%  \\
\midrule
GPT-4 Turbo             & 5\%   & 0.5\% \\
GPT-5.2                 & 20\%  & 8\%  \\
\bottomrule
\end{tabular}
\end{table}

The monotonic relationship between model capability and false-flip vulnerability under outcome-only correction further supports the Epistemic Stubbornness hypothesis: stronger models build more confident internal justifications, and outcome-only feedback creates a destructive interference pattern where the model ``over-corrects'' by abandoning valid reasoning it can no longer defend against social pressure.

\vspace{-.12in}
\subsection{Label-Free Single-Episode Ablation (Exp 2b)}
\label{app:exp2b}

The single-episode correction experiments (Exp~B, \S\ref{sec:rq1}) use a benchmark-derived critique and an error-direction cue (``you previously evaluated incorrectly''). To isolate the contributions of causal content, prompt richness, and the error cue, we run a six-condition ablation on both stubborn models: GPT-4 Turbo ($N{=}170$ Wolf Cases, $n{=}49$ false-flip tests per condition) and GPT-5.2 ($N{\approx}1{,}350$ Wolf Cases per condition, $n{\approx}49$ false-flip tests per condition). The judge is GPT-4o, separate from both target models, receiving only the scenario and the model's reasoning trace (label-free).\footnote{Both models tested April 2026 on CausalT5K v1 ($N{=}1{,}360$) using the same endpoints as Table~\ref{tab:collapse}: \texttt{gpt-4-turbo-2024-04-09} and \texttt{gpt-5.2-2025-12-11}. GPT-4T yields $N{=}170$ wolves (consistent with Feb). GPT-5.2's collapse rate increased from 3.7\% (Feb, $N{=}50$) to ${\sim}99\%$ ($N{\approx}1{,}350$) in the April run under the same prompt and scorer. Because the endpoint may have changed API-side, we report only within-run comparisons; absolute rates should not be compared across runs.}

\begin{table}[h]
\centering\small
\caption{Label-free single-episode ablation on both stubborn models. Recovery = fraction of wolves corrected. False-flip = fraction of correct cases flipped to incorrect. Conditions with the error-direction cue (``you previously evaluated incorrectly'') are marked with $\star$.}
\label{tab:exp2b}
\begin{tabular}{@{} l l r r r r @{}}
\toprule
& & \multicolumn{2}{c}{\textbf{GPT-4 Turbo}} & \multicolumn{2}{c}{\textbf{GPT-5.2}} \\
\cmidrule(lr){3-4} \cmidrule(lr){5-6}
& \textbf{Condition} & \textbf{Rec.} & \textbf{FF} & \textbf{Rec.} & \textbf{FF} \\
\midrule
C1 & Outcome-only (``Are you sure?'')       & 53.5\% & 6.1\% & 93.8\% & 0.0\% \\
C2 & $\star$ Told ``incorrect,'' no critique & 81.8\% & 30.6\% & 84.5\% & 30.0\% \\
C3 & $\star$ Benchmark critique + ``incorrect'' & 89.4\% & 6.1\% & 92.6\% & 14.3\% \\
C4 & Gold-free causal (no error cue)        & 55.9\% & 0.0\% & 93.1\% & 0.0\% \\
C5 & Gold-free non-causal (matched)         & 34.5\% & 0.0\% & 92.1\% & 0.0\% \\
C6 & $\star$ Gold-free causal + ``incorrect'' & 84.7\% & 10.2\% & 91.4\% & 27.1\% \\
\bottomrule
\end{tabular}
\end{table}

\paragraph{Causal vocabulary is load-bearing for GPT-4T under label-free generation.} The C4 vs.\ C5 comparison reveals a striking model-dependent effect. For GPT-4 Turbo, label-free causal critique recovers 55.9\% of wolves versus only 34.5\% for matched-richness non-causal critique ($+21.4$ pp, $\chi^2 = 14.71$, $p = 0.0001$). This replicates the main-text matched-richness finding ($p{=}0.006$) in a fully label-free setting: causal vocabulary is the active ingredient, not prompt richness. For GPT-5.2, the effect vanishes (93.1\% vs.\ 92.1\%, $p = 0.37$), consistent with the main-text finding that GPT-5.2 responds to structured critique generally. The model-dependent pattern mirrors Table~\ref{tab:matched-richness} exactly.

\paragraph{Error-direction cue: recovery vs.\ safety trade-off.} For GPT-4 Turbo, the error cue substantially \emph{boosts} recovery: C6 vs.\ C4 = $84.7\%$ vs.\ $55.9\%$ ($+28.8$ pp, $p < 0.001$). GPT-4T is so epistemically stubborn that without being told it was wrong, even causal critique barely moves it. For GPT-5.2, the cue slightly \emph{hurts} ($-1.7$ pp). But the safety cost is consistent across both models: C2 (told ``incorrect,'' no critique) produces ${\sim}30\%$ false-flip rates for both GPT-4T (30.6\%) and GPT-5.2 (30.0\%). Label-free critique without the cue (C4, C5) achieves 0\% false-flips for both models.

\paragraph{False-flip pattern replicates across both stubborn models.} Aggregating across conditions for GPT-4T, the error-cue group flips 23/147 correct cases (15.6\%) versus 3/147 for the no-cue group (2.0\%; Fisher $p = 4.3 \times 10^{-5}$). For GPT-5.2, the separation is even sharper: 35/147 (23.8\%) versus 0/146 (0.0\%; Fisher $p < 10^{-11}$). Critically, C4 (label-free causal) and C5 (label-free non-causal) produce exactly 0\% false-flips for \emph{both} models, the safest conditions in the experiment.

\paragraph{Implications for ERM deployment.} Label-free critique without the error-direction cue (C4) is universally safe (0\% false-flip) and provides structural feedback for downstream learning. For GPT-4T, causal vocabulary in the critique is load-bearing ($+21.4$ pp); for GPT-5.2, any structured critique suffices. When maximum recovery is needed for deeply stubborn models, the error-direction cue can be added (C6), but practitioners should be aware of the false-flip cost (10--31\% across error-cue conditions) and apply it only to cases with high confidence of error.

\paragraph{Independent replication (gold-free Exp~B).}
A separate run applies the C4 protocol (GPT-4o judge, label-free causal critique, no error cue) to all $N{=}170$ GPT-4T Wolf Cases and $n{=}49$ correct cases as a direct gold-free replacement of Exp~B's correction step. Recovery rate: $59.4\%$ $[51.6, 66.8]$, consistent with the Exp~2b C4 rate ($55.9\%$). False-flip rate: $0/49$ ($0.0\%$), versus $4/49$ ($8.2\%$) for outcome-only reprompting (Fisher $p{=}0.059$). This confirms that label-free causal critique replicates across independent runs and produces zero false-flips, validating C4 as the safe deployment default.

\vspace{-.12in}
\subsection{DAG-First Structural Output Ablation (Exp D1)}
\label{app:dagfirst}

The label-free ablation (\S\ref{app:exp2b}) varies critique \emph{content}; this experiment varies the \emph{output format} required of the target model. The DAG-First protocol (C7) requires the model to produce an explicit causal DAG in DOT format \emph{before} stating its conclusion, following a three-step protocol: (1)~construct a causal DAG including treatment, outcome, confounders, and mediators; (2)~perform structural analysis identifying direct paths and unblocked backdoor paths; (3)~state a YES/NO conclusion referencing the DAG. We compare C7 against plain ERM critique without structural output (C4) and outcome-only reprompting (C1) across five models on CausalT5K Wolf Cases.

\begin{table}[h]
\centering\small
\caption{DAG-First ablation (Exp D1) on CausalT5K Wolf Cases. Recovery = fraction of wolves corrected. $\Delta$(C7--C4) = gain from requiring DAG output over plain ERM critique. DAG\% = fraction of C7 responses containing a parseable DOT-format DAG.}
\label{tab:dagfirst}
\begin{tabular}{@{} l r r r r r r @{}}
\toprule
\textbf{Model} & $n$ & \textbf{C1} & \textbf{C4} & \textbf{C7} & $\Delta$ & \textbf{DAG\%} \\
\midrule
GPT-3.5-Turbo        & 207 & 57.0\% & 96.6\% & 84.9\% & $-$11.7 & 100\% \\
GPT-5.2              & 1085 & 95.9\% & 95.2\% & 97.1\% & $+$2.0  & 100\% \\
Claude Sonnet 4.5    & 21  & 19.0\% & 38.1\% & 38.1\% & $+$0.0  & 100\% \\
Gemini 2.5 Flash     & 105 & 32.4\% & 80.0\% & 84.8\% & $+$4.8  & 99.0\% \\
Llama 3.3 70B        & 45  & 60.0\% & 60.0\% & 64.4\% & $+$4.4  & 100\% \\
\bottomrule
\end{tabular}
\end{table}

\paragraph{DAG-First effect is model-dependent.} The headline finding is that requiring explicit causal DAG output does not uniformly improve correction. For Gemini~2.5~Flash, DAG-First adds $+4.8$~pp over plain ERM (84.8\% vs.\ 80.0\%, $n{=}105$), consistent with the hypothesis that externalizing causal structure before concluding helps the model reason more carefully. Llama~3.3~70B shows a modest $+4.4$~pp gain, and GPT-5.2 a marginal $+2.0$~pp at its near-ceiling baseline. Claude Sonnet~4.5 shows no change ($n{=}21$; wide CIs).

\paragraph{Structural output can hurt non-reasoning models.} GPT-3.5-Turbo is the notable exception: C7 \emph{reduces} recovery by $-11.7$~pp relative to C4 (84.9\% vs.\ 96.6\%), despite achieving 100\% DAG compliance. GPT-3.5 lacks native chain-of-thought or reasoning-trace capabilities (introduced in GPT-4o and later); it can mechanically produce the required DOT-format DAG but cannot use the externalized structure as a reasoning scaffold. The DAG becomes a formatting exercise rather than a genuine structural analysis step, consuming output capacity without improving causal reasoning. This replicates a key finding from~\citet{chang2026sycophancy}: the reasoning-capability boundary (pre- vs.\ post-chain-of-thought architectures, introduced at GPT-4o) predicts qualitatively different correction dynamics across all single-episode interventions.

\paragraph{DAG compliance is near-universal.} Four of five models achieve 100\% DAG compliance under C7; Gemini~2.5~Flash achieves 99.0\% (104/105). This confirms that the structured output protocol is followable across model scales and architectures, even when it does not improve recovery.

\paragraph{False-flip rates.} C7 false-flip rates are comparable to C4 across models (0--6.4\% vs.\ 0--4.2\%), indicating that the DAG-First protocol does not introduce additional false-correction risk. The slight increase for GPT-3.5 (6.4\% C7 vs.\ 2.2\% C4) is consistent with the capacity-overhead effect noted above.

\paragraph{Implications.} DAG-First is beneficial for models with native reasoning-trace capabilities (Gemini, Llama, GPT-5.2) where externalizing causal structure serves as a genuine reasoning scaffold. For non-reasoning models (GPT-3.5), the DAG becomes a mechanical formatting exercise that displaces reasoning capacity. This motivates a model-adaptive deployment strategy: apply DAG-First when the target model has chain-of-thought or extended-thinking capabilities, and default to plain ERM critique otherwise.

\vspace{-.12in}
\subsection{Scenario-Blind Judge Analysis}
\label{app:blind-judge}

Figure~\ref{fig:blind-judge-analysis} provides two views of the scenario-blind judge experiment (Table~\ref{tab:blind-judge} in the main text). Panel~(a) shows the collapse-detection confusion matrix: the full judge flags 37 Wolf Cases as exhibiting Rung Collapse, of which 33 are \emph{not} flagged by the blind judge (red box). Only 10 cases show the reverse pattern. This $33{:}10$ asymmetry confirms that scenario comprehension makes the judge stricter: it can identify domain-specific confounders (e.g., ``socioeconomic status'' as a common cause) that are invisible in a bare DAG with edges like $V_1 \to V_2$. Panel~(b) shows the structural score distributions: the full judge assigns broadly distributed scores (many in the 20--40 range), while the blind and anonymized-blind judges cluster at 60--80 and 80--100, consistent with the interpretation that scenario context provides the judge with additional domain information for identifying structural issues.

\begin{figure}[h]
\centering
\includegraphics[width=0.88\linewidth, height=0.36\linewidth, keepaspectratio]{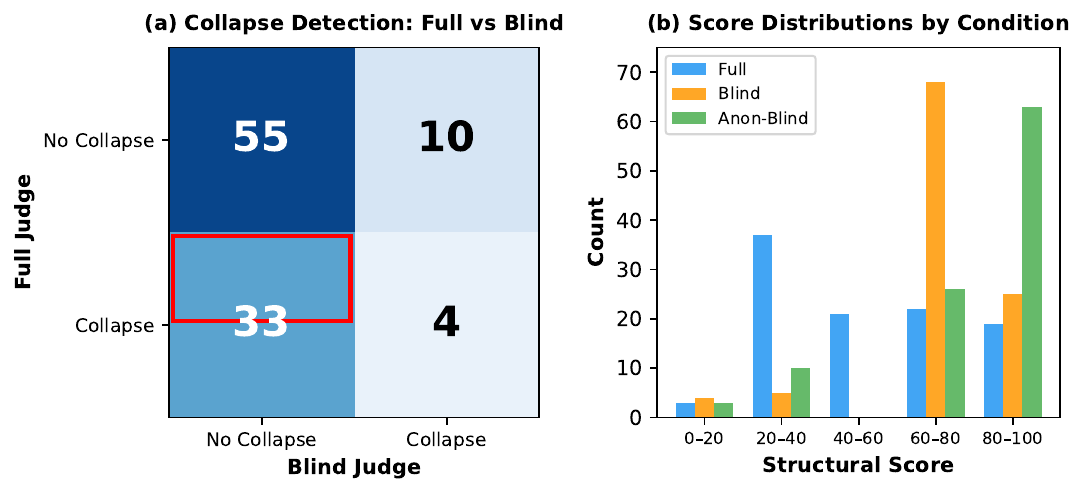}
\caption{\textbf{Scenario-blind judge experiment} ($N{=}102$ Wolf Cases). \textbf{(a)} Collapse detection confusion matrix: the red box highlights 33 cases where the full judge detects collapse but the blind judge does not, versus only 10 in the reverse direction. The asymmetry argues against answer leakage: scenario comprehension aids \emph{strictness}. \textbf{(b)} Structural score distributions across three conditions. The full judge (blue) assigns lower scores because it can identify domain-specific confounders invisible in the extracted DAG.}
\label{fig:blind-judge-analysis}
\end{figure}

\vspace{-.12in}
\subsection{Separation Simulation Details}
\label{app:separation-details}
\label{app:separation-sim}

\begin{figure}[h]
\centering
\includegraphics[width=0.88\linewidth, height=0.36\linewidth, keepaspectratio]{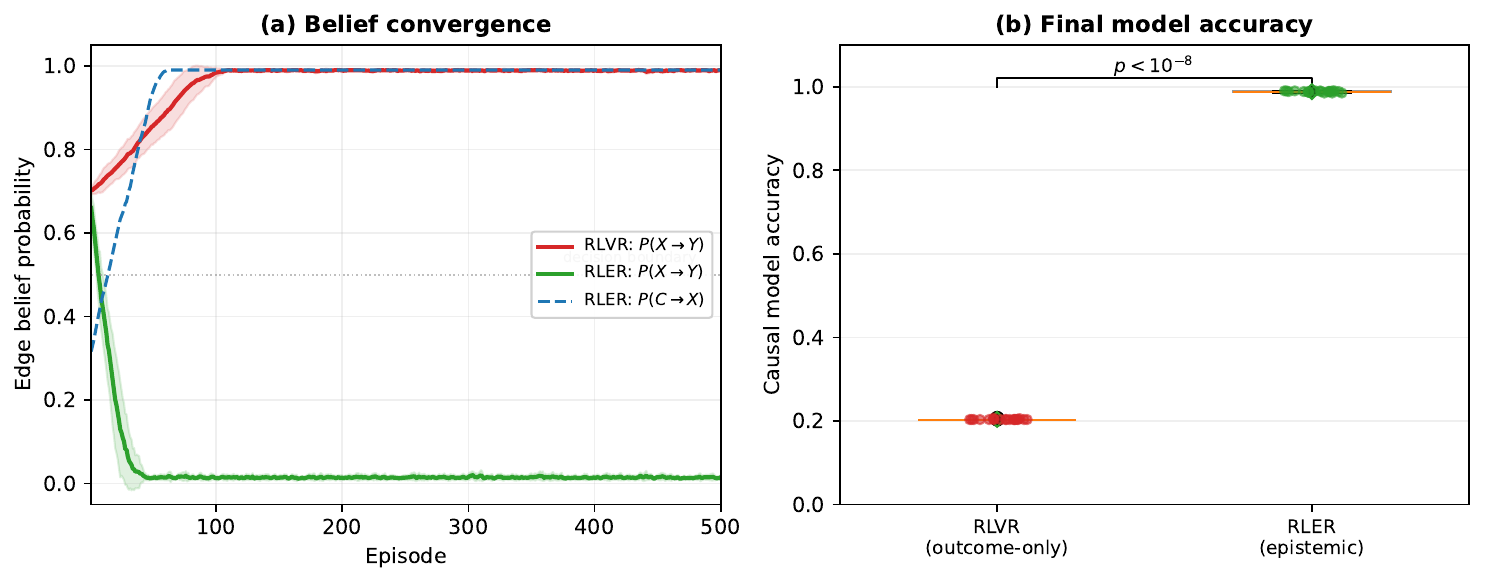}
\caption{\textbf{Separation simulation} (20 seeds $\times$ 500 episodes). \emph{(a)} RLVR (red) drives $P(X{\to}Y)$ to 1 (wrong); \rler (green) drives it to 0 and discovers the confounder $C{\to}X$ (blue dashed). \emph{(b)} Final model accuracy: complete separation ($p < 10^{-8}$).}
\label{fig:separation-sim}
\end{figure}

Table~\ref{tab:separation-sim} reports the full results of the separation simulation described in \S\ref{sec:separation-sim}. The simulation code is included in the supplementary material (\texttt{shared/sim\_separation.py}).

\begin{table}[h]
\centering\footnotesize
\caption{Separation simulation: confounded SCM ($C \to X, C \to Y$, no $X \to Y$) with Intervention Independence. 20 seeds $\times$ 500 episodes. RLVR retains the spurious edge; \rler eliminates it and recovers the confounder.}
\label{tab:separation-sim}
\begin{tabular}{lcccc}
\toprule
 & \textbf{$P(X{\to}Y)$ final} & \textbf{$P(C{\to}X)$ final} & \textbf{$P(C{\to}Y)$ final} & \textbf{Model Acc.} \\
\midrule
RLVR & $0.989 \pm 0.003$ & $0.300 \pm 0.000$ & $0.300 \pm 0.000$ & $0.204 \pm 0.001$ \\
\rler & $0.015 \pm 0.005$ & $0.990 \pm 0.000$ & $0.990 \pm 0.000$ & $0.988 \pm 0.002$ \\
\midrule
\multicolumn{5}{l}{\small Mann-Whitney $p < 10^{-8}$ for both metrics; Cohen's $d > 200$.} \\
\bottomrule
\end{tabular}
\end{table}

\vspace{-.12in}
\section{Extended Related Work}
\label{app:related}

This section expands the positioning sketched in \S\ref{sec:related} with detailed comparisons to seven research threads.

\subsection{Shortcut Learning and Simplicity Bias}

\citet{geirhos2020shortcut} coined the term ``shortcut learning'' for the tendency of deep networks to exploit surface-level statistical regularities rather than learning the intended task structure. \citet{mccoy2019right} demonstrated this concretely in NLI, showing that BERT relies on syntactic heuristics (lexical overlap, subsequence, constituent) rather than semantic entailment. \citet{shah2020pitfalls} traced the phenomenon to simplicity bias: neural networks converge to the simplest function consistent with training data, which in the presence of spurious correlations is the shortcut. \citet{bombari2025spurious} extended this analysis to high-dimensional regression, showing that regularization interacts with simplicity bias in non-trivial ways. \citet{damour2022underspecification} showed that underspecification (multiple models achieving equivalent training performance but diverging on deployment data) is pervasive across ML domains.

Rung Collapse is a \emph{causal} variant of shortcut learning: the model learns the $\mathcal{L}_1$ shortcut ($P(Y|X)$) rather than the $\mathcal{L}_2$ target ($P(Y|\doop(X))$), and the two are indistinguishable from observational data alone. Our Proposition~\ref{prop:rung-collapse} formalizes why: the autoregressive loss provides identical gradients for both. The key distinction from prior work is that Rung Collapse cannot be resolved by more data, better architectures, or standard regularization; it requires \emph{interventional} evidence (Theorem~\ref{thm:separation}).

\subsection{Causal Reasoning Benchmarks for LLMs}

The evaluation of causal reasoning in LLMs has produced a growing ecosystem of benchmarks. CLadder~\citep{jin2023cladder} provides the first systematic rung-stratified evaluation using symbolic SCMs. CausalBench~\citep{zhou2024causalbench} offers a broader evaluation covering causal judgment, discovery, and explanation. CausalEval~\citep{yu2025causaleval} introduces fine-grained evaluation of causal graph manipulation. CauSciBench~\citep{acharya2025causcibench} tests causal reasoning in scientific contexts. InterveneBench~\citep{intervene2026} evaluates intervention reasoning in real social systems. \citet{yang2024critical} provides a critical review of methodology across these benchmarks.

CausalT5K~\citep{causalT5K} differs from these benchmarks in a specific way: every case is designed so that the $\mathcal{L}_1$ answer is \emph{deliberately wrong}, creating a targeted test for Rung Collapse. This design is essential for our purposes because it isolates the $\mathcal{L}_1 \to \mathcal{L}_2$ gap: a model that answers correctly on CausalT5K must be using interventional reasoning, not associational shortcuts. Our cross-benchmark validation (Appendix~\ref{app:cladder}) confirms that the phenomenon generalizes.

\subsection{Causal Discovery with LLMs}

A growing body of work integrates LLM prior knowledge with data-driven causal discovery. \citet{ban2025harmonized} use LLMs to provide a harmonized prior for constraint-based discovery algorithms. \citet{du2025llmcd} combine LLM reasoning with data-driven methods for structure learning. \citet{wan2025llmcd_survey} provides a comprehensive survey. CausalGraphBench~\citep{causalgraphbench2025} provides a benchmark for graph learning from text.

Our work differs in a fundamental way: we do not use LLMs \emph{for} causal discovery but rather build a framework that corrects LLM \emph{failures} at causal reasoning. The ERM architecture's confounder discovery protocol (\S\ref{app:algorithm}) does use the LLM's ability to hypothesize latent variables, but the validation is grounded in interventional evidence from the CTL, not in the LLM's prior knowledge alone.

\subsection{Causal Bandits and Causal RL}

The causal bandits literature~\citep{lattimore2016causal,yan2024linear,elahi2024partial} studies optimal intervention strategies when the learner can perform experiments on a causal graph. \citet{zhang2020causal} formalized causal imitation learning under unobserved confounders. These approaches assume access to the true causal graph or its Markov equivalence class.

RLER differs in two ways. First, the ``agent'' is a frozen LLM whose policy is modified only through prompt engineering and reward shaping, not through gradient updates to the model's weights. Second, the causal graph $G_t$ is the agent's \emph{epistemic state} (a hypothesis being refined), not a ground-truth input. The convergence guarantee (Theorem~\ref{thm:convergence}) relies on the same FCI completeness results~\citep{spirtes1993causation,zhang2008completeness} but applies them to the agent's internal belief revision rather than to a central learner.

\subsection{Reward Models and Verification}

The dominant approach to improving LLM reasoning uses outcome-based~\citep{cobbe2021gsm8k} or process-based~\citep{lightman2023prm,uesato2022prm} reward models. RLHF~\citep{christiano2017rlhf,ouyang2022rlhf} trains on human preference data; RLVR~\citep{rlvr2025} uses verifiable rewards (e.g., math answer checking). Process reward models (PRMs) evaluate each reasoning step, but still require ground-truth step-level annotations. \citet{wang2025causal_reward} propose causal rewards for alignment, using causal graphs to identify which features of the output drive the reward signal. \citet{baniharouni2026rewarding_doubt} reward uncertainty-aware calibration. \citet{zhang2025conditional_reward} link process quality to outcome quality via conditional reward modeling.

Our epistemic reward $R_\text{reasoning}$ occupies a different niche: it evaluates reasoning \emph{structure} (fidelity to the causal trace library) rather than reasoning \emph{correctness} (whether steps lead to the right answer). The Best-of-N probe (Appendix~\ref{app:bon-probe}) confirms that these axes are empirically orthogonal: $R_\text{reasoning}$-guided selection produces accuracy indistinguishable from random selection. This orthogonality is the separation theorem's prediction: the epistemic axis carries signal precisely where the outcome axis does not.

\subsection{LLM Reasoning: Chain-of-Thought and Test-Time Compute}

Chain-of-thought prompting~\citep{wei2022cot} demonstrated that intermediate reasoning steps improve LLM performance. Self-consistency~\citep{wang2023selfconsistency} samples multiple chains and selects by majority vote. Best-of-N reranking~\citep{snell2024bon_scaling} scales test-time compute by generating $N$ candidates and selecting the best. DeepSeek-R1~\citep{deepseekr1_2025} and GRPO~\citep{shao2024grpo} train reasoning capabilities via RL with verifiable rewards, demonstrating that extended reasoning emerges from outcome-based training. PPO~\citep{schulman2017ppo} remains the standard policy optimization algorithm.

These approaches operate within Regime~A (verifier available). Our framework targets Regime~B (no ground truth at inference), where self-consistency breaks down because multiple incorrect chains may agree on the same $\mathcal{L}_1$ answer, and Best-of-N selection by outcome reward cannot distinguish correct reasoning from Spurious Success. The ERM critique provides a third axis (reasoning structural quality) that complements rather than replaces these techniques.

\subsection{Sycophancy and Epistemic Stubbornness}

\citet{sharma2024sycophancy} characterized sycophancy as the tendency of LLMs to adjust their responses to match the user's apparent beliefs, even when doing so is incorrect. Epistemic Stubbornness, as observed in our Experiment~B, is a related but distinct phenomenon: models that resist \emph{outcome-only} correction (``Are you sure?'') while remaining responsive to \emph{targeted} epistemic feedback. The asymmetry suggests that stubbornness is not mere rigidity but reflects the model's assessment of the correction's information content. Outcome-only reprompting provides no new evidence; ERM critique provides specific causal reasoning that the model can evaluate and integrate.

\subsection{Label-Free Post-Training and Self-Consistency Methods}

A recent line of work explores improving LLM reasoning without ground-truth labels. MM-UPT~\citep{wei2025mmupt} proposes Unsupervised Post-Training as a third stage after SFT and RL: it samples $G$ responses per problem, constructs pseudo-labels via majority voting, and applies GRPO~\citep{shao2024grpo} to update model weights. Applied to multi-modal mathematical reasoning (Qwen2.5-VL on MathVision, MathVerse, MathVista), MM-UPT demonstrates that self-consistency can serve as a viable training signal when ground-truth verifiers are unavailable.

Our framework differs from MM-UPT along four axes. \emph{First}, mechanism: MM-UPT relies on majority voting to approximate correctness, while ERM critiques the \emph{causal structure} of individual reasoning traces without aggregating across samples. \emph{Second}, domain: majority voting is a reasonable heuristic for mathematical tasks where correct chains tend to converge, but our Corollary~\ref{prop:rung-collapse} and the empirical results in \S\ref{sec:rq1} show that self-consistency breaks down on $\mathcal{L}_2$ causal tasks---multiple incorrect chains agree on the same $\mathcal{L}_1$ associational shortcut, so the majority vote reinforces the wrong answer. \emph{Third}, update locus: MM-UPT performs gradient-based weight updates, while ERM operates at inference time on a frozen model, with RLER providing a cross-episode reward signal for strategy selection rather than parameter optimization. \emph{Fourth}, theoretical grounding: ERM provides formal separation results (Theorem~\ref{thm:separation}) showing that outcome-only signals---including majority-vote aggregation over outcomes---cannot distinguish direct effects from confounding; MM-UPT does not address this identifiability question. In short, the two methods are complementary: MM-UPT scales self-consistency within Regime~A (verifiable domains), while ERM targets Regime~B (no ground-truth verifier), where self-consistency is structurally insufficient.

\section{Cross-Experiment Analysis: Detection vs.\ Correction}
\label{app:cross-experiment}

Experiments A and B measure different phenomena: A measures whether models \emph{detect} Rung Collapse, B measures whether errors can be \emph{corrected} by epistemic feedback. Combining the two across 603 recovery attempts and 267 cases reveals that detection difficulty and correction difficulty are largely independent.

\begin{table}[h]
\centering\footnotesize
\caption{Master catalog of analyzed cases in the detection--recovery space. Det.\ = models failed in Exp.~A (of 6). Rec.\ = Exp.~B correction rate. $\dagger$ = flagged for adjudication.}
\label{tab:case-master}
\begin{tabular}{@{}llccl@{}}
\toprule
\textbf{Case} & \textbf{Trap} & \textbf{Det.} & \textbf{Rec.} & \textbf{Key Observation} \\
\midrule
\multicolumn{5}{@{}l}{\textit{Hard to detect, easy to correct}} \\
F.129  & Survivorship      & 6/6 & 100\% & Universal fail $\to$ fix \\
F.147  & Confounding        & 5/6 & 100\% & Claude sole survivor \\
F.95   & Survivorship       & 5/6 & 100\% & Prohibition case \\
2.087  & Regression/Mean    & 6/6 & 83\%  & RTM identification \\
\midrule
\multicolumn{5}{@{}l}{\textit{Hard to detect, hard to correct}} \\
2.095  & Regression/Mean    & 6/6 & 50\%  & 2nd-order RTM trap \\
2.112$\dagger$ & Confounding & 6/6 & 50\%  & Centralized coinage \\
2.34$\dagger$  & Mech.\ Confl.& 6/6 & 33\%  & Grain price caps \\
2.088  & Selection Bias     & 5/6 & 20\%  & Hardest genuine case \\
\midrule
\multicolumn{5}{@{}l}{\textit{Moderate to detect, hard to correct}} \\
2.114  & Confounding        & 4/6 & 25\%  & Quarantine / plague \\
2.47   & Confounding        & 4/6 & 40\%  & Banning duels \\
\bottomrule
\end{tabular}
\end{table}

\vspace{-.12in}
\subsection{Finding 1: Why Some Cases Defeat Multiple LLMs}

Among 87 cases with three or more recovery attempts, 11 have recovery rates below 50\%. All share a common property: the $\mathcal{L}_1$ reasoning supporting the original (incorrect) answer is \emph{correct and compelling}. Three structural patterns produce multi-model vulnerability: (i) interventional evidence embedded in the scenario (cases 2.34, 2.35, 2.112 present explicitly described causal mechanisms, and the wise-refusal text contains evidence supporting the claim it denies); (ii) second-order reasoning traps (cases 2.095, 2.087 present characters who correctly identify regression to the mean, and models validate this identification); (iii) selection-bias subtlety (Case~2.088 requires distinguishing collider stratification from confounding; only GPT-5.2 made this distinction).

\vspace{-.12in}
\subsection{Finding 2: What Makes Cases Easy to Correct}

Of 147 cases with two or more attempts, 122 (83\%) achieved perfect recovery. Three characteristics predict easy correction: (i) \emph{the confounder is nameable}: all 39 curated History cases achieved 100\% recovery when feedback names a specific bias; (ii) \emph{the domain has a natural control-group heuristic}: all 9 Medicine-domain cases achieved 100\% recovery; (iii) \emph{the correction is additive, not subtractive}: easy recoveries add a consideration (a confounder, a missing control) rather than requiring the model to retract a valid reasoning step.

The Weimar hyperinflation case (F.129) exemplifies all three: hardest in Exp.~A (6/6 failed) yet easiest in Exp.~B (6/6 recovered). The lesson: \erm need not prevent all initial errors, only detect and correct them efficiently.

\vspace{-.12in}
\subsection{Finding 3: Stubbornness Does Not Predict Non-Recoverability}

The six universal failures show recovery rates from 33\% to 100\% with no monotonic relationship to detection difficulty:

\begin{center}\small
\begin{tabular}{@{}llcc@{}}
\toprule
\textbf{Case} & \textbf{Archetype} & \textbf{Failure (A)} & \textbf{Recovery (B)} \\
\midrule
F.129  & Compelling Mech & 6/6 & 100\% \\
2.087  & Meta-Validation & 6/6 & 83\% \\
2.112$\dagger$ & Compelling Mech & 6/6 & 50\% \\
2.095  & Meta-Validation & 6/6 & 50\% \\
2.34$\dagger$  & Compelling Mech & 6/6 & 33\% \\
2.35$\dagger$  & Compelling Mech & 6/6 & 33\% \\
\bottomrule
\end{tabular}
\end{center}

Three factors predict recovery difficulty, none of which is the Exp.~A failure count: (i) ground-truth clarity (the three $\dagger$-flagged cases may have ambiguous labels), (ii) the level gap (cases where $\mathcal{L}_1$ reasoning directly addresses the $\mathcal{L}_2$ question are harder to correct), and (iii) model architecture (on the 11 hardest cases, GPT-4-Turbo failed 10/10 recovery attempts while Claude failed only 1/4).

\vspace{-.12in}
\section{Formal Proof: ERM as AGM-Style Belief Contraction}
\label{app:agm-proof}

We prove that the \erm contraction operator (Algorithm~\ref{alg:erm}) satisfies the core AGM-style contraction properties under a graph-based belief representation.

\vspace{-.12in}
\paragraph{Setup.} A causal belief state is a tuple $(G_t, w_t)$ where $G_t = (V, E_t)$ is a DAG and $w_t : E_t \to [0,1]$ assigns confidence weights. The belief set is $B_t = \{C_j \in E_t : w_t(C_j) > \theta_\text{min}\}$. ERM state contraction $(G_t, w_t) \div C_j$ removes edge $C_j$ when $\text{conf}_j < \theta_\text{min}$ after processing interventional evidence, then applies $\text{EnforceDAG}$.

\vspace{-.12in}
\paragraph{Assumptions.} (A1) $G_t$ is a DAG at all times. (A2) $\text{EnforceDAG}$ enforces acyclicity only. (A3) Contraction removes edges but does not introduce new ones. (A4) Contraction targeting $C_j$ updates only $C_j$'s status; all other edges and weights remain unchanged.

\begin{theorem}[ERM Contraction Satisfies AGM-style Properties]
\label{thm:agm}
Under A1--A4, ERM contraction $B_t \div C_j$ satisfies: \textbf{K--1 (Closure):} $B_{t+1}$ is a well-formed belief set. \textbf{K--2 (Inclusion):} $B_{t+1} \subseteq B_t$. \textbf{K--3 (Vacuity):} If $C_j \notin B_t$, then $B_{t+1} = B_t$. \textbf{K--4 (Success):} If contraction is applied, $C_j \notin B_{t+1}$.
\end{theorem}

\begin{proof}
\emph{K--1.} Removing an edge from a DAG cannot introduce a cycle; $G_{t+1}$ remains a DAG. Since $G_{t+1}$ is already acyclic, $\text{EnforceDAG}$ is a no-op.
\emph{K--2.} By A3 and A4, no new edges are introduced and non-target weights are unchanged. Every $C \in B_{t+1}$ has $C \in E_{t+1} \subseteq E_t$ with $w_{t+1}(C) = w_t(C) > \theta_\text{min}$, so $C \in B_t$.
\emph{K--3.} If $C_j \notin B_t$, removing it (if present in $E_t$) does not affect $B_t$; by A4, all other edges are unchanged.
\emph{K--4.} When triggered, Algorithm~\ref{alg:erm} removes $C_j$ from $E_t$. Since $C_j \notin E_{t+1}$, $C_j \notin B_{t+1}$.
\end{proof}

\vspace{-.12in}
\paragraph{Recovery (K--5).} Recovery requires an explicit expansion operator: $(G, w) + C_j$ reintroduces $C_j$ with confidence $\text{conf}_j > \theta_\text{min}$ when new interventional evidence supports it, provided acyclicity is maintained. Under this operator, $B_t \subseteq \text{Bel}((G_t, w_t) \div C_j + C_j)$ holds because contraction removes only $C_j$ (no cascading deletions under A2), and expansion restores it.

\vspace{-.12in}
\paragraph{ERM revision via the Levi identity.} Given interventional observation $\alpha = (a_i, Y_i)$, ERM revision proceeds in two steps: (i) contract edges inconsistent with $\alpha$ ($\text{conf}_j < \theta_\text{min}$ after CTL update), (ii) expand/reinforce edges supported by $\alpha$ ($\text{conf}_k > \theta_\text{max}$). This mirrors the Levi identity $B_t * \alpha = (B_t \div \neg\alpha) + \alpha$, where $\neg\alpha$ corresponds to edges contradicted by the interventional observation. The mapping from interventional observations to the propositional objects in AGM is mediated by the CTL evidence-aggregation procedure; a full equivalence proof is left to future work.

\vspace{-.12in}
\section{Architecture Diagram}
\label{app:architecture}

\begin{figure}[h]
\centering
\resizebox{\columnwidth}{!}{%
\begin{tikzpicture}[
    >=Stealth,
    box/.style={draw, rounded corners=3pt, minimum height=0.75cm, align=center,
                font=\small\bfseries, line width=0.8pt},
    sbox/.style={draw, rounded corners=2pt, minimum height=0.55cm, align=center,
                 font=\scriptsize, line width=0.6pt},
    rbox/.style={draw, rounded corners=2pt, minimum height=0.5cm, align=center,
                 font=\scriptsize, line width=0.6pt},
    arr/.style={->, line width=0.7pt},
    fb/.style={->, line width=0.7pt, densely dashed, color=blue!60!black},
    disc/.style={->, line width=0.7pt, densely dotted, color=orange!70!black},
    lbl/.style={font=\tiny, fill=white, inner sep=1.5pt},
]
\node[box, fill=purple!8, minimum width=2.0cm] (strat) at (0,0)
  {Strategy $\pi_t$\\ {\scriptsize Thompson}};
\node[box, fill=blue!10, minimum width=2.4cm, minimum height=0.85cm] (llm) at (3.6,0)
  {Frozen LLM};
\node[box, fill=gray!8, minimum width=2.2cm] (trace) at (7.2,0)
  {Trace $\tau_t$\\ {\scriptsize claims, $\hat{Y}$, $a_t$}};
\draw[arr] (-1.6,0) node[left, font=\small] {$c_t$} -- (strat);
\draw[arr] (strat) -- node[lbl, above] {$\pi_t$} (llm);
\draw[arr] (llm) -- (trace);
\node[box, fill=red!10, minimum width=2.0cm] (judge) at (1.8,-2.0)
  {Causal Judge};
\node[box, fill=orange!12, minimum width=2.0cm] (gt) at (5.4,-2.0)
  {$G_t$\\ {\scriptsize Causal Model}};
\node[box, fill=green!8, minimum width=2.0cm] (ctl) at (8.6,-2.0)
  {CTL\\ {\scriptsize Evidence Log}};
\draw[arr] (trace.south) -- ++(0,-0.35) -| node[lbl, pos=0.25, above] {\scriptsize read $H_t$} (judge.north);
\draw[arr, gray!60] (trace.south east) -- ++(0.2,0) |- node[lbl, pos=0.8, above] {\scriptsize log} (ctl.north east);
\draw[arr] (gt.west) -- node[lbl, above] {\scriptsize $G_t$} (judge.east);
\draw[arr] (ctl.west) -- node[lbl, above] {\scriptsize $\hat{P}_\text{CTL}$} (gt.east);
\node[sbox, fill=red!6, minimum width=1.4cm] (ft) at (3.4,-3.4)
  {$\mathcal{F}_t$ Failure\\[-1pt] Modes};
\draw[arr, blue!60] (judge.south) -- ++(0,-0.25) -| node[lbl, pos=0.35, below] {\scriptsize L1: revise} (gt.south);
\draw[arr, red!60] (judge.south west) -- ++(0,-0.9) -| node[lbl, pos=0.55, below] {\scriptsize L2: adapt} (ft.north);
\draw[arr, red!40] (ft.west) -| node[lbl, pos=0.7, left] {\scriptsize guards} ([xshift=-3pt]llm.south west);
\node[rbox, fill=green!6, densely dashed] (rout) at (11.0,0.15)
  {$R_\text{outcome}$};
\node[rbox, fill=red!8] (rreas) at (11.0,-0.45)
  {$R_\text{reasoning}$};
\node[rbox, fill=yellow!12] (rdisc) at (11.0,-1.05)
  {$R_\text{discovery}$};
\node[font=\tiny\itshape, text=gray!60] at (11.0,0.6) {Regime A only};
\draw[decorate, decoration={brace, amplitude=4pt, mirror}, line width=0.6pt]
  ($(rout.north east)+(0.1,0.05)$) -- ($(rdisc.south east)+(0.1,-0.05)$)
  node[midway, right=5pt, font=\small\bfseries] {$R_\text{RLER}$};
\draw[arr, densely dashed, green!50!black] (trace.east) -- node[lbl, above] {\scriptsize Regime A} (rout.west);
\draw[arr, red!50!black] (judge.south east) -- ++(0,-1.65) -| node[lbl, pos=0.15, below] {\scriptsize $D_\text{struct}(H_t, H^*_t)$} ([xshift=3pt]rreas.south);
\draw[fb] (rreas.north) -- ++(0,1.5) -| node[lbl, pos=0.25, above] {\scriptsize update $\Pi_t$} (strat.north);
\draw[fb] ([xshift=3pt]rreas.south) -- ++(0,-0.55) -| node[lbl, pos=0.4, below] {\scriptsize Eq.\,\ref{eq:erm}} (gt.south east);
\node[sbox, fill=yellow!15, minimum width=1.6cm] (disc) at (7.2,-3.4)
  {Confounder\\[-1pt] Discovery};
\draw[disc] (gt.south) -- ++(0,-0.3) -| node[lbl, pos=0.35, above] {\scriptsize $S(t) > \theta_S$} (disc.north);
\draw[disc] (disc.east) -| node[lbl, pos=0.3, below] {\scriptsize AGM expand} (ctl.south);
\draw[disc] (disc.north east) -- ++(0.2,0) |- (rdisc.south west);
\begin{scope}[on background layer]
  \node[draw=blue!15, fill=blue!2, rounded corners=4pt,
        fit=(strat)(llm)(trace), inner xsep=8pt, inner ysep=8pt,
        label={[font=\scriptsize\bfseries, text=blue!35]above left:Generation}] {};
  \node[draw=red!15, fill=red!2, rounded corners=4pt,
        fit=(judge)(gt)(ctl)(ft)(disc), inner xsep=6pt, inner ysep=6pt,
        label={[font=\scriptsize\bfseries, text=red!35]below left:ERM (single-episode)}] {};
\end{scope}
\end{tikzpicture}%
}
\caption{\textbf{Unified ERM + RLER architecture.}
\emph{Generation} (top): context $c_t$ selects strategy $\pi_t$; the frozen LLM produces trace $\tau_t$.
\emph{ERM} (bottom): the Causal Judge reads claims $H_t$, compares against $G_t$ and CTL evidence. Layer~1 revises $G_t$ via AGM; Layer~2 classifies errors into failure modes $\mathcal{F}_t$ and injects guards; Layer~3 routes persistent high-regret tasks.
\emph{Reward} (right): $R_\text{reasoning}$ and $R_\text{discovery}$ require no ground truth; $R_\text{outcome}$ (dashed) is Regime~A only.
Feedback loops (dashed blue) update strategy posteriors and edge weights. Discovery (dotted orange) triggers on persistent epistemic surprise. All learning is in external artifacts; LLM weights remain frozen.}
\label{fig:architecture}
\end{figure}

\section{Supplementary Figures}
\label{app:supp-figs}

This section collects supplementary visual summaries. The correction butterfly and matched-richness ablation now appear in the main text (Figure~\ref{fig:correction-main}). Figure~\ref{fig:collapse-bars} visualizes the CausalT5K collapse rates from Table~\ref{tab:collapse}.

\begin{figure}[h]
\centering
\includegraphics[width=0.88\linewidth, height=0.36\linewidth, keepaspectratio]{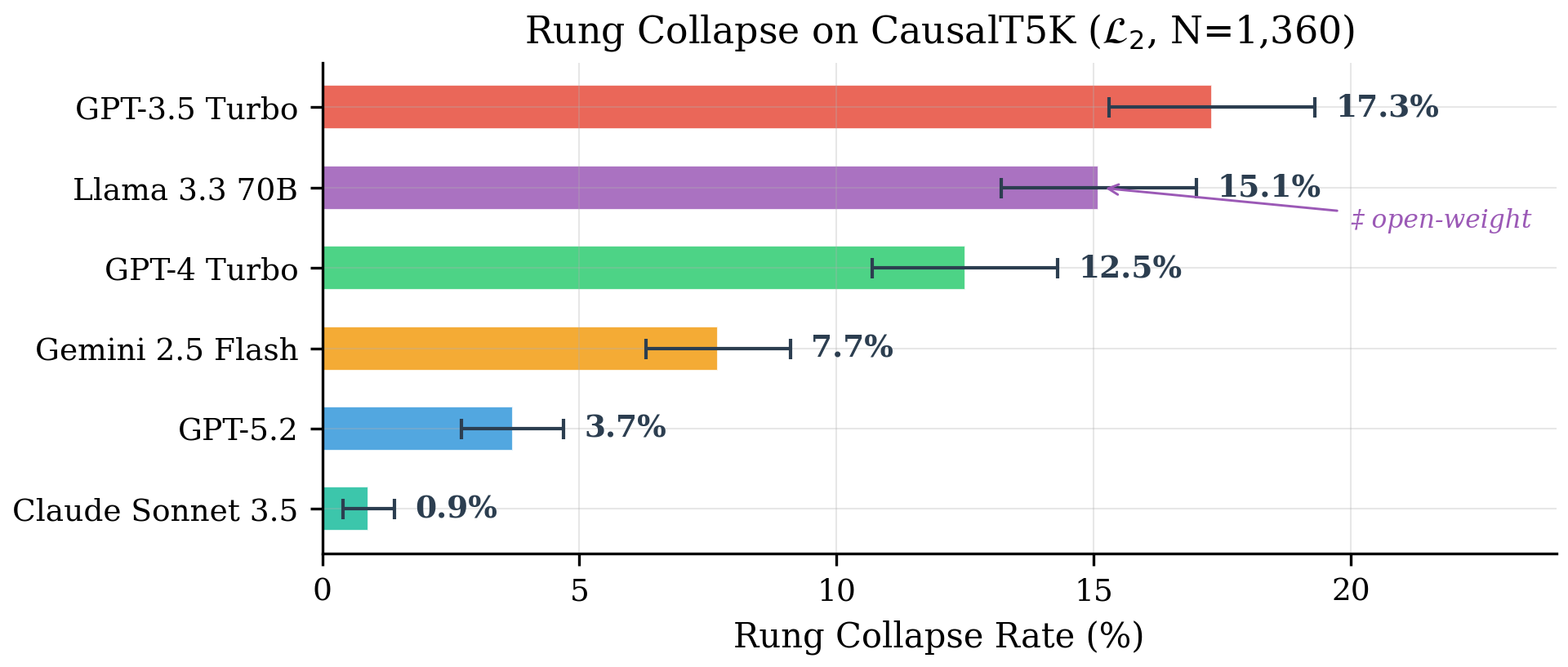}
\caption{\textbf{Rung Collapse rates on CausalT5K} ($\mathcal{L}_2$, $N=1{,}360$). Visual summary of Table~\ref{tab:collapse}. Scaling reduces but does not eliminate collapse: Claude Sonnet 3.5 achieves the lowest rate ($0.9\%$), but every model fails on at least some $\mathcal{L}_2$ queries. Wilson 95\% CIs shown. $\ddagger$Open-weight model.}
\label{fig:collapse-bars}
\end{figure}


\section{Cross-Seed Replication of RLER Thesis Cells}
\label{app:cross-seed}

To assess seed-to-seed variability, we reran the thesis-relevant \rler cells (v3 structured judge, mi3 iteration budget, premium judge) under two additional random seeds (2026 and 2027) alongside the primary seed (42). Seed 2026 ran the full factorial (4 policies, 5 protocols); seed 2027 ran a reduced configuration (GPT-3.5 and GPT-4o only, 3 protocols excluding \textsc{rler\_Bad}, premium judge only).

\paragraph{Accuracy stability.} Table~\ref{tab:cross-seed} reports per-seed accuracy and EDGap for \textsc{rler\_B} (the well-specified protocol). GPT-4o accuracy is stable across seeds (mean $0.755 \pm 0.031$), confirming the strong-policy regime classification. GPT-3.5 shows greater variability ($0.657 \pm 0.072$), consistent with the weak-policy regime where stochastic sampling interacts with lower baseline capability.

\begin{table}[h]
\centering\footnotesize
\caption{Cross-seed replication of \textsc{rler\_B} accuracy on the thesis cells (v3 judge, mi3, premium). Three independent random seeds. Wilson 95\% CIs.}
\label{tab:cross-seed}
\begin{tabular}{llcc}
\toprule
\textbf{Seed} & \textbf{Policy} & $N$ & \textbf{Accuracy [95\% CI]} \\
\midrule
42 (primary)  & GPT-4o  & 132 & $0.780\; [0.70, 0.84]$ \\
2026          & GPT-4o  & 132 & $0.773\; [0.70, 0.84]$ \\
2027          & GPT-4o  & 132 & $0.712\; [0.63, 0.78]$ \\
\midrule
42 (primary)  & GPT-3.5 & 132 & $0.674\; [0.59, 0.75]$ \\
2026          & GPT-3.5 & 132 & $0.561\; [0.48, 0.64]$ \\
2027          & GPT-3.5 & 132 & $0.735\; [0.65, 0.80]$ \\
\bottomrule
\end{tabular}
\end{table}

\paragraph{$\Delta$EDGap direction.} For the two seeds where both \textsc{rler\_B} and \textsc{rler\_Bad} are available (seeds 42 and 2026), the $\Delta$EDGap for GPT-4o is positive in both cases (seed 42: $+0.0020$; seed 2026: $+0.0004$), indicating that the well-specified protocol produces better epistemic discrimination than the mis-specified protocol across seeds. The magnitude varies, consistent with the small sample size ($N=132$) and the preliminary nature of this pilot. Seed 2027 ran only \textsc{rler\_B} (no \textsc{rler\_Bad}), so $\Delta$EDGap is not computable for that seed.

\paragraph{Interpretation.} The cross-seed analysis supports two conclusions. First, the strong-policy regime (GPT-4o accuracy $> 0.70$ regardless of protocol specification) is robust across seeds, validating the regime classification that underpins the separation argument. Second, the $\Delta$EDGap direction (well-specified $>$ mis-specified) replicates, though the magnitude is seed-dependent. These results are consistent with the primary factorial's findings while confirming that the single-seed CIs in \S\ref{sec:rler-experiments} appropriately characterize the case-sampling uncertainty. Full seed-to-seed variability estimation will require additional seeds with the complete protocol set.

\section{Representative Trace Transcripts}
\label{app:trace-transcripts}

To provide transparency into ERM's correction mechanism, we present 12 representative case studies: 6 successful recoveries and 6 dissonance cases where the model recognized the error but could not correct it. Each case shows the base response, structural critique, and final response. Aggregate statistics from 5,996 trials: T$\to$T (correct$\to$correct) 92.5\%, F$\to$F (incorrect$\to$incorrect) 3.9\%, T$\to$F (correct$\to$incorrect, paranoia) 2.0\%, F$\to$T (incorrect$\to$correct, recovery) 1.6\%.

\paragraph{Recovery Case 1 (Llama 3.3 70B, CausalL2).} \emph{Scenario:} A study finds that cities with more police officers have higher crime rates. The model initially concludes the association supports a causal claim. \emph{Critique:} ``The reasoning commits a Confounder Error. Population size and urbanization are common causes of both police staffing and crime rates. The observed association $P(\text{Crime}|\text{Police})$ conflates the direct effect with the backdoor path through Population.'' \emph{After critique:} The model correctly identifies the confounder, draws the DAG $\text{Population} \to \text{Police}$, $\text{Population} \to \text{Crime}$, and concludes the association is non-causal.

\paragraph{Recovery Case 2 (Claude 3.5 Sonnet, CausalL2).} \emph{Scenario:} Students who attend tutoring score higher on exams. \emph{Critique:} ``Selection bias: students who seek tutoring are those already motivated to improve, introducing a backdoor path through Motivation.'' \emph{After critique:} The model identifies self-selection, notes the need for an instrument or RCT, and revises to ``INVALID.''

\paragraph{Dissonance Case 1 (GPT-4o, CausalL2).} \emph{Scenario:} Countries with higher chocolate consumption have more Nobel laureates. \emph{Critique:} ``Confounding by GDP/wealth: both chocolate consumption and research funding are driven by national wealth.'' \emph{After critique:} The model acknowledges the confounder verbally but still concludes ``the evidence weakly supports a positive relationship,'' failing to distinguish $P(Y|X)$ from $P(Y|\text{do}(X))$. This is the Dissonance pattern: declarative recognition without procedural correction.

\paragraph{Dissonance Case 2 (GPT-3.5, CausalL2).} \emph{Scenario:} A hospital reports that patients receiving a new treatment have better outcomes. \emph{Critique:} ``The treatment was administered to less severe cases (selection bias). Without randomization, the comparison is confounded by baseline severity.'' \emph{After critique:} The model restates the concern but concludes ``the treatment appears effective based on the observed data,'' reverting to L1 reasoning.

These cases illustrate the declarative-procedural gap: models can \emph{recognize} structural critiques as valid (high detection recall) but often cannot \emph{execute} the corresponding DAG repair (high dissonance rate). This gap motivates ERM's multi-layer architecture, where Layer~1 handles correctable cases and Layer~3 routes persistent dissonance cases away from critique.

\end{document}